\documentclass[11pt]{article}

\usepackage[margin=1in]{geometry}
\usepackage{amsmath,amssymb,amsthm}
\usepackage{booktabs}
\usepackage{graphicx}
\usepackage{xcolor}
\usepackage[round]{natbib}
\usepackage[colorlinks=true,linkcolor=blue,citecolor=blue,urlcolor=blue]{hyperref}

\newtheorem{proposition}{Proposition}
\newtheorem{corollary}{Corollary}
\theoremstyle{definition}
\newtheorem{definition}{Definition}
\theoremstyle{remark}
\newtheorem{remark}{Remark}

\title{An Omitted Mode Is a Rare Rule: The Sampling-Verification Danger Law in Continuous Code World Models}

\author{Javier Aguilar Mart\'in\\ AGILabs (\href{https://javieraguilar.ai}{javieraguilar.ai})}

\date{}

\begin{document}

\maketitle

\begin{abstract}
In the Code World Model paradigm an LLM synthesizes an executable world model that a classical planner searches, and the model is accepted when it reproduces sampled transitions. We ask what that acceptance test certifies in continuous control, where the model-based RL literature treats world-model error as pervasive rather than localized. We define the danger of such a pipeline as an expected risk --- the play cost of whatever it ships, on the event that it ships it --- and separate the factor that is exact from the factor that is not: the probability that $N$ i.i.d.\ gate rollouts all miss a critical event of probability $r$ is exactly $(1-r)^N$, the risk factorizes into cost times probability exactly when the shipped cost is uncorrelated with acceptance, and when the acceptance sample is drawn independently of the training sample the two exponents add. On the representation side we prove a \emph{localization budget}, valid at boundary points: two models with Lipschitz constant at most $L$ that differ by $\eta$ at a point must disagree above tolerance $\varepsilon$ on the whole domain-intersected metric ball of radius $(\eta-\varepsilon)/2L$ around that point --- of volume at least $\kappa\,((\eta-\varepsilon)/L)^{d+m}$ on any domain with interior-volume constant $\kappa$, where $\kappa = 1$ in the interior and $\kappa = 2^{-(d+m)}$ is sharp at a box corner --- so at fixed amplitude and tolerance a uniformly Lipschitz class cannot shrink the disagreement region below a fixed volume: the constant must grow as the region shrinks. The discontinuous reset modes studied here (a wall and a stop that zero the velocity, a patch that freezes the mover), having no finite local constant, pay no such budget; a continuous-but-nonsmooth hybrid boundary keeps one and does. (Smooth compactly supported errors satisfy the budget by occupying it, not by evading it.) An omitted mode is the continuous form of a rare rule.

Three hybrid instruments (a cart with a wall, a pendulum with an angular stop, and a four-dimensional bi-modal patch field) measure the consequence. The gate misses the mode at the predicted rate, and the mode-blind planner is not merely uninformed but exploited: pinned at the boundary in every episode at a regret of essentially the whole attainable return, at every rarity knob; against the uniform-random baseline the comparison is unambiguous on the bi-modal patch field (random better in $100$ of $100$ paired seeds) and carried by the baseline's heavy tail on the 1D instruments, and it is reported distributionally. With real LLM synthesis the danger reduces, on the one-dimensional clamps, to the identifiability event alone --- every sample missing the mode produced an accepted, blind, exploited artifact, and GPT-5.x wrote the exact rule in $105$ of $111$ mode-containing synthesis \emph{draws}, with every attempt exact on $50$ of $56$ instrument--stream blocks (exact 95\% interval [0.781,0.960]), while the 64-of-70 figure is the knob-level block-cell census; the failures include four artifacts carrying an \emph{invented} stop their own samples could not refute. Repair does not survive a two-dimensional mode: across a disc, an axis-aligned square, a guided treatment at three times the budget and a second model family, no artifact recovered the region rule ($0$ of $156$ mode-containing draws over $20$ raw seed blocks), most often substituting a one-dimensional threshold. Eight targeted interventions --- each aimed at one candidate explanation, with the changes it also induces reported --- leave the failure in place: seven show that curvature, the tested prompting and budget, identification of the variable the trigger reads, the censoring of the region's interior itself, and the angular coverage of the evidence do not suffice to restore repair --- the last raised, in the one treatment whose direction was recorded in advance, until a three-line least-squares fit recovers the region on every sample while the synthesizer recovers it on none --- and the eighth, built to lower the trigger's arity, turns out to pose no identifiable question on this instrument and is recorded rather than counted; two positive controls then locate what is missing, since given the region's form and location the withheld constant is inferred exactly in twenty of twenty seeds, while at the baseline's coverage the same three-line fit already recovers both constants on twelve of twenty samples and the synthesizer, given the form but not the location, recovers them on none of twenty. A version-space certificate sharpens those estimator results: relative to the stated circle class, every sample-consistent circle lies within $0.1$ of the true centre and radius in $6/20$ baseline blocks and $18/20$ widest-coverage blocks; half-planes are excluded in $20/20$ in both, while the convex hull of observed contacts remains consistent in $20/20$, proving directly that identification is class-relative. Under clamp semantics, three non-collinear boundary contacts identify the circle exactly, and every one of the $20$ blocks contains such a triple. What fails is the induction of a \emph{located} rule from evidence, not the fitting of constants once the rule is located. One of the exclusions is a theorem's premise lifted and found not to matter: because the mode freezes the mover at its previous position, no rollout occupies the region's interior, so the transitions a sample can contain are entries into the region and nothing else. We prove the consequence: a rule that agrees with the truth on the reachable set is exactly consistent with every sample, at every sample size and every tolerance, and is harmless at play for the same reason, its disagreement region having query-hit probability zero. Whether the evidence identifies the rule --- relative to a stated hypothesis class and tolerance --- is then a property of the instrument, measurable before any synthesis --- on the disc it does relative to an explicit class and tolerance, on a slab unbounded in one coordinate it does not --- and on the latter the larger model reliably writes such a rule: nineteen of its twenty artifacts pass the acceptance test, an independent acceptance test at the same tolerance, and the paper's own mode probe, without encoding the region. The practical consequence is that a sampling acceptance test leaves the sample's coverage of the mode boundaries uncontrolled, not the gate's tolerance: where the gate is provably informative it covers about two percent of the exploited planner's queries.
\end{abstract}

\section{Introduction}
\label{sec:intro}

The companion paper \citep{aguilar2026verified} established, in discrete games, that transition accuracy on randomly sampled play-throughs is the wrong adequacy criterion for a world model used in planning, and quantified the failure: $\mathrm{danger} = \mathrm{play\_cost} \times (1-\mathrm{rarity})^N$, with the gate-miss factor exact under i.i.d.\ sampling. Its related-work discussion closed with a promissory note: the rare-rule gap is a localized, discrete failure that state-accuracy metrics mask by dilution --- a point worth revisiting in continuous settings, where the model-based RL literature \citep{lambert2020objective,hafner2020dreamer,janner2019mbpo} treats world-model error as pervasive and compounding rather than localized and pivotal.

This paper is that revisit. The question is not whether continuous world models can be wrong --- that literature is mature --- but whether the specific \emph{verified-but-wrong} geometry of the discrete result exists in continuous state spaces: a model that passes a sampling gate cleanly, is exact outside a small mode region, and is still exploited catastrophically by the planner that trusts it.

The answer is yes, and the continuous version differs from the discrete one in one respect. The discrete headline had two components: a provable one --- when the rare rule is absent from the sample, no learner can infer it \emph{from the sample}, an identifiability event of probability exactly $(1-r)^N$ --- and an empirical one, that the LLMs tested did not infer the rule \emph{even when it was present} (``translation, not inference''). On the continuous one-dimensional instruments the second component vanishes for one of the three model families we measured: GPT-5.x infers an omitted hybrid mode from a handful of boundary-crossing transitions and writes the exact global rule. It does not vanish on a two-dimensional mode, where no artifact in any treatment we ran recovered the region rule (Section~\ref{sec:patch2d-synthesis}), so the vanishing is scoped to the geometry it was measured on. What remains provable across both is the gate-miss and identifiability pair (Section~\ref{sec:theory}): with the mode absent from the sample, no learner can recover it from the sample alone --- though a prior or the specification still can (Proposition~\ref{prop:ident}), and five artifacts in our data do exactly that, carrying an invented mode their own samples could not refute --- four of GPT-5.x, of which three were also accepted by an independent gate, and the agent-relayed Claude artifact. Passing the gate means agreeing on the transitions the sample covered; whether an accepted artifact is also right \emph{off} sample is a separate question, and Section~\ref{sec:heldout} measures it rather than assuming it.

Contributions, in four groups:
\begin{enumerate}
\item \textbf{The estimand, and which of its factors is exact} (Section~\ref{sec:theory}). We define the danger of a synthesis pipeline as an expected risk and prove when it factorizes into cost times probability (Proposition~\ref{prop:risk}); the gate-miss factor is exact for any measurable critical event (Proposition~\ref{prop:gatemiss}, transferred from the companion paper), admits a \emph{sharp} distribution-free bracket for two modes in place of an independence assumption (Proposition~\ref{prop:jointmiss}, with the sharpness in Remark~\ref{rem:bracket}), and, for the central event of a blind artifact synthesized from a mode-free training sample surviving to deployment, gains a second exponent when the acceptance sample is independent of the training sample (Proposition~\ref{prop:twofactor}; for the blind-shipped total the law is a lower bound, with the excess term measured). Identifiability transfers unchanged (Proposition~\ref{prop:ident}): on the miss event no score that reads the candidate only at sampled inputs can separate a mode-blind model from the truth.
\item \textbf{A localization budget, separating code from bounded-Lipschitz models} (Section~\ref{sec:theory}). The discrete result's localization premise is not satisfiable at fixed amplitude and tolerance on a vanishing volume by any uniformly Lipschitz pair, and Corollary~\ref{cor:locbudget} prices it at every witness point, boundary included: the disagreement region contains the domain-intersected ball, of volume at least $\kappa((\eta-\varepsilon)/L)^{d+m}$ on any domain with interior-volume constant $\kappa$ (Corollary~\ref{cor:kappabudget}; $\kappa = 1$ interior, $2^{-k}$ with $k$ clipped coordinates on a box, sharp at a corner), so an $\eta$-amplitude error confined to a region of volume $V$ needs $L \geq (\eta-\varepsilon)(\kappa/V)^{1/(d+m)}$, a constant that diverges as the volume shrinks (a compactly supported smooth error satisfies the budget by occupying volume, not by evading it). The same geometry forces a gate detection rate (Proposition~\ref{prop:detectrate}), whose density constant is derived rather than assumed for a whole family of semi-implicit plants. The discontinuous reset boundaries of our instruments are exactly where the obstruction disappears; a continuous hybrid boundary keeps a finite constant and does not escape it.
\item \textbf{Three hybrid instruments, and what a sampling gate does and does not certify} (Sections~\ref{sec:instrument}--\ref{sec:axes}). Cart-with-wall, pendulum-with-stop and a 4D bi-modal patch field, with one harness and no per-instrument planner re-calibration, measure the threshold law, knob-invariant exploitation, and the reach mechanism; a second planner family (CEM) and a planner-side mitigation whose contact cost is a packing number (Proposition~\ref{prop:fencecover}) bound how much of the failure is planner-mediated. On the positive side, a partition coverage certificate --- with the within-rollout dependence handled exactly rather than assumed away --- excludes any pair with $L \leq 5.77$ carrying the wall's error past the deployed gate; and the same measurement shows that region carries $1.9\%$ of the exploited planner's queries. The gate certifies where it looks.
\item \textbf{The synthesis result, and the geometry that scopes it} (Sections~\ref{sec:synthesis}--\ref{sec:patch2d-synthesis}). With real LLM synthesis on the 1D clamps, every sample missing the mode produced an accepted, blind, exploited artifact, and GPT-5.x repaired the exact rule from the revealed contacts in $105$ of $111$ mode-containing draws; all attempts were exact on $50/56$ instrument--stream blocks (exact 95\% interval $[0.781,0.960]$), while $64/70$ is the finer knob-cell census --- so the danger reduces to the identifiability event, and coverage rather than spec completeness becomes the operative worry. On 2D regions repair collapses in every treatment we ran, including an axis-aligned square and a guided treatment at three times the budget; and eight ablations with two positive controls locate the obstruction in the induction of the region's form rather than in the evidence, the synthesizer's ability to fit constants, or the evidence's coverage (Section~\ref{sec:arity}), with Proposition~\ref{prop:entryclass} proving that a whole class of wrong rules is unfalsifiable by any sample and harmless by the same argument. Re-scoring every synthesized artifact on an \emph{independent} acceptance sample (Section~\ref{sec:heldout}) separates what the acceptance test establishes from what the training sample already implied.
\end{enumerate}

Supporting results --- the $\varepsilon$-invariance of the reveal-rarity and its quadratic rate, the derived play-cost normalizers, the arithmetic behind knob-invariance, the step-$t$ density and the certificate's variants --- are stated where they are used and proved in the supplementary material.

Scope up front: three instruments (two 1D --- cart, pendulum --- plus a 4D bi-modal PatchField2D), two base planner families (random-shooting MPC and CEM, one fixed configuration each), 20 seeds per headline synthesis cell across the instruments (both GPT-5.x sizes; plus the caught cell on the pendulum and two bi-knob cells on PatchField2D), a 3-seed cross-family spot-check per 1D instrument (not full sweeps), and a probe-grade multilayer perceptron (MLP). Section~\ref{sec:limitations} states the limitations and the full seed accounting.

\section{The instruments}
\label{sec:instrument}

This section specifies the hand-written hybrid instruments and the planner the later experiments reuse unless they say otherwise, together with three design conditions such an instrument has to satisfy.

\subsection{Cart-with-wall}
\label{sec:cartwall}

State $(x, v)$; action $a \in [-a_{\max}, a_{\max}]$; semi-implicit Euler with fixed $dt$:
\[
a \leftarrow \mathrm{clamp}(a), \qquad
v' = v + (\mathrm{gain}\cdot a - \mathrm{drag}\cdot v)\,dt, \qquad
x' = x + v'\,dt,
\]
with defaults $dt = 0.1$, $\mathrm{gain} = 3$, $\mathrm{drag} = 0.3$, $a_{\max} = 1$. The \textbf{hybrid mode} is an inelastic wall at $x_\mathrm{wall}$: if $x' \geq x_\mathrm{wall}$, the next state is exactly $(x_\mathrm{wall}, 0)$. The \textbf{blind model} is the same code path with the wall branch removed --- the hand-written on-manifold proxy for a CWM synthesized from a spec that omits the wall, \emph{bit-exact off-mode by construction} (tested: exact float equality on wall-free trajectories). We call such a model \textbf{mode-blind} (the canonical term; instrument-specific synonyms like ``wall-blind'', or ``blind'' for short, mean the same thing). Reward is two sigmoid plateaus: a small reachable one on the left ($0.3$, at $x \leq -6$) and a large one on the right ($1.0$, at $x \geq 12$) whose approach every swept wall position blocks. Episodes are 80 steps from $x_0 \sim U(-0.5, 0.5)$, $v_0 = 0$.

The design requirements mirror the companion paper's material-at-cap instrument (paper 1's discrete instrument: a rare resource-cap rule likewise omittable from sampled play): (a) random rollouts rarely fire the mode --- \emph{rarity} throughout this paper is the probability that a whole gate \emph{rollout} contains at least one mode contact, not the volume or measure of the mode region, and the wall position is the \textbf{rarity knob} ($r$ sweeps $0.317 \to 0.0024$ over $x_\mathrm{wall} \in [2, 10]$; Table~\ref{tab:danger}); (b) the omission is \emph{exploited}, not merely mispredicted --- the wall-blind model predicts coasting through the wall toward the large plateau, so the planner it advises drives right and is pinned; (c) the truth planner's optimal play differs qualitatively --- it goes left.

The second instrument (pendulum with a hard angular stop; nonlinear gravity term) is described with its results in Section~\ref{sec:pendulum}, and the third (PatchField2D, a 4D bi-modal patch field) in Section~\ref{sec:patch2d}.

\subsection{Planner and play\_cost}
\label{sec:planner}

The planner is model-predictive control by random shooting \citep{nagabandi2018neural,chua2018deep}: at each step, sample candidate action sequences (piecewise-constant blocks plus the three constant sequences $\{-a_{\max}, 0, +a_{\max}\}$), roll each out on the \emph{model}, take the best first action, replan. The planner is a deterministic function of its model's responses and the seed, so the play-cost upper bound via query-hit mass applies verbatim (Section~\ref{sec:theory}; Corollary~\ref{cor:playcost} ties that bound to the normalization below). Single-agent control makes play\_cost a normalized regret, cleaner than a two-player arena (no opponent confound):
\[
\mathrm{play\_cost} \;=\; \frac{J_{\mathrm{truth}} - J_{\mathrm{model}}}{J_{\mathrm{truth}} - J_{\mathrm{rand}}},
\]
where $J_{\mathrm{truth}}$, $J_{\mathrm{model}}$, and $J_{\mathrm{rand}}$ are the returns of the \textbf{truth planner} --- the \emph{same} random-shooting MPC given the true dynamics, which is a reference and \emph{not} a proved optimal policy, so $J_{\mathrm{truth}}$ is a benchmark rather than a maximum --- the model-planner (written $J_{\mathrm{blind}}$ in the tables when the model is the blind one), and the uniform-random policy, all measured in the true environment on paired seeds. The pairing is explicit and reused deliberately: an episode's seed is $900{,}000 + 1000i$ in the synthesis arms and $1000i$ in the sweeps, so the \emph{same} episode seeds recur across knobs, across instruments and across both planner families, and the truth, model and random arms of a row share them. Every interval we report on a play quantity is therefore paired (a bootstrap or a randomization test over seed triples, not two independent samples), which is what makes differences of $10^{-4}$ in play\_cost meaningful at $20$ episodes. The blind planner can score below the uniform-random policy, so the normalized value can exceed $1$; we report it unclamped, always beside the raw returns it is computed from, and Section~\ref{sec:mechanism} measures how much of that comparison is carried by the random baseline's tail.

\subsection{Three design conditions on a continuous danger instrument}
\label{sec:lessons}

Each is a condition on the instrument, not a tuning choice:
\begin{enumerate}
\item \textbf{I.i.d.\ per-step candidate sampling silently removes the mode from imagination.} With i.i.d.\ candidates, imagined displacement is diffusive; no sampled sequence reaches distant reward within the horizon, so the truth model and the blind model rank all candidates \emph{identically} and the wall never enters imagination --- the two arms become indistinguishable not because the models agree but because the planner never queries where they differ. Piecewise-constant blocks plus constant candidates fix it. The planner's \emph{query} distribution, not just its trajectory distribution, is part of the instrument.
\item \textbf{Point (Gaussian) reward lodes demand braking finesse that random shooting lacks}; sigmoid plateaus remove the parking problem and give clean imagined-value margins in both directions.
\item \textbf{The plant's drag time-constant must sit well inside the planning horizon}, or no arm can act on the reward at all (at $\tau = 1/\mathrm{drag} = 10$\,s against a 3\,s horizon, nothing moves).
\end{enumerate}

\section{Theory: what transfers, what changes, what is new}
\label{sec:theory}

Notation: a verification gate draws $N$ i.i.d.\ rollouts from the gate policy $\rho$ (uniform-random actions from the initial-state distribution) and accepts a model $\hat f$ if it matches the truth $f$ within $\varepsilon$ in sup-norm on every visited transition. Two gates of that form appear in this paper and must not be confused. A \textbf{sample-consistency gate} scores the artifact on the very sample it was synthesized and refined against --- the protocol of the companion paper and of every committed synthesis run here --- so passing it means agreement with the data the artifact has already seen. A \textbf{held-out gate} scores it on an independent acceptance sample (Section~\ref{sec:heldout}). We reserve \emph{sound} for a stated logical property and never use it as a synonym for either.

\paragraph{Measure-space setup.} Fix a horizon $T$. A trajectory is a point of $(S \times A)^T$, equipped with the product Borel $\sigma$-algebra. The gate policy $\rho$ together with the transition kernel induced by $f$ defines a trajectory law $P_\rho$ on this space (the initial-state distribution pushed forward through $\rho$ and $f$). The \emph{query-hit probability} $q_{\mathrm{hit}}(E)$ is the probability, under the planner's trajectory law, that at least one model query during the episode lands in $E$. It is a hitting probability --- monotone and subadditive in $E$, but not a measure --- and we name it accordingly rather than calling it a mass, as the companion paper did. We assume throughout that the critical region $R$, the disagreement region $E$, and the reward and dynamics maps are Borel measurable, so every probability below is well defined. With this in place the discrete arguments of \citet{aguilar2026verified} transfer with only notational change.

The paper compares instruments whose modes differ ``in dimension'', and two different integers are involved. Keeping them apart is what makes the comparison an identification rather than a slogan.

\begin{definition}[trigger arity and entry-barrier dimension]
\label{def:arity}
Write $\Phi$ for the plant's \emph{unclamped} one-step map, so that $\Phi(s,a)$ is the \emph{landing} state the integrator alone would produce, and let the hybrid mode be $M = \{(s,a) : g(\Phi(s,a)) \geq 0\}$ for a predicate $g$ on the landing state. Two integers describe $M$, and neither determines the other:
\begin{itemize}
\item the \textbf{trigger arity} $p$ is the number of landing coordinates on which $g$ depends non-trivially --- what a synthesizer must identify in order to write the rule at all;
\item the \textbf{entry-barrier dimension} $b$ is the dimension of $\{g \circ \Phi = 0\}$ as a subset of the landing \emph{position} space, intersected with the reachable set --- what a deployment-time fence has to cover.
\end{itemize}
The cart's wall, $g(x') = x' - x_\mathrm{wall}$, has $p = 1$ and $b = 0$ (a single point in a one-dimensional position space); the pendulum's angular stop likewise $p = 1$, $b = 0$. PatchField2D's disc, $g(x',y') = R^2 - \|(x',y') - c\|^2$, has $p = 2$ and $b = 1$ (a circle). A slab in the same plane, $g(x') = R - |x' - c_x|$, has $p = 1$ and $b = 1$: arity of a clamp, barrier of a region.
\end{definition}

\begin{remark}[the two dimensional readings concern different integers]
\label{rem:twodims}
Section~\ref{sec:patch2d-synthesis} finds that \emph{repair from data} degrades from the 1D clamps to the 2D regions, and Corollary~\ref{cor:fencedim} that \emph{fencing cost} grows exponentially in the boundary's dimension. Those are statements about $p$ and about $b$. The cart-to-disc contrast moves both at once ($p: 1 \to 2$, $b: 0 \to 1$), which is exactly why ``1D versus 2D'' is ambiguous as a causal claim and why a slab instrument --- $p$ of a clamp, $b$ of a region --- separates them. Neither integer is the dimension of the mode boundary in the joint state--action space: that is a hypersurface in every instrument here, and nothing in this paper claims otherwise.
\end{remark}

\begin{proposition}[gate miss; Proposition~1 of \citealp{aguilar2026verified}, transferred]
\label{prop:gatemiss}
Let $R$ be any measurable set of rollouts (``the critical event''; here: the rollout fires the wall mode) with $r = P_\rho(R)$. The probability that $N$ i.i.d.\ gate rollouts all avoid $R$ is exactly $(1-r)^N$.
\end{proposition}

Proved as Proposition~1 of the companion paper \citep{aguilar2026verified}; nothing in that proof uses discreteness, the event being Bernoulli($r$) with i.i.d.\ draws.

\begin{proposition}[joint gate miss: bracketed, not factored]
\label{prop:jointmiss}
Let $R_1, R_2$ be measurable critical events (two modes) with $r_i = P_\rho(R_i)$ and $r_\cup = P_\rho(R_1 \cup R_2)$. The probability that $N$ i.i.d.\ gate rollouts avoid \emph{both} is exactly $(1-r_\cup)^N$, and with no independence assumption
\[
  \bigl(1 - \min(1,\, r_1+r_2)\bigr)^N \;\leq\; (1-r_\cup)^N \;\leq\; \bigl(1 - \max(r_1, r_2)\bigr)^N,
\]
a bracket in which the product $\bigl((1-r_1)(1-r_2)\bigr)^N$ also lies. Assume $r_\cup < 1$ and $N \geq 1$ for the ratio below (at $N = 0$ both sides are $1$ whatever the dependence). The product equals the truth iff $P_\rho(R_1 \cap R_2) = r_1 r_2$, and in general
\[
  \frac{\bigl((1-r_1)(1-r_2)\bigr)^N}{(1-r_\cup)^N}
  = \left(1 + \frac{r_1 r_2 - P_\rho(R_1 \cap R_2)}{1-r_\cup}\right)^{\!N},
\]
so the product \emph{over}-estimates the joint miss probability under negative dependence ($P_\rho(R_1 \cap R_2) < r_1 r_2$) and \emph{under}-estimates it under positive dependence.
\end{proposition}

\begin{proof}
Proposition~\ref{prop:gatemiss} applied to the event $R_1 \cup R_2$ gives the exact value $(1-r_\cup)^N$. Monotonicity gives $r_\cup \geq \max(r_1,r_2)$ and subadditivity $r_\cup \leq \min(1, r_1+r_2)$; since $x \mapsto x^N$ is increasing on $[0,1]$, the bracket follows. For the product, $(1-r_1)(1-r_2) = 1 - r_1 - r_2 + r_1 r_2 \geq \max(0,\, 1 - r_1 - r_2)$ --- the $\max$ matters when $r_1 + r_2 > 1$, where the left-hand bracket end is $0$ --- and $\leq 1 - \max(r_1,r_2)$ because $r_1 r_2 \leq \min(r_1,r_2)$. Finally inclusion--exclusion gives $1 - r_\cup = 1 - r_1 - r_2 + P_\rho(R_1 \cap R_2)$, hence $(1-r_1)(1-r_2) - (1-r_\cup) = r_1 r_2 - P_\rho(R_1 \cap R_2)$, which yields the displayed ratio and its sign.
\end{proof}

The bracket is \emph{sharp} --- its ends are the Fr\'echet--Hoeffding bounds for $P_\rho(R_1 \cup R_2)$ given the marginals, so no bound in $r_1, r_2$ alone can be tighter --- and the product form cannot be rescued by a fixed correction: measured at $50{,}000$ rollouts per knob, the sign of the dependence changes across the grid, negative at two knobs and positive at another with non-overlapping intervals. A stratified gate does not buy the product back either. Section~\ref{sup:multimode} gives the sharpness argument, the stratification result and the measurements.

\subsection{The estimand: danger as an expected risk, and which factor of it is exact}
\label{sec:estimand}

The product $\mathrm{play\_cost} \times (1-r)^N$ is quoted throughout as ``the danger law''. Two different objects can hide behind it, and only one of them is a probabilistic estimand, so we define it before using it.

Fix a synthesis procedure $A$ mapping a training sample to a model, a gate that accepts or rejects on an acceptance sample, and write $\mathrm{PC}(\hat f)$ for the play cost of a model. The \textbf{danger at budget $N$} is the expected play cost of whatever the pipeline ships:
\[
  D_N \;=\; \mathbb{E}\Bigl[\mathrm{PC}\bigl(A(D_{\mathrm{tr}})\bigr)\;\mathbf 1\{A(D_{\mathrm{tr}}) \text{ accepted}\}\Bigr],
\]
the expectation being over the training and acceptance samples and any randomness in $A$.

\begin{proposition}[the danger factorizes exactly when the shipped cost is uncorrelated with acceptance]
\label{prop:risk}
Write $G$ for the acceptance event and $X = \mathrm{PC}(A(D_{\mathrm{tr}}))$; assume $X$ integrable and $P(G) > 0$ (if $P(G) = 0$ then $D_N = 0$ and there is nothing to factor). Then
\[
  D_N \;=\; \mathbb{E}[X \mid G]\;P(G),
\]
and the factored form $D_N = \mathbb{E}[X]\,P(G)$ --- what writing $\mathrm{play\_cost} \times (1-r)^N$ asserts --- holds if and only if $\mathrm{Cov}(X, \mathbf 1_G) = 0$; in particular it holds when $X$ is almost surely constant (a single fixed blind model) and when $X$ is independent of $G$. Constancy on $G$ \emph{alone} does not suffice for the factored form --- $X = \mathbf 1_G$ is constant on $G$ yet has $\mathrm{Cov}(X, \mathbf 1_G) = P(G)(1-P(G)) > 0$ --- but it does make the conditional display exact with a constant in front: $X = c$ almost surely on $G$ gives $D_N = c\,P(G)$ with $c = \mathbb{E}[X \mid G]$, so the product is then correct with the \emph{accepted-artifact} cost as its first factor, not the unconditional mean. With neither hypothesis, if $c_- \leq X \leq c_+$ almost surely on $G$ then
\[
  c_-\,P(G) \;\leq\; D_N \;\leq\; c_+\,P(G),
\]
and no sharper bound is available from $P(G)$ and the range alone: both ends are attained.
\end{proposition}

\begin{proof}
The first display is the definition of conditional expectation, $\mathbb{E}[X\mathbf 1_G] = \mathbb{E}[X\mid G]P(G)$. Expanding the covariance, $\mathbb{E}[X\mathbf 1_G] = \mathbb{E}[X]P(G) + \mathrm{Cov}(X,\mathbf 1_G)$, which gives the iff; a globally constant $X$ and an $X$ independent of $G$ each force the covariance to vanish, while $X = c$ on $G$ gives $\mathbb{E}[X\mathbf 1_G] = c\,P(G)$ by direct integration, no covariance statement needed. The bounds follow from monotonicity of the expectation on $G$, and are attained by $X \equiv c_\pm$ on $G$, so they cannot be improved without measuring the conditional law.
\end{proof}

The paper keeps the two readings separate throughout. On the hand-written instruments the shipped model is \emph{fixed} --- the mode-blind program, whose play cost is a property of the instrument and not of the sample --- so $X$ is constant and the product is an identity, $D_N = \mathrm{PC}(\hat f_{\mathrm{blind}})\,P(\text{gate misses})$, whose second factor is exactly $(1-r)^N$ by Proposition~\ref{prop:gatemiss}. This is the precise sense in which ``the law is exact'': \emph{the gate-miss factor} is exact, and the cost factor is an input to the identity rather than a consequence of it. In the LLM arms the artifact --- and therefore its cost --- is a function of the sample, so the product is an estimate under Proposition~\ref{prop:risk}'s hypothesis, and what we report is the conditional quantity the identity actually needs: the mean play cost \emph{of accepted artifacts}, per campaign and arm, with a block-clustered bootstrap interval (the \texttt{accepted\_play\_cost} rows of \texttt{results/\allowbreak paper2\_\allowbreak statistics.json}; on the cart's headline incomplete arm, $0.508$ over $61$ accepted draws, cluster interval $[0.344, 0.671]$).

\begin{remark}[the miss factor's form does not need the i.i.d.\ hypothesis]
\label{rem:adaptive}
For rollouts drawn by \emph{any} rule --- adaptively, from history --- the chain rule of conditional probability gives, exactly,
\[
  P(\text{no critical event in } N \text{ rollouts}) \;=\; \prod_{i=1}^{N} \bigl(1 - r_i\bigr),
  \qquad r_i = P\bigl(E_i \mid \text{no } E_1, \dots, E_{i-1}\bigr),
\]
and interval knowledge $\ell_i \leq r_i \leq u_i$ along miss histories gives the sharp bracket $\prod_i (1-u_i) \leq P(\text{miss}) \leq \prod_i (1-\ell_i)$. The i.i.d.\ exponent $(1-r)^N$, the independent-gate sum $N_{\mathrm{tr}} + N_{\mathrm{g}}$ of Proposition~\ref{prop:twofactor}, a stratified schedule's product, and the two-mode Fr\'echet bracket of Proposition~\ref{prop:jointmiss} are corollaries of this one display (formalized and tested against enumerated dependent processes in \texttt{scripts/\allowbreak danger\_\allowbreak law\_\allowbreak h4.py}, \texttt{results/\allowbreak danger\_\allowbreak law\_\allowbreak h4.json}). What dependence or adaptivity removes is not the identity but its \emph{identification}: the $r_i$ are conditional hazards, and unconditional marginal event rates alone no longer determine the product.
\end{remark}

\paragraph{An independent acceptance sample.} A gate that scores the artifact on the sample it was refined against certifies consistency with its own training data. Separating the two samples leaves the law's central event with an exponent that depends on the two budgets only through their sum --- the content of the next proposition, together with what that does and does not imply.

\begin{proposition}[gate miss with an independent acceptance sample]
\label{prop:twofactor}
Let $D_{\mathrm{tr}}$ and $D_{\mathrm{g}}$ be independent samples of $N_{\mathrm{tr}}$ and $N_{\mathrm{g}}$ i.i.d.\ gate-policy rollouts; let the pipeline synthesize and refine on $D_{\mathrm{tr}}$ and accept the artifact iff it matches the truth within $\varepsilon$ on every transition of $D_{\mathrm{g}}$; and let $R$ be the critical event of Proposition~\ref{prop:gatemiss}, with $r = P_\rho(R)$. Suppose that \textbf{(i)} whenever $D_{\mathrm{tr}}$ contains no rollout in $R$, the artifact $A(D_{\mathrm{tr}})$ is mode-blind, and \textbf{(ii)} a mode-blind artifact fails the acceptance sample if and only if that sample contains a rollout in $R$. Then the danger law's central event --- a mode-free training sample yielding a (by (i)) mode-blind artifact that survives the gate --- has probability exactly
\[
  P\bigl(D_{\mathrm{tr}} \cap R = \emptyset,\ \text{artifact shipped}\bigr) \;=\; (1-r)^{N_{\mathrm{tr}}}\,(1-r)^{N_{\mathrm{g}}} \;=\; (1-r)^{N_{\mathrm{tr}} + N_{\mathrm{g}}} .
\]
Hypothesis (i) is one implication, not an equivalence, so mode-blind artifacts may also arise from training samples that \emph{do} contain $R$, and the total splits accordingly:
\[
  P(\text{a mode-blind artifact is shipped}) \;=\; (1-r)^{N_{\mathrm{tr}} + N_{\mathrm{g}}} \;+\; P\bigl(\text{blind},\ D_{\mathrm{tr}} \cap R \neq \emptyset\bigr)\,(1-r)^{N_{\mathrm{g}}},
\]
so the law is a lower bound on that total, with equality (for $r < 1$; at $r = 1$ both sides vanish and the clause is vacuous) if and only if blindness arises only from the miss --- the ``iff'' strengthening of (i).
\end{proposition}

\begin{proof}
On $\{D_{\mathrm{tr}} \cap R = \emptyset\}$, of probability $(1-r)^{N_{\mathrm{tr}}}$ by Proposition~\ref{prop:gatemiss}, the artifact is mode-blind by (i). Blindness is a function of $D_{\mathrm{tr}}$ and the synthesizer's randomness alone, so $D_{\mathrm{g}}$ is independent of it, and by (ii) acceptance is then $\{D_{\mathrm{g}} \cap R = \emptyset\}$, of conditional probability $(1-r)^{N_{\mathrm{g}}}$; multiplying gives the first display. For the second, split the blind event into $\{D_{\mathrm{tr}} \cap R = \emptyset\}$ --- which it contains, by (i), and on which the first display stands --- and $\{\text{blind},\ D_{\mathrm{tr}} \cap R \neq \emptyset\}$; on each, (ii) with the independence of $D_{\mathrm{g}}$ makes acceptance a further factor $(1-r)^{N_{\mathrm{g}}}$, and the two terms add. The excess term vanishes exactly when blindness implies the miss.
\end{proof}

Three things follow. First, for the event the proposition prices, the two budgets enter only through their sum: at fixed $N_{\mathrm{tr}} + N_{\mathrm{g}}$ that miss exponent is unchanged. This is a statement about one exponent and nothing else --- it does not say that synthesis quality is preserved when training rollouts are moved to the gate (a smaller $D_{\mathrm{tr}}$ changes the artifact, and with it hypothesis (i) and the excess term), nor that the total risk $D_N$, the compute, or the rate of other wrong artifacts is unchanged at a fixed budget. What the separated design buys is that acceptance is not self-confirming; what it changes is which artifacts ship when exactly one of the two samples contains the mode --- a case a sample-consistency gate cannot even represent, since it refines until it agrees with everything it has seen. Section~\ref{sec:heldout}'s experiment, note, \emph{adds} an acceptance sample to the committed artifacts ($N_{\mathrm{tr}} = 40$ against $N_{\mathrm{tr}} + N_{\mathrm{g}} = 80$) rather than redistributing a fixed budget. Second, the hypotheses are where the empirical content sits, and they are of different kinds. (ii) is nearly structural for a program with an omitted branch: it agrees with the truth exactly off the mode (Section~\ref{sec:cartwall}) and disagrees on a contact by far more than any deployment tolerance, so it fails an acceptance sample exactly when that sample contains a contact --- and Section~\ref{sec:heldout} verifies it artifact by artifact rather than assuming it. (i) is a genuine empirical premise: it is precisely what the discrete companion paper's translation-not-inference residual denies in general, and what fails on 2D regions (Section~\ref{sec:patch2d-synthesis}), which is why it appears as a hypothesis rather than being folded into the statement. Third, ``exact'' attaches to this proposition and to Proposition~\ref{prop:gatemiss}, never to the full risk $D_N$.

The gate's tolerance is not the axis this paper is about, and the reason is worth one sentence here and a proof in the supplement: because the mode-blind model agrees with the truth \emph{exactly} off the mode, the probability that a rollout reveals a disagreement at tolerance $\varepsilon$ equals the mode-firing rarity for every $\varepsilon$ below the smallest contact disagreement, and in the population the two agree at a quadratic rate --- proved for the whole semi-implicit family in Section~\ref{sup:epsflat}. Tightening $\varepsilon$ therefore cannot catch a hard mode, and loosening it does not widen the hole (measured in Section~\ref{sec:axes}).

\begin{proposition}[identifiability; Proposition~3 of \citealp{aguilar2026verified}, transferred]
\label{prop:ident}
Let $M \subseteq S \times A$ be the mode region (for the cart the clamp fires depending on $(s,a)$ through $x' \geq x_{\mathrm{wall}}$) and instantiate Proposition~\ref{prop:gatemiss} at the critical event $R = \{\text{rollouts that visit } M\}$, so that ``miss'' means exactly that no sampled transition lies in $M$. Condition on that miss event and let $\hat f_1, \hat f_2$ be any two models that agree on $(S \times A) \setminus M$. Then every score that depends on the candidate models \emph{only through their values at the sampled inputs} --- the gate, a one-step likelihood, a one-step refinement objective --- is constant across the two, so no such score can distinguish them, and preference for the correct one must come from the prior or the specification.
\end{proposition}

\begin{proof}
On the miss event the sample visits no transition in $M$, and the two models agree off $M$, so they produce identical outputs on every sampled input; a score that reads them only there is therefore the same number for both.
\end{proof}

\begin{proposition}[finite-sample disc identification is class-relative]
\label{prop:discident}
For the pinned PatchField2D integrator, each sampled transition determines its free landing point $z=\mathrm{pos}(\Phi(s,a))$ and its contact label. Fix the far patch and let
\[
  \mathcal H_{\circ}=\{D(c,R)\cup D_2^*:c\in\mathbb R^2,\ R>0\}.
\]
Writing $I$ for landings labelled inside the near patch and $O$ for landings labelled outside both patches, $(c,R)$ is sample-consistent exactly when
\[
  \max_{z\in I}\|z-c\|_2\leq R<\min_{z\in O}\|z-c\|_2.                 \tag{1}
\]
Consequently the finite sample identifies the true disc at tolerances $(\tau_c,\tau_R)$ relative to $\mathcal H_{\circ}$ if and only if every pair satisfying (1) obeys $\|c-c^*\|_2\leq\tau_c$ and $|R-R^*|\leq\tau_R$. This conclusion does not extend to an unrestricted region class: when the true near patch is convex, $\operatorname{conv}(I)$ is itself sample-consistent. Under clamp rather than freeze semantics, three non-collinear near-patch contact states lie on the boundary and determine $(c^*,R^*)$ uniquely.
\end{proposition}

\begin{proof}
The pinned integrator makes $z$ computable from $(s,a)$, and the transition records whether the mode fired, so it supplies the membership label. A disc contains all of $I$ and none of $O$ exactly when its radius lies in the interval (1), proving the version-space criterion. Since $I$ lies in the convex true disc, $\operatorname{conv}(I)$ also lies in it; every point of $O$ is outside the true disc and hence outside that hull, proving consistency of the broader alternative. In the clamp variant each contact state is the radial projection onto the circle. Three non-collinear points have a unique circumcircle.
\end{proof}

The criterion is checked block by block by \texttt{scripts/\allowbreak disc\_\allowbreak identifiability\_\allowbreak certificate.py}. Its multilevel search uses the $2$-Lipschitz feasibility gap in (1), with grid slack and a direction-uniform far-field exclusion bound, to outer-bound \emph{every} consistent $(c,R)$ rather than merely fit one. At $\tau_c=\tau_R=0.1$, it proves identification relative to $\mathcal H_{\circ}$ in $6/20$ baseline blocks, $9/20$ at the intermediate evidence dose and $18/20$ at the widest dose. It separately excludes every half-plane in $20/20$ blocks at all three doses; axis-aligned boxes are excluded in $19/20$, $19/20$ and $20/20$. The hull construction remains consistent in $20/20$ throughout, so no statement silently promotes circle-relative identification to uniqueness among arbitrary regions. In the clamp campaign, the three-point clause identifies the circle in $20/20$ blocks. The versioned certificate and its oracle tests are \texttt{results/\allowbreak disc\_\allowbreak identifiability\_\allowbreak certificate.json} and \texttt{tests/\allowbreak test\_\allowbreak disc\_\allowbreak identifiability\_\allowbreak certificate.py}.

\paragraph{A sharper obstruction under freeze semantics: the mode's own dynamics censor its interior.} Proposition~\ref{prop:ident} conditions on the mode being \emph{absent} from the sample. There is a second, structural obstruction that survives the mode being present --- and, unlike the first, it is not repaired by drawing more rollouts or by drawing an independent acceptance sample.

\begin{proposition}[freeze semantics make an entry rule unfalsifiable]
\label{prop:entryclass}
Let $\Phi$ be the plant's unclamped one-step map, $\mathrm{pos}(\cdot)$ the position coordinates, $R$ a mode region in position space, and let the truth be
\[
  f(s,a) \;=\;
  \begin{cases}
    (\mathrm{pos}(s), 0) & \text{if } \mathrm{pos}(\Phi(s,a)) \in R,\\
    \Phi(s,a) & \text{otherwise,}
  \end{cases}
\]
so that the mode freezes the mover at its \emph{previous} position, and let the initial-state distribution be supported on $\{\mathrm{pos} \notin R\}$. Then:
\textbf{(i)} every state a truth rollout visits has $\mathrm{pos} \notin R$;
\textbf{(ii)} writing $D = \{(s,a) : \mathrm{pos}(s) \notin R\}$, every transition of every truth rollout lies in $D$, so any model built from a predicate that agrees with the membership test $\mathrm{pos}(\Phi(s,a)) \in R$ on $D$ is \emph{exactly} equal to $f$ on every rollout of every length under any policy --- no sample distinguishes the two, at any tolerance and any $N$;
\textbf{(iii)} for any planner that rolls such a model forward from a reachable state, the imagined states also stay outside $R$, so the disagreement region $E$ has $q_{\mathrm{hit}}(E) = 0$ and Proposition~\ref{prop:playcost} forces $\mathrm{play\_cost} = 0$.
\end{proposition}

\begin{proof}
(i) By induction on $t$. At $t = 0$ it is the hypothesis on the initial distribution. If the step at $t$ contacts the mode then $\mathrm{pos}(s_{t+1}) = \mathrm{pos}(s_t) \notin R$ by the inductive hypothesis; if it does not, then $\mathrm{pos}(s_{t+1}) = \mathrm{pos}(\Phi(s_t,a_t)) \notin R$ by the definition of a contact. (ii) By (i) every visited $(s,a)$ lies in $D$; the two transition maps agree there, so a second induction gives identical trajectories, hence identical transitions and identical rewards. (iii) The imagined trajectory is generated by the \emph{model}, which freezes on entry exactly as the truth does, so the same induction applies to it; $E \subseteq \{\mathrm{pos}(s) \in R\}$ by (ii), which the imagined trajectory never enters.
\end{proof}

Three consequences. First, the obstruction is \emph{structural}: it is the mode's freeze semantics, not the sample size, that removes the evidence, so an independent acceptance sample --- the fix of Section~\ref{sec:heldout} --- cannot catch a model in this class either. Second, it is \emph{harmless}, by (iii): the models it contains differ from the truth only strictly inside a region no rollout and no imagination reaches. So this is the verified-and-fine cell of the axis separation, now characterised structurally rather than observed. Third, it locates what the evidence \emph{can} identify: on such an instrument a sample witnesses only \emph{entries} into the mode, so it pins the boundary as approached and says nothing about membership --- which is why an artifact can be exact on every transition it will ever be shown and still not encode the region.

Section~\ref{sec:arity} exhibits a synthesized artifact in this class, together with the check that (ii) and (iii) hold of it as stated.

The qualifier is not decoration: a score that rolls a candidate \emph{forward under itself} (a multi-step prediction loss, or any objective evaluated on model-generated states) queries the candidate at inputs the sample never contained, and those can lie in $M$. Such a score is not constant across the two candidates, so it escapes this proposition --- it does not thereby become informative, since the two candidates disagree there precisely because nothing observed says which is right, but the proposition does not cover it and we do not claim it does. This holds for \emph{any} learner --- LLM, linear regression, MLP --- a point Section~\ref{sec:smooth} instantiates empirically.

\begin{proposition}[play-cost upper bound; Proposition~2 of \citealp{aguilar2026verified}, transferred]
\label{prop:playcost}
Assume returns are normalized to $[0,1]$ (WLOG, by rescaling $J \mapsto (J - J_{\min})/(J_{\max} - J_{\min})$; equivalently carry an explicit factor $J_{\max} - J_{\min}$ throughout). Let $E = \{(s,a) : f(s,a) \neq \hat f(s,a)\}$ be the \emph{exact} disagreement region --- exactness is what the coupling needs, since off $E$ the two models must return \emph{identical} answers. For any planner that is a deterministic function of model responses and a seed, $|J(f) - J(\hat f)| \leq q_{\mathrm{hit}}(E)$, where $q_{\mathrm{hit}}(E)$ is the probability that the planner queries its model somewhere in $E$ during an episode.
\end{proposition}

\begin{proof}[Proof (coupling; Proposition~2 of \citealp{aguilar2026verified}, recorded here rather than deferred)]
Couple the two runs on a common seed. On the event that no query lands in $E$, $f$ and $\hat f$ return identical responses to every query the planner issues (they agree off $E$), so with shared rng the two runs select the same action sequence and traverse the same states, realizing the same return. That event is unambiguous: the two runs are identical up to the first query landing in $E$, so ``no query ever lands in $E$'' is a single coupling-measurable event with one probability, not two model-dependent ones. The realized returns can therefore differ only on its complement, of probability $q_{\mathrm{hit}}(E)$, where normalized returns differ by at most $1$; taking expectations gives $|J(f) - J(\hat f)| \leq q_{\mathrm{hit}}(E)$. MPC's imagined rollouts are the queries.
\end{proof}

The raw $J$ reported in the tables (e.g.\ $J_\mathrm{truth} = 17.77$) is this normalized quantity rescaled by $J_{\max} - J_{\min}$, so the bound is a statement about the normalized return. The paper's reported play\_cost, however, divides by $J_{\mathrm{truth}} - J_{\mathrm{rand}}$ (Section~\ref{sec:planner}), \emph{not} by $J_{\max} - J_{\min}$; the following corollary states the relationship rather than silently identifying the two.

\begin{corollary}[play\_cost saturation]
\label{cor:playcost}
Let $J_{\max}, J_{\min}$ be the essential supremum and infimum of the \emph{realized} return under the episode law, and let $\bar J \geq J_{\max}$, $\underline J \leq J_{\min}$ be the explicit \emph{bounds} on them proved in Proposition~\ref{prop:normalizers}. With play\_cost the signed normalization of Section~\ref{sec:planner}, Proposition~\ref{prop:playcost} in raw units ($|J(f) - J(\hat f)| \leq q_{\mathrm{hit}}(E)\,(J_{\max} - J_{\min})$) gives
\[
  |\mathrm{play\_cost}| \;\leq\; q_{\mathrm{hit}}(E)\,\frac{J_{\max} - J_{\min}}{J_{\mathrm{truth}} - J_{\mathrm{rand}}}
  \;\leq\; q_{\mathrm{hit}}(E)\,\frac{\bar J - \underline J}{J_{\mathrm{truth}} - J_{\mathrm{rand}}},
\]
the second inequality being the computable one.
\end{corollary}

\begin{proof}
Divide Proposition~\ref{prop:playcost}'s raw-unit bound by the knob-free constant $J_{\mathrm{truth}} - J_{\mathrm{rand}} > 0$, which is what the definition of play\_cost normalizes by; then $J_{\max} - J_{\min} \leq \bar J - \underline J$ by Proposition~\ref{prop:normalizers}.
\end{proof}

Three statuses must be kept apart wherever this ceiling is quoted, and Proposition~\ref{prop:normalizers} supplies only the first: a \textbf{proved bound} ($\bar J$, $\underline J$, valid at every knob and every policy), an \textbf{attained extremum} (a policy exhibited that realizes it --- the push-left trajectory does so for $\bar J$ at $x_\mathrm{wall} \leq 6$, to within $2\times10^{-7}$), and a \textbf{numerical approach} to an extremum whose attainment is not proved (the wider knobs, where the bound exceeds the best policy found by $1.0015\times$ at $x_\mathrm{wall} = 8$ and $1.0799\times$ at $10$). We write $\bar J, \underline J$ for the bounds throughout, and say ``known'' only where attainment is exhibited.

Pointwise extremes are what Proposition~\ref{prop:playcost} needs --- its coupling bounds a realized per-seed difference, not a difference of expectations --- and the distinction is not pedantic here: read as a supremum over policies of the \emph{expected} return, $J_{\max}$ would be $17.697$ on the cart (the best constant policy averaged over $x_0 \sim U(-\tfrac12,\tfrac12)$), \emph{below} the measured $J_{\mathrm{truth}} = 17.772$, whose seeds happen to have $\bar x_0 = -0.123$; the normalization would then exceed $1$ by construction.

Both normalizers are \emph{derived} rather than estimated --- explicit numbers valid at every knob and every policy --- and the knob-invariance of the reported play\_cost is an arithmetic consequence of the exploited planner's return rather than an empirical regularity (Proposition~\ref{prop:knobinv}). Both derivations, with the certificate that the truth planner's own return is knob-free, are in Section~\ref{sup:normalizers}.

There is a positive counterpart to all of this, and it is worth knowing exactly how far it reaches. For \emph{Lipschitz} pairs the gate does certify something: a partition argument on the deployed cart gate excludes any pair with $L = \max(\mathrm{Lip} f, \mathrm{Lip}\hat f) \leq 5.77$ carrying the wall's error of $4.2$, with probability at least $1-\delta$ over the gate's draws. The same measurement shows why it does not rescue sampling verification: the certified region carries $1.9\%$ of the exploited planner's queries. The certificates, their statistical accounting and that measurement are in Section~\ref{sup:coverage}.

\paragraph{What is new: the localization premise is a theorem-shaped obstruction.}

\begin{proposition}[smoothness forbids localized error]
\label{prop:lipschitz}
Let $f, \hat f : S \times A \subseteq \mathbb{R}^{d} \times \mathbb{R}^{m} \to \mathbb{R}^d$ be $L$-Lipschitz on the joint state-action space in the sup-metric on $S \times A$, with $L \in [0,\infty)$ (i.e.\ $L = \max(\mathrm{Lip}\, f, \mathrm{Lip}\, \hat f)$). Suppose $\|f(s_0,a_0) - \hat f(s_0,a_0)\|_\infty = \eta$ at some $(s_0,a_0) \in S \times A$. Fix any tolerance $\varepsilon < \eta$ and write $E_\varepsilon = \{(s,a) \in S \times A : \|f(s,a) - \hat f(s,a)\|_\infty > \varepsilon\}$. If $L > 0$, then $E_\varepsilon$ contains the open metric ball $B\!\left((s_0,a_0), \tfrac{\eta-\varepsilon}{2L}\right) \cap (S \times A)$ (the intersection matters only when $(s_0,a_0)$ lies near $\partial(S \times A)$). If $L = 0$, then $f - \hat f$ is constant and $E_\varepsilon$ is all of $S \times A$ --- the degenerate extreme of the same statement, and the reason the radius is written for $L > 0$.
\end{proposition}

\begin{proof}
$g = f - \hat f$ is $2L$-Lipschitz on $S \times A$, so for $\|(s,a) - (s_0,a_0)\|_\infty < (\eta-\varepsilon)/2L$ we have $\|g(s,a)\|_\infty \geq \eta - 2L\,\|(s,a)-(s_0,a_0)\|_\infty > \varepsilon$; the inequality is strict, so the ball is open. The $L = 0$ case is immediate: $g \equiv g(s_0,a_0)$ with $\|g(s_0,a_0)\|_\infty = \eta > \varepsilon$.
\end{proof}

\begin{corollary}[the localization budget]
\label{cor:locbudget}
In the setting of Proposition~\ref{prop:lipschitz} with $L > 0$, write $z_0 = (s_0,a_0)$ and $\rho = (\eta-\varepsilon)/2L$. Then, with no condition on where $z_0$ sits,
\[
  \mathrm{vol}(E_\varepsilon) \;\geq\; \mathrm{vol}\!\left(B_\infty(z_0,\rho) \cap (S \times A)\right),
\]
since the proposition places the intersected ball inside $E_\varepsilon$ and volume is monotone under inclusion. When the ball is not clipped by $\partial(S \times A)$ --- the interior case --- the right side is the volume of a sup-metric cube of side $2\rho$,
\[
  \mathrm{vol}(E_\varepsilon) \;\geq\; \left(\frac{\eta-\varepsilon}{L}\right)^{d+m},
\]
and, equivalently, confining the $\varepsilon$-disagreement of an $\eta$-amplitude error to a set of volume $V$ requires
\[
  L \;\geq\; \frac{\eta-\varepsilon}{V^{1/(d+m)}} .
\]
\end{corollary}

The first display is the always-valid form; what the boundary can cost is a constant, made uniform by one piece of domain regularity.

\begin{corollary}[the budget at the boundary]
\label{cor:kappabudget}
Say a domain $D \subseteq \mathbb{R}^{d+m}$ has \emph{interior-volume constant} $\kappa \in (0,1]$ at scale $r_0 > 0$ if $\mathrm{vol}(B_\infty(z,r) \cap D) \geq \kappa\,(2r)^{d+m}$ for every $z \in D$ and every $r \leq r_0$. If $S \times A$ has interior-volume constant $\kappa$ at scale $r_0$ and $\rho = (\eta-\varepsilon)/2L \leq r_0$, then
\[
  \mathrm{vol}(E_\varepsilon) \;\geq\; \kappa \left(\frac{\eta-\varepsilon}{L}\right)^{d+m},
  \qquad\text{equivalently}\qquad
  L \;\geq\; (\eta-\varepsilon)\left(\frac{\kappa}{V}\right)^{1/(d+m)}
\]
for confinement to volume $V$. For an axis-aligned box with side lengths $\ell_1,\dots,\ell_{d+m}$ this holds at scale $r_0 = \min_i \ell_i$ with the per-point refinement
\[
  \mathrm{vol}\!\left(B_\infty(z,r) \cap D\right) \;\geq\; 2^{-k}\,(2r)^{d+m}, \qquad k = \#\{i : \min(z_i - a_i,\, b_i - z_i) < r\},
\]
where $k$ counts the clipped coordinates: $k = 0$ in the interior recovers the unclipped display of Corollary~\ref{cor:locbudget} ($\kappa = 1$), a face point has $k = 1$ (half the cube survives), an edge $k = 2$, and a full corner $k = d+m$ keeps exactly one orthant. So a box has $\kappa = 2^{-(d+m)}$, attained at a corner --- the constant is sharp --- and the boundary's whole price in the equivalent form is $L \geq (\eta-\varepsilon)/(2V^{1/(d+m)})$: clipping costs at most a factor $2$ in $L$, never the budget itself. The same product formula covers half-bounded and unbounded coordinates (an unbounded coordinate never clips and contributes its full factor $2r$); on this paper's instrument domains the state coordinates are unbounded and the single scalar action lives in $[-a_{\max}, a_{\max}]$, so $k \leq 1$ and $\kappa \geq 1/2$ at every witness with $r \leq 2a_{\max}$.
\end{corollary}

\begin{proof}
The general displays follow from Corollary~\ref{cor:locbudget}'s first display and the definition of $\kappa$. For the box, the ball-domain intersection is a product of interval intersections, and coordinate $i$ contributes length $\min(z_i + r, b_i) - \max(z_i - r, a_i)$: this is $2r$ when $z_i$ is at distance at least $r$ from both endpoints, and at least $\min(\ell_i, r) \geq r$ otherwise (at least $r$ of the window $[z_i - r, z_i + r]$ lies on the side of $z_i$ with more room, using $r \leq \ell_i$). Multiplying the factors gives $2^{-k}(2r)^{d+m}$. At a corner $z = (a_1,\dots)$ with $r \leq \min_i \ell_i$ every factor is exactly $r$, so the bound is attained. (The case constants are computed from this product formula in exact rational arithmetic by \texttt{scripts/locbudget\_boundary\_constants.py} rather than asserted; \texttt{tests/test\_locbudget\_boundary.py} checks the formula against brute-force Monte Carlo volumes, the sharpness at a corner, and that the scale condition $r \leq \min_i \ell_i$ is load-bearing.)
\end{proof}

So localization is not \emph{forbidden} at finite $L$; it is \emph{purchased}, at a Lipschitz constant that grows as the region shrinks --- and only \emph{exact} localization ($V \to 0$ at fixed amplitude) requires $L \to \infty$. The purchase price is now unconditional: at an interior witness the budget is $((\eta-\varepsilon)/L)^{d+m}$, and at a boundary witness it is the same volume times a domain constant that a box keeps above $2^{-(d+m)}$. This is the precise form of the claim, and what it does not say matters as much, because four statements are easy to merge and only the first is this proposition's.

\begin{remark}[four statements that must be kept apart]
\label{rem:fourstatements}
(i) \emph{Representability.} A pair with Lipschitz constant at most $L$ cannot realize an $\eta$-amplitude error confined below the volume of Corollaries~\ref{cor:locbudget}--\ref{cor:kappabudget}; a \emph{discontinuous reset} boundary --- the kind all three instruments carry: the map jumps ($v \mapsto 0$ at the wall and the stop, the mover frozen at the patch), so no finite local constant exists --- faces no such constraint. The scope cuts both ways: a hybrid boundary that is continuous but nonsmooth (a kink, an actuation saturation) keeps a finite Lipschitz constant and still pays the budget, so the exemption is proved for the reset modes studied here, not for hybrid boundaries as a class. Compactly supported \emph{smooth} errors are not excluded either, and this paper uses one --- the $C^\infty$ bump arm of Section~\ref{sec:axes} --- so ``smooth models cannot localize'' would be false as stated; what is true is the bounded-Lipschitz, fixed-amplitude, arbitrarily-sharp form above.
(ii) \emph{Detectability.} The same geometry forces a gate reveal rate (Proposition~\ref{prop:detectrate}): at fixed amplitude, a smoother error is \emph{easier} for the gate to find, not harder --- measured at $0.18$ against the wall's $0.14$ in Section~\ref{sec:axes}.
(iii) \emph{Exactness at a tolerance.} Agreeing off the mode at $\varepsilon = 10^{-9}$ is a float-level property, which a program inherits from a pinned integrator by construction and a fitted model attains only up to its optimization residual. That is a claim about representation classes, not about either of the above.
(iv) \emph{Learnability from a finite sample.} What a given learner does with a given sample is an empirical question. Section~\ref{sec:smooth} answers it for two learners --- a closed-form linear fit and a probe-grade $h{=}8$ MLP --- and those measurements license no conclusion about modern learned dynamics models, whose floor would be lower and whose structural inability to be bit-exact off a mode is an argument rather than a measurement (Section~\ref{sec:limitations}).
\end{remark}

The disagreement region $E_\varepsilon$ lives in the same joint state-action space as the exact disagreement region $E$ of Proposition~\ref{prop:playcost}, and satisfies $E_\varepsilon \subseteq E$ for every $\varepsilon \geq 0$, so the ball is \emph{directly} a lower bound on the query-relevant region the planner can hit --- no fixed-action slice argument is needed to bring the two together. The inclusion is the only relation the two need: Proposition~\ref{prop:playcost} requires exact agreement off $E$, so it cannot be restated at tolerance $\varepsilon$ without a perturbation term (a $\varepsilon$-sized answer difference can flip an argmax over 200 candidates and decouple the runs). (Restricting to a fixed action $a_0$ recovers the state-slice form: the open ball of radius $(\eta-\varepsilon)/2L$ in $S \times \{a_0\}$ lies in $E_\varepsilon$, the version convenient when reasoning about a single instrument's mode.)

Read as its contrapositive: a model error that is \emph{large somewhere} (large enough to matter at play) and \emph{invisible at tolerance $\varepsilon$ outside a metrically tiny region} requires $L$ large --- in the limit, a discontinuity. This is the sense in which the discrete localization premise (``the wrong model is exact off the rule region'') is discrete: in its exact form it is unsatisfiable at any finite Lipschitz bound, and in its approximate form it costs the volume of Corollaries~\ref{cor:locbudget}--\ref{cor:kappabudget}. Two consequences structure the paper: (i) the natural continuous home of the danger geometry is \textbf{hybrid dynamics with discontinuous resets} --- inelastic contacts, hard stops, regime switches that jump the state --- where the truth itself has unbounded local Lipschitz constant across the mode boundary (our wall: $v$ jumps to 0), and an omitted mode is precisely a rare rule; a continuous hybrid boundary (a kink, a saturation) keeps a finite constant and stays inside the budget, so this home is the reset class and not hybrid dynamics wholesale; (ii) exact off-mode agreement is available \emph{representationally} to programs (an omitted or added branch is bit-exact off the omitted case) and to a bounded-Lipschitz model only at that volume price, a contrast Section~\ref{sec:smooth} measures on two learners.

The geometry also has a measure counterpart --- smoothness does not merely forbid exact localization, it forces a gate detection rate, at a Lipschitz constant that must grow like $N^{1/(d+m)}$ to keep hiding. That proposition, the closed-form visitation density it consumes, and the three scope conditions that decide what it says about \emph{this} instrument are in Section~\ref{sup:detect}. The short version is the one the measurements confirm: at comparable amplitude the smooth error is if anything the more detectable of the two, and its harmlessness comes from play cost rather than from hiding.

\section{The mechanism and the threshold law}
\label{sec:mechanism}

This section measures the danger law's full phenomenology on the hand-written instruments, before any LLM enters the loop. Every ``danger'' column below is the product $\mathrm{play\_cost} \times (1-r)^N$, which is the risk of Proposition~\ref{prop:risk} exactly when that proposition's covariance hypothesis holds; here the artifact is hand-written and fixed, so it holds by construction, and the qualifier only becomes live in the LLM arms of Section~\ref{sec:synthesis}. All numbers CPU-only, with the hand-written blind model as the on-manifold proxy (\texttt{scripts/continuous\_reach.py}; 30{,}000 rarity rollouts and 20 MPC episodes/arm per knob; Wilson 95\% CIs \citep{wilson1927} in the versioned results). One convention holds for every table below: a printed $0$ is a \emph{censored} zero --- no occurrence in the sample, not a demonstrated impossibility --- and its content is the Wilson upper bound, which the versioned JSON carries. The ones stated inline: $0/20$ truth-planner contacts is $[0, 0.16]$, $0/300$ gate passes is $[0, 0.013]$, and $0/2000$ rarity is $[0, 0.0019]$. Every zero-valued table cell now carries its own bound, its own interval, or the note that it is a demonstrated rather than a censored zero, so this convention is a reading aid and not a load-bearing one. Uncertainty on the danger curves is propagated per factor and labelled: every rarity across the $40$ curve rows of the three instruments carries a family-wise simultaneous $95\%$ exact interval propagated monotonically through $\mathrm{play\_cost} \cdot (1-r)^N$ (\texttt{scripts/\allowbreak danger\_\allowbreak law\_\allowbreak h4.py}, \texttt{results/\allowbreak danger\_\allowbreak law\_\allowbreak h4.json}; Bonferroni, so the family coverage holds under the rows' dependence, and the exponent identities the propagation relies on are the corollaries of Remark~\ref{rem:adaptive}), and on the cells with committed paired episode triples --- cart $x_\mathrm{wall}=8$, pendulum $\theta_\mathrm{stop}=1.4$, and PatchField2D $k=(3,7)$ with its three critical-event curves --- the band combines \emph{both} estimated factors, the rarity interval and a paired play-cost bootstrap. For the remaining rows the play-cost factor is a point estimate and the band is rarity-only; the JSON records that as an explicit per-row gap rather than presenting a rarity-only band as a full danger interval.

\begin{table}[ht]
\centering
\small
\begin{tabular}{rrrrrrrrrr}
\toprule
$x_\mathrm{wall}$ & rarity & $J_\mathrm{truth}$ & $J_\mathrm{blind}$ & $J_\mathrm{rand}$ & play\_cost & blind hit & d@20 & d@40 & d@80 \\
\midrule
2 & 0.3168 & 17.77 & $1.4{\times}10^{-5}$ & 0.53 & 1.031 & 1.00 & 0.001 & $2.5{\times}10^{-7}$ & $6.0{\times}10^{-14}$ \\
3 & 0.2135 & 17.77 & $1.3{\times}10^{-5}$ & 0.53 & 1.031 & 1.00 & 0.008 & $6.9{\times}10^{-5}$ & $4.7{\times}10^{-9}$ \\
4 & 0.1352 & 17.77 & $1.9{\times}10^{-5}$ & 0.53 & 1.031 & 1.00 & 0.056 & 0.003 & $9.3{\times}10^{-6}$ \\
5 & 0.0797 & 17.77 & $6.3{\times}10^{-5}$ & 0.53 & 1.031 & 1.00 & 0.196 & 0.037 & 0.001 \\
6 & 0.0439 & 17.77 & $3.8{\times}10^{-4}$ & 0.53 & 1.031 & 1.00 & 0.420 & 0.171 & 0.028 \\
8 & 0.0114 & 17.77 & 0.02 & 0.53 & 1.030 & 1.00 & 0.818 & 0.650 & 0.410 \\
10 & 0.0024 & 17.77 & 0.94 & 0.53 & 0.977 & 1.00 & 0.930 & 0.886 & 0.804 \\
\bottomrule
\end{tabular}
\caption{The danger curve on the wall-position knob. ``blind hit'' = fraction of episodes in which the blind planner's trajectory fires the wall mode; the truth planner's trajectory fires it in no episode at any knob ($0/20$, so $<0.16$); d@$N$ = $\mathrm{play\_cost}\cdot(1-r)^N$. Rarity is measured on \textbf{30{,}000} random rollouts (Wilson 95\% CIs in the versioned JSON), the same sample size as the pendulum's, so the mode-firing and reveal-rarity estimates of the same event are now directly comparable.}
\label{tab:danger}
\end{table}

\begin{figure}[ht]
\centering
\includegraphics[width=0.55\textwidth]{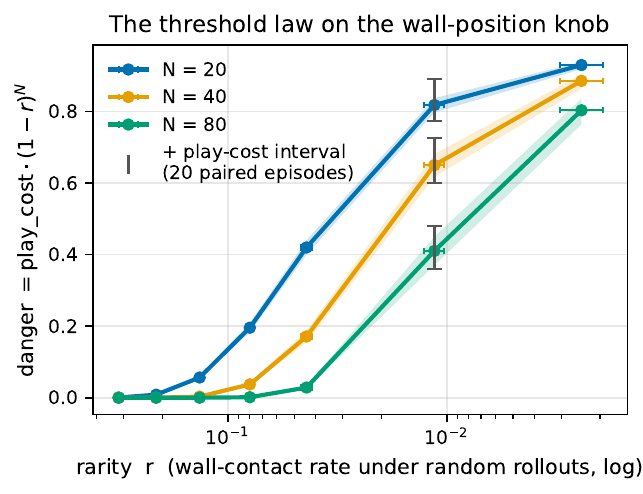}
\caption{The threshold law on the wall-position knob (Table~\ref{tab:danger} rendered): danger $\approx 0$ while the wall is inside the random envelope, rises through the elbow, plateaus at full play\_cost; $N$ shifts the threshold. Bands and horizontal bars are the rarity's Wilson interval propagated through $\mathrm{play\_cost}\cdot(1-r)^N$, which is monotone in $r$, so the transformed interval is the corner pair; a printed point with no visible band has an interval narrower than the marker.}
\label{fig:threshold}
\end{figure}

\begin{figure}[ht]
\centering
\includegraphics[width=0.55\textwidth]{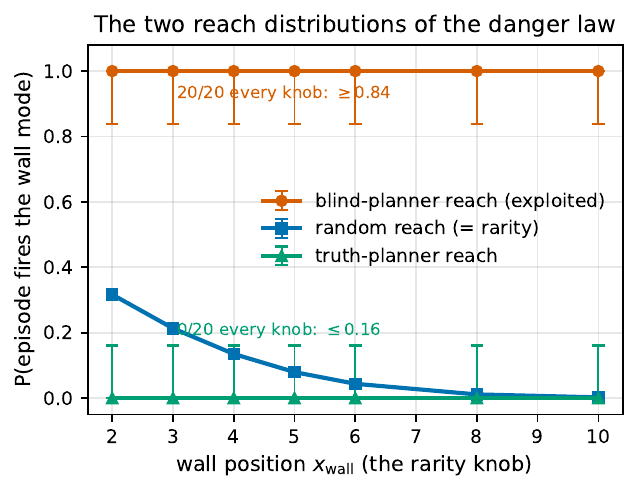}
\caption{The two reach distributions of the danger law: the exploited planner's mode reach is flat at 1.00 across the knob; random reach (= rarity) falls two orders of magnitude; the truth planner's trajectory reach is a censored zero at every knob ($0/20$ episodes) and is drawn as an upper bound rather than as a point --- the mode lives on the truth planner's \emph{query} distribution (Proposition~\ref{prop:playcost}), not its path. Error bars are Wilson intervals: binomial over $20$ episodes for the two reach curves, and over $30{,}000$ rollouts for rarity.}
\label{fig:reach}
\end{figure}

Readings: \textbf{the threshold law again} --- danger $\approx 0$ while the wall sits inside the random-rollout envelope, rises through the elbow, and plateaus at full play\_cost, with the entire elbow inside the sweep (Figure~\ref{fig:threshold}). \textbf{play\_cost is knob-invariant and $\approx 1$} (exceeding 1 at every knob except the largest, where far-plateau leakage gives 0.977): the blind planner is not merely uninformed but \emph{exploited} --- MPC on the wall-less model plans into the phantom region, is pinned at the wall (final $x = x_\mathrm{wall}$ exactly, contact rate 1.00 at every knob), and replans the same doomed plan every step, for the entire episode, ending at a return of $2\times10^{-5}$ against the truth planner's $17.77$. The 0.977 at $x_\mathrm{wall} = 10$ is the sigmoid tail of the far plateau leaking $J_\mathrm{blind} = 0.94$; the mechanism is unchanged. \textbf{The reach mechanism} appears in its cleanest form (Figure~\ref{fig:reach}), with the query/trajectory distinction now visible. Gate-miss exactness is re-verified in-tests and again at gate scale in Section~\ref{sec:axes}.

One qualification belongs with the claim. At the widest knobs the pinned planner scores \emph{above} the uniform-random policy, because the far plateau's sigmoid tail pays a planner frozen short of it; the exploitation is what the mechanism asserts, and ``below random'' is a property of the reward's shape as much as of the planner. Narrowing only the phantom plateau removes the tail and makes the strong form hold at every knob on both instruments, with play\_cost invariant to $10^{-4}$ (Section~\ref{sup:sharp}); we report the default instrument in the main tables and treat that variant as a robustness check rather than as the headline.

\paragraph{At $100$ paired episodes: the exploitation is robust, ``below random'' is not.} Twenty episodes per cell is too few for a comparison against a baseline as heavy-tailed as the uniform-random policy's, so the three headline rows were re-run at \textbf{100} paired episodes with a paired bootstrap and a randomization test (\texttt{scripts/\allowbreak play\_\allowbreak cost\_\allowbreak intervals.py}; the same seed in all three arms, $20{,}000$ resamples). The exploitation claim survives unchanged and tightens: on the cart's headline knob $\mathrm{play\_cost} = 1.014$ with a paired $95\%$ interval $[1.001, 1.033]$, on the pendulum's $0.996$ $[0.996, 0.997]$, and on PatchField2D $1.044$ $[1.020, 1.075]$ --- and the raw quantity behind them is a regret of $17.72$, $19.96$ and $19.75$ return units respectively, which is what those normalized numbers mean.

The comparison against \emph{random}, however, does not survive at face value, and the reason is the baseline rather than the planner. On the cart, $J_{\mathrm{rand}} - J_{\mathrm{blind}} = 0.249$ with interval $[0.024, 0.559]$ and a sign-flip $p = 0.040$, so the \emph{mean} difference is positive --- but $J_{\mathrm{random}}$ is heavy-tailed (median $4.7\times10^{-4}$, maximum $10.58$, skewness $6.1$: a handful of episodes in which the random walk happens to reach the left plateau), and \textbf{in $86$ of the $100$ seeds the blind planner scores \emph{above} random}. So ``below random'' on the cart is a statement about a mean that a few lucky random episodes carry, not about a typical episode, and it is stated here only in that form. At the pendulum's headline knob it is simply false at this sample size ($J_{\mathrm{rand}} - J_{\mathrm{blind}} = -0.078$, random better in $7$ of $100$ seeds) --- that knob is one the sigmoid-tail leak of the previous paragraph affects. On PatchField2D it is unambiguous: $J_{\mathrm{rand}} - J_{\mathrm{blind}} = 0.834$ $[0.395, 1.357]$, random better in $\mathbf{100}$ of $100$ seeds, $p = 1.2\times10^{-14}$.

The load-bearing claim was never the comparison with random --- it is that the planner is \emph{exploited}: pinned at the boundary in every episode, at a regret of essentially the whole available return, with $\mathrm{play\_cost} \approx 1$ and its interval excluding zero by three orders of magnitude. That is what the tables report, and what ``below random'' added was rhetorical rather than evidential. One further caution the larger sample exposes: on PatchField2D the truth planner's own return has a standard deviation of $7.9$ across seeds (median $17.0$, maximum $41.6$), because it sometimes finds a route past the patches to the phantom lode, so the per-seed play cost there is genuinely dispersed (median $0.899$, interquartile range $[0.860, 0.932]$) and its mean is not a typical episode either. Per-seed values for all three rows are in the versioned JSON, and Figure~\ref{fig:perseed} plots them, tails unsummarized.

\begin{figure}[ht]
\centering
\includegraphics[width=\textwidth]{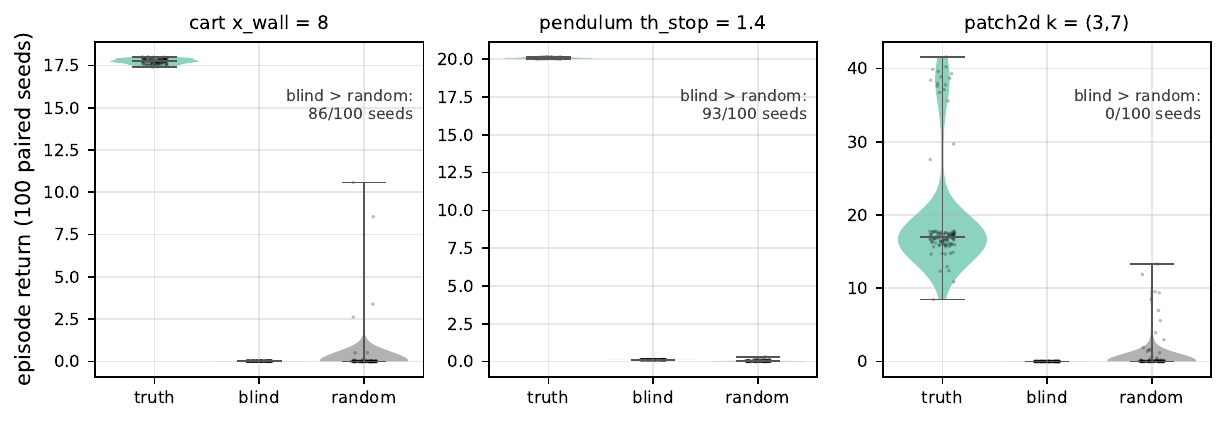}
\caption{Per-seed return distributions for the three $100$-paired-episode experiments (violins with every seed overplotted). The exploitation claim is the gap between the truth planner and the pinned blind planner, and it is total in every panel. The random baseline's heavy upper tail --- a few episodes that wander onto a reward plateau --- is what carries its \emph{mean} above the blind planner's on the cart, against $86$ of $100$ individual seeds; on PatchField2D the comparison is unambiguous in the other direction ($0$ of $100$).}
\label{fig:perseed}
\end{figure}

\subsection{Robustness: the same law on a nonlinear plant}
\label{sec:pendulum}

The cart's off-mode dynamics are linear, which is convenient for Section~\ref{sec:smooth} but invites the worry that the phenomenology depends on it. A second instrument --- a pendulum (gravity term $\sin\theta$, $\theta = 0$ hanging down) with a hard angular stop, same interface, same MPC, same two-plateau reward on $\theta$ --- reproduces the identical picture with \emph{no} re-calibration (rarity is natural here: gravity confines the random walk near the bottom, so climbing to the stop is rare):

\begin{table}[ht]
\centering
\small
\begin{tabular}{rrrrrrr}
\toprule
$\theta_\mathrm{stop}$ & rarity & $J_\mathrm{truth}$ & $J_\mathrm{blind}$ & play\_cost & blind hit & d@40 \\
\midrule
0.8 & 0.2876 & 20.08 & 0.01 & 1.002 & 1.00 & $1.3{\times}10^{-6}$ \\
1.0 & 0.1270 & 20.08 & 0.03 & 1.002 & 1.00 & 0.004 \\
1.2 & 0.0509 & 20.08 & 0.05 & 1.000 & 1.00 & 0.124 \\
1.4 & 0.0196 & 20.08 & 0.12 & 0.997 & 1.00 & 0.452 \\
1.6 & 0.0070 & 20.08 & 0.26 & 0.990 & 1.00 & 0.748 \\
2.0 & 0.0007 & 20.08 & 1.23 & 0.942 & 1.00 & 0.917 \\
\bottomrule
\end{tabular}
\caption{Pendulum-with-stop: threshold law, knob-invariant exploitation (pinned at the stop in every episode at every knob), truth planner untouched by the mode. The mechanism does not care that the plant is nonlinear --- only that the mode is hard and rare under the gate's measure. Rarity is the fraction of \textbf{30{,}000} random rollouts that fire the stop (Wilson 95\% CIs for every row in the versioned JSON). The sample was raised from 3000 specifically because the widest knob was a censored zero there ($0/3000$); at 30{,}000 it resolves to $0.0007$ $[0.0004, 0.0010]$, so every row is now a point estimate and every d@40 a number rather than a bound.}
\label{tab:pendulum}
\end{table}

\subsection{Two modes at once: a 4D bi-modal instrument}
\label{sec:patch2d}

Both instruments so far are one hard boundary in a 2D state. \textbf{PatchField2D} closes both structural gaps at once --- a 4D state and two distinct modes (their contact events are \emph{not} independent, and Section~\ref{sec:patch2d} measures the dependence) --- and asks what the 1D instruments cannot: does the danger law compose mode-wise? State $(x, y, v_x, v_y)$; a \emph{scalar} action $a \in [-a_{\max}, a_{\max}]$ is mapped to a heading $\phi = \pi\,\mathrm{clamp}(a)/a_{\max}$ so all planner machinery (which consumes scalar actions) is reused unchanged; semi-implicit Euler:
\[
v_x' = v_x + (\mathrm{gain}\cos\phi - \mathrm{drag}\,v_x)\,dt, \quad
v_y' = v_y + (\mathrm{gain}\sin\phi - \mathrm{drag}\,v_y)\,dt,
\]
$x' = x + v_x'\,dt$, $y' = y + v_y'\,dt$. The two \textbf{modes} are circular sticky patches $P_i = \mathrm{disc}(c_i, R)$: if $(x', y') \in P_i$ the next state is $(x, y, 0, 0)$ --- the probe stops inelastically at its \emph{previous} position with zero velocity (the 2D analogue of the wall clamp, and the structure the 2D mitigation exploits, Section~\ref{sec:patch2d-mitigation}). Reward is two radial sigmoid lodes, a small real one near the start and a large phantom one behind the patches; the mode-blind planner aims straight at the phantom and freezes at a patch edge. The two knobs are the patch centers' distances along and off the start-to-lode corridor ($k_1$ the near patch, $k_2$ the far one); defaults are frozen once, never tuned per cell.

Table~\ref{tab:patch2d} sweeps the $3\times3$ knob grid (\texttt{scripts/continuous\_patch2d.py}; 600 rarity rollouts and 20 MPC episodes/arm per cell). The per-mode rarities separate cleanly ($r_1 \in [0.085, 0.245]$, $r_2 \in [0.0067, 0.0100]$), $J_\mathrm{blind} = 0$ at every cell (the blind planner freezes at a patch edge every episode; $J_\mathrm{rand} = 0.11$), and play\_cost is knob-invariant at $[1.005, 1.006]$ --- the same below-random exploitation as the 1D instruments, now on a 4D plant. \textbf{The danger law applies per mode and to the pair --- but it does not \emph{factor}.} Proposition~\ref{prop:gatemiss} holds for any measurable critical event, so the per-mode gate-miss factors are $(1-r_i)^N$ and the joint one is $(1-r_\cup)^N$, where $r_\cup$ is the probability that a rollout contacts \emph{either} patch (Proposition~\ref{prop:jointmiss}). We \emph{measure} $r_\cup$ rather than compose it: writing the joint factor as the product $(1-r_1)^N(1-r_2)^N$ would assume the two per-mode contact events are independent within a rollout, and they are not. At 600 rollouts $P(\text{both})$ is only $0$ to $3$ counts per knob, so the product's $-17\%$ to $+12\%$ relative error across this grid is not by itself evidence of anything: six of the nine cells are censored zeros, which force the product to over-estimate by construction. The dependence was therefore re-measured at $50{,}000$ rollouts on four knobs (\texttt{scripts/\allowbreak patch2d\_\allowbreak dependence\_\allowbreak 50k.py}), where it \emph{is} resolvable and does change direction: negative at $(2,6)$ and $(3,7)$, positive at $(4,6)$, all three with Wilson intervals excluding $r_1r_2$, and undecided at $(4,7)$. The mechanism is visible in the sign: a rollout that freezes at the near patch has spent its travel (negative), while reaching a far patch usually means passing the near one (positive). Both columns are in the versioned JSON (\texttt{d40\_joint} and \texttt{d40\_joint\_indep\_approx}); the table reports the measured one. Three qualifications: the $r_2$ knob is only weakly resolved at 600 rollouts (its trend is under-resolved, not flat), the $k_1{=}4$ one-hit bump in $r_2$ is sampling noise rather than a knob effect, and the four-decimal \texttt{d@40} figures in the $P_2$ and joint columns inherit that resolution --- at $4$ to $6$ hits in $600$ the interval on $\mathrm{d@40}\,P_2 = 0.7700$ is $[0.51, 0.91]$, so those digits are bookkeeping, not precision.

\begin{table}[ht]
\centering
\small
\begin{tabular}{rrrrrrrrrr}
\toprule
$k_1$ & $k_2$ & $r_1$ & $r_2$ & $r_\cup$ & $J_\mathrm{truth}$ & play\_cost & d@40 $P_1$ & d@40 $P_2$ & d@40 joint \\
\midrule
2 & 6 & 0.245 & 0.010 & 0.252 & 17.98 & 1.006 & $1.3{\times}10^{-5}$ & 0.673 & $9.3{\times}10^{-6}$ \\
2 & 7 & 0.245 & 0.008 & 0.253 & 18.02 & 1.006 & $1.3{\times}10^{-5}$ & 0.720 & $8.5{\times}10^{-6}$ \\
2 & 8 & 0.245 & 0.007 & 0.252 & 18.08 & 1.006 & $1.3{\times}10^{-5}$ & 0.770 & $9.3{\times}10^{-6}$ \\
3 & 6 & 0.142 & 0.008 & 0.150 & 18.57 & 1.006 & 0.002 & 0.720 & 0.002 \\
3 & 7 & 0.142 & 0.008 & 0.150 & 18.47 & 1.006 & 0.002 & 0.720 & 0.002 \\
3 & 8 & 0.142 & 0.007 & 0.148 & 17.72 & 1.006 & 0.002 & 0.770 & 0.002 \\
4 & 6 & 0.088 & 0.008 & 0.092 & 20.85 & 1.005 & 0.025 & 0.719 & 0.021 \\
4 & 7 & 0.088 & 0.010 & 0.095 & 19.52 & 1.006 & 0.025 & 0.673 & 0.019 \\
4 & 8 & 0.085 & 0.008 & 0.093 & 19.08 & 1.006 & 0.029 & 0.720 & 0.020 \\
\bottomrule
\end{tabular}
\caption{PatchField2D mechanism sweep ($3\times3$ knobs; 600 rarity rollouts, 20 MPC episodes/arm; $J_\mathrm{blind} = 0$, $J_\mathrm{rand} = 0.11$ at every cell). d@40 $P_i$ = $\mathrm{play\_cost}\cdot(1-r_i)^{40}$ (per-mode danger); $r_\cup$ = measured probability that a rollout contacts either patch; d@40 joint = $\mathrm{play\_cost}\cdot(1-r_\cup)^{40}$, the danger law at the joint critical event --- \emph{not} the product $((1-r_1)(1-r_2))^{40}$, which would assume within-rollout independence and errs by $-17\%$ to $+12\%$ here. Full table (with $J_\mathrm{blind}$, $J_\mathrm{rand}$, hit rates) in \texttt{docs/EXPERIMENTS.md}.}
\label{tab:patch2d}
\end{table}

\paragraph{Play cost is planner-dependent --- measured on two families, one configuration each --- and the bound says which direction is forced.} Proposition~\ref{prop:playcost} bounds play cost by the planner's query-hit probability on the disagreement region, so low query reach forces low play cost while high reach merely permits high cost. A second base planner family (the cross-entropy method, one fixed configuration) sits at the low-reach end on all eleven knobs of both 1D instruments and on PatchField2D, with play cost statistically indistinguishable from zero and imagined boundary-crossing strictly below random-shooting MPC's. Limited reach is not knowledge, and the same search that misses phantom reward can miss real reward; the rows, the censoring of the two zero-crossing cells and the caveats are in Section~\ref{sup:cem}.

\section{Axis separation: which errors the gate catches at \texorpdfstring{$\varepsilon = 10^{-2}$}{eps = 1e-2}, and which it misses}
\label{sec:axes}

The classic continuous-model failure axis is pervasive sub-tolerance error; the danger law's axis is a localized hard mode. A tolerance gate must be shown to fail \emph{only} on the second axis, and only at the $(1-r)^N$ rate.

\paragraph{What fixes $\varepsilon$, and why the axis separation does not depend on it.} We use $\varepsilon = 10^{-2}$ as a \textbf{representative tolerance}, not a deployment-derived one: these are deterministic simulators with no sensor noise and float64 arithmetic, so there is no physical error budget to read it off. What the value does is sit in the middle of the only two scales the instrument supplies --- three orders of magnitude above the integrator's float noise ($\sim\!10^{-16}$, and $10^{-9}$ is where the pinned-integrator gate of Section~\ref{sec:synthesis} runs), and between two and three orders \emph{below} the mode's own disagreement ($4.2$ on the cart) --- while amounting to about $0.06\%$ of the $18$-unit position span between the two reward plateaus and $0.1\%$ of the reachable speed range. The claim that matters is that the choice is immaterial: the $\varepsilon$-sweep below varies it over eight orders of magnitude, $10^{-9}$ to $0.3$, and the axis separation is unchanged throughout. Likewise $N = 40$ is inherited from the companion paper's sweep so the two studies' counts are comparable; it is not derived from a verification cost model, and the tables report $\mathrm{d@}N$ at $N \in \{20, 40, 80\}$ so the threshold's dependence on it is visible rather than assumed away. Five arms, one table (\texttt{scripts/continuous\_axes.py}; reveal-rarity $=$ P(a random rollout contains a transition where truth and model differ $> \varepsilon$), 20{,}000 rollouts; pass@40 over 300 independent $N{=}40$ gates; 20 MPC episodes/arm). For the hard mode, the mode-firing rarity of Tables~\ref{tab:danger}--\ref{tab:pendulum} and the reveal-rarity here coincide as events --- the mode error exceeds every $\varepsilon$ in the swept range (the $\varepsilon$-sweep below), so a rollout fires the mode iff it reveals a disagreement --- which licenses the shared symbol $r$.\footnote{With both samples raised, the two estimates of the same event now agree to the third decimal: mode-firing rarity $0.1352$ on 30{,}000 rollouts (Table~\ref{tab:danger}) against reveal-rarity $0.1351$ on 20{,}000 (Table~\ref{tab:axes}). At smaller samples they read $0.1430$ and $0.1385$; the gap was resampling noise.}

\begin{table}[ht]
\centering
\small
\begin{tabular}{lrrrrr}
\toprule
arm & reveal-rarity & $(1-r)^{40}$ & pass@40 measured & play\_cost & d@40 \\
\midrule
wall@4 omitted & 0.1351 & 0.0030 & 0.003 & 1.031 & 0.0031 \\
wall@8 omitted & 0.0103 & 0.6622 & 0.667 & 1.030 & 0.6821 \\
drag bias $\times$1.03 (sub-$\varepsilon$) & 0.0001 & 0.9940 & 0.997 & $0$ (exact) & $0$ (exact) \\
drag bias $\times$2.0 (supra-$\varepsilon$) & 1.0000 & $<10^{-4}$ & $<0.013$ & $0$ (exact) & $0$ (exact) \\
$C^\infty$ bump@4, amp 0.5 & 0.1829 & 0.0003 & $<0.013$ & $0$ (exact) & $0$ (exact) \\
$C^\infty$ bump@4, amp 1.0 & 0.2018 & 0.0001 & $<0.013$ & $-0.745$ & $-0.0001$ \\
\bottomrule
\end{tabular}
\caption{Axis separation at $\varepsilon = 0.01$, $N = 40$; d@$N$ = $\mathrm{play\_cost}\cdot(1-r)^N$, as in Table~\ref{tab:danger}. Danger lives in one quadrant only: rare $\wedge$ hard-mode. Reveal-rarity is measured on \textbf{20{,}000} rollouts (raised from 2000 because the sub-$\varepsilon$ arm was a censored zero there): it now resolves to $0.0001$, and its predicted pass rate $0.9940$ sits inside the measured $[0.9814, 0.9994]$ instead of above it.}
\label{tab:axes}
\end{table}

\begin{figure}[ht]
\centering
\includegraphics[width=0.6\textwidth]{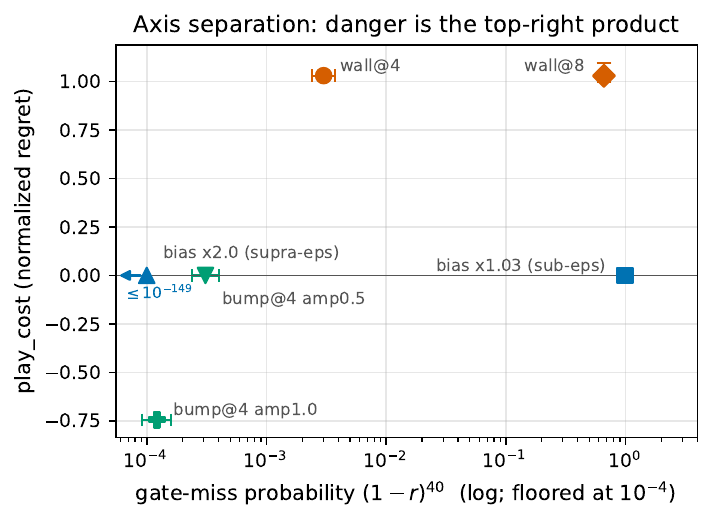}
\caption{Table~\ref{tab:axes} as the danger quadrant: gate-miss probability (log) vs play\_cost. Only the rare hard mode reaches the top right; danger is the product of the coordinates. Horizontal bars propagate each arm's rarity interval through $(1-r)^{40}$; the play\_cost axis carries a paired interval where per-episode returns exist and no bar where the arms are bit-identical, since there the difference is exactly zero rather than uncertain.}
\label{fig:axes}
\end{figure}

Readings (Table~\ref{tab:axes}, rendered as the danger quadrant in Figure~\ref{fig:axes}): \textbf{gate exactness at gate scale} --- measured pass@40 matches $(1-r)^{40}$ in both wall rows --- 0.003 vs 0.0030 and 0.667 vs 0.6622, each prediction inside the 300-gate Wilson interval. Two bookkeeping points matter for reading any of these agreements. First, $\hat r$ carries its own error, and the prediction $(1-\hat r)^{40}$ inherits it: propagating the rarity's Wilson interval, the $20{,}000$-rollout reveal-rarity $0.01025$ $[0.00895, 0.01174]$ predicts a pass rate in $[0.623, 0.698]$, which contains the measured $0.667$; on $2000$ rollouts, $\hat r = 0.0125$ $[0.0085, 0.0184]$ predicts $[0.476, 0.711]$, comfortably containing $0.667$. A point prediction like $0.6046$ compared against the measurement's interval, while its own interval goes unquoted, would misread that cell as a marginal failure; the propagated comparison is the valid one. Second, this instrument has \emph{two} rarity estimates for what Proposition~\ref{prop:epsinv} says is one event: $0.01140$ from the $30{,}000$-rollout mode-firing sweep of Table~\ref{tab:danger} (predicting $0.631$) and $0.01025$ from the $20{,}000$-rollout reveal-rarity here (predicting $0.662$). Their intervals overlap and the proposition says the events coincide below $\varepsilon^\star$, so this is sampling noise and not a contradiction --- but the $10\%$ relative gap is larger than either interval's half-width suggests one should ignore, and we quote the reveal-rarity throughout this table and the firing rarity throughout Table~\ref{tab:danger} rather than mixing them. The proposition is the observed acceptance rate, not asymptotic decoration. \textbf{The gate polices the pervasive axis}: a global drag bias above tolerance is revealed on every rollout and never accepted; a sub-tolerance bias is accepted at 0.997 --- which the law now \emph{predicts} rather than excuses: that arm does reveal a disagreement on a thin slice of rollouts (reveal-rarity $0.0001$, the extreme velocity tail), for a predicted pass rate of $0.9940$ --- and it is harmless at play. Verified-and-fine is a real cell, and the gate finds it. \textbf{Smoothness kills consequence, not detectability}: the $C^\infty$ drag bump at the wall's location has comparable rarity to the wall (0.18 vs 0.14 --- it is just as \emph{detectable}, confirming Proposition~\ref{prop:lipschitz} is about error geometry, not about hiding from the gate, and matching the direction of Proposition~\ref{prop:detectrate}) yet play\_cost 0.000 at amplitude 0.5. At amplitude 1.0 play\_cost turns \emph{negative} ($-0.745$): the truth planner, seeing the slowdown near its horizon edge, is over-pessimistic and often settles for the small plateau, while the bump-blind planner pushes through and wins. A smooth localized omission produces planner-side timing effects of ambiguous sign; only the hard mode produces the one-way exploitation geometry (pinned, forever, below random).

\textbf{On the three instruments and the swept grid, the separation does not depend on the tolerance.} A sweep over $\varepsilon \in \{10^{-9}, \dots, 0.3\}$ leaves the mode arms' reveal-rarity flat on all three instruments while the pervasive-bias arms switch sharply at their own error scale, and $\mathrm{pass@}40 \approx (1-r)^{40}$ continues to hold for the mode arms at every $\varepsilon$ in the grid (Section~\ref{sup:epssweep}). The gate's $\varepsilon$ is a pervasive-error dial, not a mode-detection dial: tightening it cannot catch the hard mode, and loosening it does not widen the hole.

\section{The exploitation is planner-mediated: a distrust-region fence collapses it on the 1D instruments}
\label{sec:mitigation}

The exploitation measured above is planner-mediated rather than model-mediated, and a planner-side fix collapses it without touching the model or the gate --- which does \textbf{not} contradict the danger law, since the gate still accepts a wrong model and Proposition~\ref{prop:ident} is untouched. Distrust-region replanning fences the positions of the model's refuted predictions and truncates any imagined trajectory that crosses a fence. On the two 1D instruments a single contact suffices to fence the mode on all eleven knob rows --- a separation fact rather than a covering one, with a measured range of validity --- and the mitigation is bit-identical to plain MPC when the model is right. On a 2D circular mode it degrades and, at the farthest knob, fails outright in $7$ of $20$ episodes, because its tie-break is an unsigned distance. The design, the packing bound on its contact cost, the dimensional reading of that bound and the 2D failure are in Section~\ref{sup:mitigation}.

\section{LLM synthesis: on the 1D clamps the danger reduces to the identifiability event}
\label{sec:synthesis}

Real-LLM arms (\texttt{scripts/continuous\_danger\_synthesis.py}; Azure GPT-5.x mini and large\footnote{\textbf{Exact models.} ``large'' is the Azure OpenAI deployment \texttt{gpt-5.4} and ``mini'' is \texttt{gpt-5.4-mini}, both at API version \texttt{2025-04-01-preview}; the cross-family arms are \texttt{Qwen/Qwen3-Coder-30B-A3B-Instruct} (Hugging Face Inference Providers router) and Claude Sonnet (agent-relayed, Section~\ref{sec:synthesis}). Every result JSON under \texttt{results/} records its own \texttt{model} field, so each number is attributable to a named deployment; the runs are dated 2026-07-07 to 2026-07-20 in the repository history. We write ``GPT-5.x'' in prose where the claim holds for both sizes.}; $N = 40$ training rollouts which double as the gate, as in the companion paper's sweep; $\varepsilon = 10^{-9}$ pinned-integrator gate; \textbf{20 seeds/cell on the headline $x_\mathrm{wall} = 8$ cell, both sizes} (the companion paper's standard); 6 MPC play episodes/seed; per-seed JSON with the synthesized code versioned in the repository). The contract pins the integrator (Section~\ref{sec:cartwall}'s equations, stated in the spec text, constants generated from the environment instance so they cannot drift); the \emph{full} arm includes the wall clause, the \emph{incomplete} arm omits it. Crucially, each seed logs whether the wall fired in its training sample --- the identifiability event that the companion paper could not condition on post hoc.

\paragraph{One design fact governs every pooled count in this paper.} The gate sample of seed index $i$ is drawn with rollout seed $10^4(i{+}1)$ \emph{independently of which model is being run} (\texttt{scripts/continuous\_danger\_synthesis.py}), so at every instrument and knob the mini and large arms are synthesized from the \emph{same} samples. Pooling the two sizes therefore doubles the number of \emph{synthesis draws}, not the number of independent gate samples: a pooled $k/2n$ is $n$ shared samples with two model draws each. Two consequences we hold to throughout. (i) Any empirical check of the gate-miss rate $(1-r)^N$ is a per-size statement --- the mode-absent/present split is a property of the sample, so it is \emph{identical} across sizes by construction and cannot be pooled at all. (ii) For the conditional outcomes (blind-and-exploited, repaired), pooling across sizes is meaningful because the synthesis draw is independent, but a Wilson interval over the pooled count assumes independence the samples do not have; we therefore quote the per-size bound as the conservative one and mark pooled bounds as such wherever they appear.

\textbf{On the two headline cells the shared-sample constraint is removed, not caveated.} A \texttt{-{}-seed-offset} run shifts the sample block ($\mathrm{rollout\_seed} = 10^4(i{+}1{+}\mathrm{offset})$), so re-running the \emph{large} arm at offset 20 gives it a block \emph{disjoint} from mini's. The headline cart cell and the headline pendulum knob therefore have three arms: mini on block $S_0$, large on $S_0$ (same samples, different model) and large on $S_{20}$ (different samples, same model) --- two orthogonal replications instead of one. Pooling mini@$S_0$ with large@$S_{20}$ is then a pool of \emph{independent} samples, and the Wilson interval over it is legitimate. Every bound below labelled ``disjoint'' is of that kind; the per-size bound is still reported for the cells where only shared blocks exist (the pendulum caught knob, PatchField2D, the ablations).

\paragraph{Full arm (both sizes, 40 seeds): gate 1.000 in 0 refinement iterations, wall probes exact, play at truth parity --- every seed.} The pinned-integrator premise holds with a real LLM: correct synthesis is float-exact through the sandbox, so $\varepsilon = 10^{-9}$ costs nothing and the tolerance axis is fully disarmed. As in the discrete setting, given the rule, the model translates it perfectly.

\paragraph{Incomplete arm.} Three-way structure:
\begin{enumerate}
\item \textbf{Wall absent from the sample} (the $(1-r)^N$ event; 10/20 seeds at $x_\mathrm{wall} = 8$ for each size, consistent with $(1-0.0114)^{40} \approx 0.63$): every such seed --- \textbf{20/20 across mini and large} --- passed the gate at 1.000, fully wall-blind on the probes, and was exploited at play: pinned at the wall, contact rate 1.0, \textbf{play\_cost 0.999}. Because \texttt{-{}-seed-offset 20} supplied the large arm with a sample block \emph{disjoint} from mini's, the two together are $\mathbf{20/20}$ over $20$ distinct gate-sample blocks --- a bound at the right unit rather than over draws that share evidence: exact (Clopper--Pearson) $95\%$ lower bound $\mathbf{0.832}$, Wilson $0.839$. The two arms differ in \emph{both} block and model size, so the interval covers the mixture rate rather than either arm's; per arm it is $10/10$ blocks, exact lower bound $0.692$. The third arm (large on $S_0$, the same 10 samples mini saw) reproduces the conditional model-for-model, $10/10$. The gate-miss rate is now poolable too: $20/40$ of the independent samples lacked the mode against the predicted $0.6046$ (Wilson $[0.35, 0.65]$). This is the discrete headline, synthesized end-to-end in a continuous CWM: a verified model, exact outside the mode region, that performs worse than random.
\item \textbf{Wall present, repaired:} the LLM does \emph{not} stay blind. It reads the failing transitions and writes the true global rule --- \texttt{if x2 >= 8.0: return [8.0, 0.0]} (or the equivalent \texttt{if x2 > 8.0: x2 = 8.0}) --- not a curve fit. At 20 seeds both sizes repaired \textbf{every} wall-present seed: \textbf{large 10/10 in 0--1 iterations} (two from the synthesis examples alone), \textbf{mini 10/10 in 0--5 iterations}. This is the divergence from the discrete setting, where rules demonstrated by example transitions were persistently not learned (translation-not-inference). A numerically-manifested discontinuity is learnable from data in a way a symbolic game rule was not.
\item \textbf{Wall present, not repaired:} at 20 seeds on the headline cell GPT-5.x produced \textbf{no stalls}, but stalls are real where they occur (the 5-seed $x_\mathrm{wall} = 4$ cell; the Qwen cross-family arm below). They are not near-misses of the rule but \textbf{superstitious local patches}: e.g.\ \texttt{if abs(x2 - 4.0) <= 0.15 and abs(v2) <= 2.5: x2 = 4.0; v2 = 0.0} (verbatim from the $x_\mathrm{wall} = 4$ cell, mini seed 20000, gate $0.974$), or Qwen's \texttt{if x2 >= 8.0 and v2 <= 0.0: ...} --- clamps fitted to the \emph{observed manifestation} of the mode (low-speed, near-wall contacts), which mispredict other approaches to the wall. \textbf{The gate rejected every one of them} (gate $0.49$--$0.999$, never 1.000).
\end{enumerate}

\paragraph{Two other model families.} Spot-checks in two further families (Qwen via an open-router deployment, and Claude relayed through an agent scaffold with byte-identical pipeline messages; 3 seeds plus a control per instrument, so small-$n$) separate the two halves of the result cleanly. The mode-absent blind-and-exploited event fires in \emph{every} family, as Proposition~\ref{prop:ident} requires of a property of the sample. Repair-from-data instead differs in \emph{mechanism}: GPT-5.x recovers almost every revealed clamp exactly, Qwen recovers none and the gate refuses its superstitious patches, and Claude recovers most through a \emph{symmetry prior} that generalizes one-sided evidence into a symmetric pair of boundaries --- which produces a fourth artifact class the trichotomy does not cover: an artifact accepted at gate $1.000$ while carrying an \emph{invented} mode its own sample cannot refute. That is Proposition~\ref{prop:ident}'s prior caveat measured directly, and it is why the paper's off-sample regularity is stated as measured rather than proved. Section~\ref{sup:crossfamily} gives the arms, the artifact and the relay caveats.

\paragraph{Second-instrument robustness (pendulum-with-stop, Section~\ref{sec:pendulum}, 20 seeds/cell, both sizes).} The synthesis arm is not cart-only. Running the identical pipeline on the nonlinear pendulum --- headline $\theta_\mathrm{stop} = 1.4$ (rarity 0.019) and caught $\theta_\mathrm{stop} = 1.0$ (rarity 0.128; ``caught'' labels a cell whose mode is common enough that the gate sample essentially always contains it, so the identifiability event essentially never fires) --- reproduces every branch, including a 3-seed Qwen cross-family spot-check at the headline knob (Table~\ref{tab:pendulum-synthesis}).

\begin{table}[ht]
\centering
\footnotesize
\begin{tabular}{lccc}
\toprule
cell (20 seeds each) & full & absent $\to$ blind \& exploited & present $\to$ repaired (stalled) \\
\midrule
mini $\theta_\mathrm{stop}=1.4$ & 20/20 & $9 \to 9$ (pc 0.995) & $11 \to 11$ (0) \\
large $\theta_\mathrm{stop}=1.4$ & 20/20 & $9 \to 9$ (pc 0.995) & $11 \to 11$ (0) \\
mini $\theta_\mathrm{stop}=1.0$ & 20/20 & $0 \to$ --- & $20 \to 20$ (0) \\
large $\theta_\mathrm{stop}=1.0$ & 20/20 & $0 \to$ --- & $20 \to 20$ (0) \\
Qwen $\theta_\mathrm{stop}=1.4$ (3 seeds) & 3/3 & $1 \to 1$ (pc 0.995) & $2 \to 0$ (2 stalled @0.9997) \\
\bottomrule
\end{tabular}
\caption{Pendulum synthesis cells (20 seeds/cell, both GPT-5.x sizes; 3-seed Qwen spot-check at $\theta_\mathrm{stop}=1.4$). $k \to m$ = of the $k$ seeds in that branch, $m$ had the stated outcome (blind \& exploited, resp.\ repaired); the parenthetical count is stalled seeds; pc = play\_cost.}
\label{tab:pendulum-synthesis}
\end{table}

Pooled across both knobs and both sizes (Table~\ref{tab:pendulum-synthesis}), every mode-absent occurrence was blind and exploited at play\_cost 0.995 (headline knob, disjoint blocks: $15/15$ distinct blocks, exact $95\%$ lower bound $0.782$, Wilson $0.796$; the same-sample pair mini/large@$S_0$ adds $9/9$ draws each, which are not extra blocks) --- the same fixed-point exploitation as the cart's play\_cost $\approx 1$ --- and GPT-5.x passed the repair criterion on \textbf{62/62} mode-present seeds, with no stalls, writing the exact angular clamp (verbatim at the headline knob: \texttt{if th2 >= 1.4: return [1.4, 0.0]}) in \textbf{58} of them; the other \textbf{four} write that clamp \emph{and} a phantom stop on the negative side, which the mode probe cannot see (Section~\ref{sec:patch2d-synthesis}) --- counting those four as non-repairs, as Section~\ref{sec:patch2d-synthesis} does, $30$ of $34$ distinct blocks are repaired on every attempt, exact $95\%$ interval $[0.725, 0.967]$; the per-size, per-knob cells are 11/11 and 20/20 draws by the probe criterion. Qwen reproduces the mode-absent blind-exploited event but stalls on both its mode-present seeds, the same superstitious-patch signature as on the cart; Claude's agent-relayed spot-check on this instrument produced the phantom-mode artifact discussed in Section~\ref{sec:synthesis} (an accepted artifact carrying a stop at $\theta = -1.4$ its sample never reaches, plus a stall at the other mode-present seed). Repair is model-dependent; identifiability, being a property of the sample, is not --- on this instrument too. The mechanism was already validated on two instruments (Section~\ref{sec:pendulum}); the synthesis result now is as well: a nonlinear plant with an angular, not positional, hard stop reproduces the same danger law and the same repair capability, so the repair finding is not a cart artifact.

\subsection{Repair does not transfer to a 2D region mode: 0/156 across the treatments run}
\label{sec:patch2d-synthesis}

Everything above --- and the (b)-residual-vanishes claim of the Introduction --- was measured on \emph{one-dimensional} hard boundaries: a position clamp (cart) and an angular clamp (pendulum), on which GPT-5.x recovered the rule in \textbf{105 of 111} mode-revealing synthesis draws. Two aspects of the criterion govern how that count may be read, and both tighten it.

First, the criterion. The paper's repair test is a gate pass plus a \texttt{mode\_blindness} probe score of $0$, and that probe fires only where the \emph{truth's} mode is active --- so it cannot see an artifact that \emph{invents} a mode elsewhere. Re-testing every artifact the probe calls repaired against the truth on a dense state--action grid (\texttt{scripts/\allowbreak repair\_\allowbreak exactness\_\allowbreak 1d.py}) finds \textbf{four} that are not exact: all four write the correct angular clamp \emph{and} a second, phantom stop on the negative side --- one at the symmetric $-1.4$, three at $-2.0$, which is where the left reward plateau sits. Two are \texttt{gpt-5.4-mini} and two are \texttt{gpt-5.4}, so the invented-mode class the paper attributes to the Claude arm is not family-specific; it appears in the main arms of both GPT-5.x sizes, at a rate of $4$ of $109$. Those four are counted as non-repairs here.

Second, the unit. The draws are not independent trials: the gate sample's random stream depends on the seed index and offset alone --- not on the instrument, the knob, the patch shape or the prompt variant --- so those $111$ draws sit over $36$ distinct gate-sample \emph{blocks} ($22$ on the cart, $34$ on the pendulum, some shared), and varying a knob adds a treatment rather than a sample (\texttt{scripts/\allowbreak sample\_\allowbreak stream\_\allowbreak census.py}, \texttt{scripts/\allowbreak paper2\_\allowbreak statistics.py}). The primary inference unit is the \emph{instrument--stream} block --- one stream pushed through one instrument's dynamics is one sample, and two instruments sharing a stream are dependent only through common random numbers --- and at that unit every attempt is exact on $50$ of $56$ blocks, exact $95\%$ interval $[0.781, 0.960]$. Clustered all the way down to the raw stream, the strictest possible unit, $30$ of $36$ blocks are repaired on every attempt, an exact $95\%$ interval of $[0.672, 0.936]$. Resolved the other way, by knob, the census is $64$ of $70$ knob-level block cells --- a block sampled at two knobs contributes two cells, so those are treatment cells, not additional samples. The cart is unaffected by the first correction: all $33$ of its repairs are exact on the grid, as its mode is a half-line with no far side to over-cover. Per instrument, the corrected criterion leaves the two alike in rate and different in kind: the pendulum is $30/34$ blocks all-repair (exact $95\%$ interval $[0.725, 0.967]$; its four exceptions are the four phantom-stop blocks --- over-coverage the probe accepted), while the cart is $20/22$ --- both its exceptions are superstitious local patches the gate \emph{rejected}, and one of them is the only mode-present draw its block ever produced --- giving an exact lower bound of $0.708$ for all-repair, or $0.772$ for the weaker estimand ``some attempt on a fresh sample repairs''. Both cart exceptions live in the 5-seed $x_\mathrm{wall} = 4$ cell. PatchField2D asks whether that repair survives a genuinely 2D mode --- a circular boundary $(x'-c_x)^2 + (y'-c_y)^2 \le R^2$ --- and whether an artifact can repair one patch while staying blind to the other (\textbf{partial repair}, which the design predicted and per-mode blindness makes measurable). We ran the identical pipeline (Azure GPT-5.x mini and large, 20 seeds/cell, $N = 40$, $\varepsilon = 10^{-9}$) on two bi-knob cells, $k = (3,7)$ and $k = (5,9)$, classifying each seed by which modes its training sample contained \{miss both, see one miss the other, see both\}. The full arm is clean (20/20 per cell, both sizes).

\paragraph{The instrument was built to detect partial repair; none occurred in its 66 draws.} Across the \textbf{66} see-one-miss-the-other seeds (64 see-$P_1$-miss-$P_2$ plus 2 miss-$P_1$-see-$P_2$; a complete census of this branch over $20$ raw seed blocks, not a sampled estimate), \textbf{zero} produced a partial-repair certificate --- not one artifact repaired the seen patch and stayed blind to the unseen one while passing the gate. More strongly, of the \textbf{76} incomplete-arm seeds whose sample contained \emph{at least one} mode, \textbf{0/76 recovered the circular disc rule at all}.

\paragraph{What the artifacts do instead.} A code inspection of all $76$ artifacts, confirmed by an independent behavioural audit that probes each \texttt{step()} on a state grid, locates the failure: plant translation succeeds in ${\approx}74/76$, and what collapses is induction of the region. The dominant class is \textbf{dimensional reduction} --- the disc written as a half-plane at the right location and the wrong shape ($38/76$ by source, $39/76$ by behaviour) --- with pure-blind, superstitious local patch and failed-disc classes making up the rest, and \textbf{no artifact encoding its seen patch} even where one was seen. The $\varepsilon$-exactness alternative is falsified: no correct-form disc failed on arithmetic. Two ablations then exclude two candidate mechanisms. A guided treatment at $3\times$ budget removes the half-plane entirely and yields bounded 2D regions instead --- ellipses, rectangles, unions of micro-discs --- fitted to the \emph{hull of the observed freeze positions} and, in $36/40$ artifacts, conditioned on the current rather than the landing position; none is the true disc. An axis-aligned square with flat edges is not repaired either, and fails with the errors \emph{reflected}: discs written on square evidence. A third family reproduces the same template set. Section~\ref{sup:artifacts} gives the full audit, the three ablations and the per-iteration ledger.

\paragraph{In these campaigns certification was all-or-nothing: no partial artifact was accepted.} Certification occurred \emph{iff} the sample missed \emph{all} modes: the only certified incomplete artifacts were the \textbf{4/80} miss-both seeds (2 per size at $k = (5,9)$), blind on both patches and exploited at play\_cost $1.095$, contact rate $1.0$ --- the doubly-blind fixed point, at the joint-miss event of probability $(1-r_\cup)^N$ (Section~\ref{sec:patch2d}: measured for the pair, not factored). Every see-one-miss-the-other artifact was \emph{rejected}, not certified: the all-or-nothing gate refused every partial artifact rather than certifying a half-repair. The gate was therefore \emph{sample-consistent} throughout, in the only sense it can be: no accepted artifact contradicted a transition of the sample it was scored on. That is the acceptance criterion restated, not a guarantee about the artifact. Two cautions on reading the failing gates: a high near-miss score (e.g.\ $0.9997$) reflects the \emph{rarity of a mode contact} in a short rollout, \emph{not} a near-repair, so more refinement iterations would not populate a partial-repair branch that does not exist; and the mini/large partitions are seed-identical by construction (shared seeds --- the design fact of Section~\ref{sec:synthesis}, which holds at every instrument, not only here), so the two sizes are not independent replications.

\paragraph{What the collapse is about, after eight ablations.} On the 1D clamps the discrete paper's (b)-residual all but vanished for GPT-5.x; on the 2D regions it returns in every treatment we ran. Six candidate causes are \emph{measured negatives}: it is not boundary curvature (the axis-aligned square with flat edges fails identically, with the errors reflected --- discs written on square evidence), not the tested prompting and budget (one region-first prompt at $3\times$ budget fails, in a different class --- other prompts, longer contexts, refinement with memory or an explicit fitting tool remain untried), not identification of the variable the trigger reads, not the censoring of the region's interior itself --- the mechanism Proposition~\ref{prop:entryclass} makes the natural candidate --- which two further campaigns lift without restoring repair, and not --- the one prediction in the 2D program written down before its test --- the angular coverage of the contacts, raised until a three-line least-squares fit recovers the region on every sample while the synthesizer still recovers it on none (all in Section~\ref{sec:arity}). A seventh candidate, the trigger's arity in the sense of Definition~\ref{def:arity}, is the one the instrument cannot rule on: the slab built to lower it has a target identified only up to Proposition~\ref{prop:entryclass}'s class, so that intervention excludes nothing and is recorded rather than counted.

What the campaigns establish about the \emph{evidence} is a separate and narrower thing. Proposition~\ref{prop:entryclass} shows that the mode's freeze semantics keep every visited state outside the region, so a sample witnesses only entries; whether that pins the rule then depends on whether a rollout can reach the region's far side, which is measurable in advance and differs across these instruments. On the disc it can, so the rule is identified there --- relative to the circle class and tolerance of Proposition~\ref{prop:discident}, in the block counts measured in Section~\ref{sec:arity} --- and the failure is a genuine one of induction: what the artifacts write instead is drawn from a small library of low-complexity forms --- a 1D threshold, a radial ball, a reward-landmark zone, the hull or the bounding box of the observed contacts --- and the gate refuses each in turn. That library is the template prior, and after eight interventions and two positive controls it is the description that survives: each intervention removes a candidate cause and the failure stays. What interventions on the \emph{instrument} cannot do is single out a mechanism inside the model --- an inability to carry out an algebraic fit over textually presented transitions, a memoryless refinement loop, and the absence of any system-identification objective in the prompt are all compatible with every campaign here, and we do not claim to distinguish among them. Removing curvature, budget, arity, variable ambiguity and the interior's censoring each leaves the failure in place --- the last while supplying eleven times more mode evidence and, separately, an easier rule --- and the controls locate what is missing exactly: given the region's form \emph{and} its location the synthesizer infers the remaining constant to float precision in $20$ of $20$ seeds, while a plain least-squares circle fit recovers both constants from the same evidence on $12$ of $20$ samples and the synthesizer given only the form recovers none (Section~\ref{sec:arity}). Raising the evidence's angular coverage until the fit succeeds on \emph{all} $20$ leaves the synthesizer at none of $20$, which is what makes the prior a disposition rather than a weight. What is not induced is a \emph{located} rule: the form alone does not rescue the failure, the form and its location together do.

Two readings this does not support. It is not that the synthesizer cannot represent a region: told the rule, it writes the disc, the square, the slab and the boundary projection at gate $1.000$ in zero refinement iterations, in every arm and both sizes. And it is not simply that more evidence would suffice --- the interior-witnessing campaigns supply more and it does not help --- though on an instrument whose far side is unreachable no sample at any size distinguishes the truth from an entry rule, and there the larger model reliably writes the latter (Section~\ref{sec:arity}). What does \emph{not} change across any of this is the provable core: identifiability and the gate-miss law hold verbatim, no accepted artifact contradicted its acceptance sample, and the joint-miss event remains where the law's \emph{exact} factor lives; the estimand $D_N$ itself also collects whatever every other accepted artifact costs --- the invented-stop class of Section~\ref{sec:synthesis} included --- which the per-campaign accepted-cost means report rather than the law absorbing.

\section{An independent acceptance sample}
\label{sec:heldout}

Every arm above scores the artifact on the sample it was synthesized and refined against. That is the protocol the companion paper used and the one deployed pipelines use, and it is a \emph{sample-consistency} gate: passing it means agreeing with the data the artifact has already seen. Proposition~\ref{prop:twofactor} says what changes when the acceptance sample is drawn independently, under two hypotheses. Both are testable on the artifacts we already have, because every synthesized program is versioned with its cell: for each of the \textbf{1034} committed artifacts we reproduce its training block $D_{\mathrm{tr}}$, draw two further blocks disjoint from it and from each other --- an acceptance sample $D_{\mathrm{g}}$ of the same size ($N_{\mathrm{g}} = 40$ rollouts) and an evaluation sample $D_{\mathrm{eval}}$ of $100$ --- and re-score the artifact on both without any refinement (\texttt{scripts/\allowbreak heldout\_\allowbreak gate\_\allowbreak audit.py}, \texttt{scripts/\allowbreak heldout\_\allowbreak gate\_\allowbreak paper\_\allowbreak numbers.py}). Refine on $D_{\mathrm{tr}}$, accept on $D_{\mathrm{g}}$, classify on $D_{\mathrm{eval}}$ \emph{is} the prospective protocol; evaluating it after the fact costs no LLM calls and changes nothing about its validity. Two validity checks come first: the three blocks are verified disjoint at the level of individual rollout seeds, and re-scoring an artifact on its \emph{reproduced} training block returns the accuracy the original run stored in $102$ of $102$ spot-checks, which is what establishes that the reconstruction is the original sample.

\paragraph{Both hypotheses hold where they were tested, so the exponents add there.} Hypothesis (ii) --- a mode-blind artifact fails the acceptance sample exactly when that sample contains a mode contact --- is nearly structural for a program with an omitted branch, and it is exact in the data: over the $60$ mode-blind artifacts synthesized from a mode-free training block, held-out acceptance and a mode-free acceptance sample coincide in every case ($25$ accepted with the mode absent from $D_{\mathrm{g}}$, $35$ rejected with it present, no off-diagonal cell). Hypothesis (i) --- a mode-free training block yields a mode-blind artifact --- is the genuine empirical premise, and on the cart's headline cell it holds for $30$ of $30$ draws. That check is cell-local, and extending it across the audited arms strengthens it: every one of the $60$ incomplete-arm draws whose training block misses the mode is probe-blind, so (i) has no known violation among the $1034$ API-arm artifacts, and the nearest counterexample to a global reading is the agent-relayed Claude artifact --- which carried an \emph{invented} mode, and sits outside these arms. The four GPT phantom stops of Section~\ref{sec:synthesis} are \emph{not} violations of (i): their training blocks contain the true mode (one contact transition each), and what they bound is a different claim --- the reach of evaluation-sample exactness --- taken up below. With both in place the proposition's prediction is a statement about the sample alone, so its unit is the distinct block: at $r = 0.0114$ the two-factor prediction is $(1-r)^{80} = 0.399$ against a measured $10/40 = 0.250$ over the $40$ distinct blocks, $95\%$ interval $[0.142, 0.402]$; the one-factor rate the earlier sections check is $(1-r)^{40} = 0.631$ against $20/40 = 0.500$, interval $[0.352, 0.648]$. Both predictions sit inside their intervals. Note the scope of what the second exponent buys: it halves the probability of the law's central event --- a blind artifact from a mode-free training block surviving to deployment --- and at fixed total budget $N_{\mathrm{tr}} + N_{\mathrm{g}}$ that one exponent is unchanged. It is a statement about that exponent alone (Proposition~\ref{prop:twofactor}): it does not price synthesis quality, total risk, or compute under a redistributed budget.

\paragraph{The sample-consistency gate accepts artifacts an independent gate rejects.} Of the \textbf{650} draws that reached $1.000$ on their own sample, \textbf{40} are rejected by an independent acceptance sample of the same size --- $6.2\%$ of draws, spanning $25$ distinct rollout-seed blocks. We report these as counts and attach no binomial interval: the draws share rollout-seed blocks across campaigns and treatments, so they are not independent trials from any population such an interval would describe. All $40$ are incomplete-arm artifacts, the full arm regressing on none of its own; $3$ draws move the other way --- below $1.000$ in-sample (at $0.9975$ to $0.9991$) yet accepted by the independent sample --- which is why $650 - 40$ leaves $613$ accepted rather than $610$. The failures are located where the law says they must be: $39$ of the $40$ fail \emph{only} on mode contacts and are off-mode exact, and the one exception is one of the four phantom-stop artifacts of Section~\ref{sec:synthesis}, rejected exactly where its invented stop fires ($11$ off-mode failing transitions on $D_{\mathrm{g}}$). So a $1.000$ score on the training sample is worth what Proposition~\ref{prop:ident} says it is worth and no more, and the quantity it overstates is measurable: about one accepted draw in sixteen here.

\paragraph{What acceptance does buy, measured rather than inspected.} The paper's off-sample regularity is a measurement, not a code-inspected observation: of the \textbf{613} draws an independent gate accepts, all \textbf{613} are exact outside the mode region \emph{on the further independent $100$-rollout evaluation sample} --- every failing transition of $D_{\mathrm{eval}}$, where there is one, is a mode contact. Two scope limits belong beside the number. The draws share seed blocks, so no binomial bound is attached to $613/613$; and exactness on $D_{\mathrm{eval}}$ is exactness under the evaluation distribution, not a global property of the artifact --- an accepted artifact carrying an invented mode in a region that distribution rarely visits would pass this check, and the pendulum's phantom stops (Section~\ref{sec:synthesis}) are exactly that shape. This is the sense in which ``verified but wrong'' is precise here: acceptance establishes exactness away from the mode on samples it never saw, drawn from the same distribution, and says nothing about the mode or about regions those samples rarely reach.

\paragraph{The limit an independent gate does not remove.} One limit belongs beside the result, and it is not a matter of sample size. Proposition~\ref{prop:entryclass} exhibits a class of models that agree with the truth on \emph{every} transition of \emph{every} rollout --- the mode's freeze semantics remove the evidence that would separate a membership rule from an entry rule --- and a synthesized artifact in that class is accepted here at $\varepsilon = 10^{-9}$ by its own sample, by an independent acceptance sample and by a $100$-rollout evaluation sample alike (Section~\ref{sec:arity}). An independent gate makes acceptance mean what it says about the sampled inputs; it does not make it mean more than that. What saves this particular class is not the gate but Proposition~\ref{prop:entryclass}(iii): its members are wrong only strictly inside a region no planner rolling them forward can reach, so the play cost is zero by the same argument that makes them unfalsifiable.

Two further scope notes. The re-scoring covers the $1034$ API-arm artifacts; the agent-relayed Claude cells are not in it, so Claude's phantom-mode artifact is not among the $613$. The $613$ do, however, contain known-wrong members of the same class, and they make the evaluation-exactness limit concrete. Of the four GPT phantom-stop artifacts (Section~\ref{sec:synthesis}) --- true mode repaired, an invented second stop where their samples are silent --- \textbf{three are accepted by the independent gate and score $1.000$ on $D_{\mathrm{eval}}$}, because no evaluation rollout happens to enter the invented stop's region, while the behavioural audit's state--action grid shows all four wrong there ($131$ to $280$ mismatched grid points each); the fourth drew gate and evaluation samples that do visit the region and is rejected ($13$ off-mode failures on $D_{\mathrm{eval}}$). Chance decided which one of the four was caught. That is Proposition~\ref{prop:ident} operating on the checks themselves: acceptance and evaluation read the artifact only at sampled inputs, so what convicts the other three is the grid, not any rollout. And an independent gate is not a solution to the problem this paper is about: it makes acceptance mean what it says, but the mode still has to appear in \emph{some} sample, and Proposition~\ref{prop:twofactor}'s exponent $N_{\mathrm{tr}} + N_{\mathrm{g}}$ is the same $N$ the danger law always had.

\section{What the 2D collapse is about: eight interventions, two controls, and a target that is sometimes not identifiable}
\label{sec:arity}

Section~\ref{sec:patch2d-synthesis} reported that no artifact recovered the 2D region rule and that the artifacts are \emph{consistent with} a low-complexity template prior, a mechanism that design does not isolate. Two further campaigns hold most of the confound fixed and vary one thing each. Both were run after the finding they respond to (Section~\ref{sup:prespec}); the first answers a different question from the one it was built to pose, and is reported against the question it does answer.

\paragraph{Whether the evidence identifies the rule --- relative to a stated class and tolerance --- is a property of the instrument, computable in advance.} Proposition~\ref{prop:entryclass} says a sample witnesses only \emph{entries} into the mode. Whether that suffices to pin the rule depends on one further question: can a rollout reach the region's \emph{far} side? Measured over $2000$ rollouts per instrument (\texttt{scripts/\allowbreak mode\_\allowbreak identifiability.py}), the answer separates the instruments the paper compares. On the disc the mover goes \emph{around} the patch --- the patches are bounded in $y$ --- and reaches $x = 14.11$, with $4695$ visited states east of the near patch's far edge; the square likewise ($3757$). So on those instruments the gate \emph{distribution} reaches the region's far side --- a population fact, and no more. Three claims must be kept apart. \textbf{Population support}: the distribution constrains the region from more than one side, which is what the $2000$-rollout measurement establishes; it says nothing about a particular finite sample. \textbf{Recovery by a specified estimator}: the three-line least-squares circle fit lands within $0.1$ of both constants on $12/20$ baseline samples and $20/20$ at the dose's $185^\circ$ coverage. This is an existence result, not uniqueness. \textbf{Finite-sample identification relative to a specified class}: Proposition~\ref{prop:discident}'s version-space certificate proves that \emph{every} consistent circle lies within the same tolerance on $6/20$ baseline samples and $18/20$ widest-dose samples. It excludes half-planes in $20/20$ at both doses and boxes in $19/20$ and $20/20$, but the hull of the observed contacts remains consistent in every block; therefore the finite evidence does not identify a unique arbitrary region or the entire artifact template library. The induction failure can be charged against estimator-level recoverability in all $20$ widest-dose blocks, and against universal circle-class identification in $18$ of them; both statements are reported because they answer different questions. A \emph{slab}, being unbounded in $y$, is the instrument that separates the two integers Remark~\ref{rem:twodims} keeps apart, and it cannot be circumvented: the furthest position reached is $x = 5.00$, exactly its near face, and \textbf{no} state east of it is ever visited. Its rule is therefore identified only up to Proposition~\ref{prop:entryclass}'s equivalence class, and ``recover the true rule'' is not a well-posed target there.

\paragraph{Eight interventions, and what each shows did not suffice.} Each is aimed at
one candidate explanation; none is single-variable in the strict sense --- the guided
treatment also changes the example count and the budget, the post-state variants also
change dwell time and the amount of mode evidence, the dose also changes the start
distribution and the rollout count (held to the same contact count by calibration) ---
and the changes each induces beyond its target are stated with its campaign
(Section~\ref{sup:ablations}; the full per-intervention matrix --- target
hypothesis, co-changes, prediction status, unit, control, licensed and
unlicensed inference --- is versioned in
\texttt{results/\allowbreak h6\_\allowbreak exclusion\_\allowbreak matrix\_\allowbreak v1.json}).
All were run after the finding they respond to
(Section~\ref{sup:prespec}), and none restores repair. Table~\ref{tab:ablations} is the
ledger; Section~\ref{sup:ablations} gives each campaign, its admissibility check and its
artifacts.

\begin{table}[ht]
\centering
\small
\begin{tabular}{llrl}
\toprule
\# & what it targets & repaired & what did not suffice, as tested \\
\midrule
1 & region-first guidance, $3\times$ budget & 0/40 & prompting and budget \\
2 & axis-aligned square, flat edges & 0/40 & boundary curvature \\
3 & a second model family (Claude) & 0/3 & one family's idiosyncrasy \\
4 & trigger arity: a band in one coordinate & 0/40 & \emph{nothing --- target unidentifiable} \\
5 & the trigger's argument named & 0/40 & variable identification \\
6 & the mover stops inside the region & 0/40 & the interior's censoring \\
7 & the mover clamped to the boundary & 0/40 & the same, at matched evidence \\
8 & angular coverage of the contacts & 0/40 & the coverage of the evidence \\
\bottomrule
\end{tabular}
\caption{The eight ablations on the 2D mode. ``repaired'' counts mode-containing synthesis
draws recovering the region rule, by the criterion of Section~\ref{sup:ablations} (a gate
pass, a mode probe at $0$, \emph{and} exact agreement with the truth on a state--action
grid, since the probe alone accepts a rule that over-covers the region). Ablation 4 excludes
nothing because its target is not identifiable --- the measurement above --- which is
recorded rather than reported as a $0/40$ like the others. Ablations 6 and 7 carry
complementary confounds: 6 supplies $11\times$ more mode evidence, 7 matches the baseline's
evidence and pays in the rule's complexity, so the negative survives both. Ablation 8 holds
the contact count fixed and raises only their angular coverage, to the point where a
three-line least-squares fit recovers the region on 20 of 20 samples; it is the only one of
the eight whose direction was written down before the run.}
\label{tab:ablations}
\end{table}

\paragraph{Two positive controls bracket the frontier: the constants follow from form and location; nothing follows from the form alone.} Every result above is a negative, and a negative is worth what the guarantee that its target is learnable is worth. Two controls supply that guarantee, one from inside the pipeline and one from outside.

Inside: replace the incomplete arm's missing clause with a \emph{partial} one that states the rule's form and its effect while withholding constants (\texttt{-{}-mode-hint}; the withheld constants appear nowhere in the text, asserted rather than assumed). Two levels, $20$ seeds each on \texttt{gpt-5.4}, and they separate completely.
\begin{itemize}
\item Given the form \emph{and} the centres, with only the radius withheld --- one unknown scalar --- the synthesizer infers it \textbf{exactly in $20$ of $20$} seeds. Every artifact agrees with the truth at $\mathrm{IoU} = 1.000$ and on all $9020$ points of the state--action grid, and all $20$ are accepted by an independent gate and an independent $100$-rollout evaluation at $\varepsilon = 10^{-9}$; one writes the comment ``radius inferred from the provided transitions''. Four of the twenty need no refinement iteration at all.
\item Given the form alone, with the centres and the radius withheld, it recovers \textbf{$0$ of $20$}, best agreement $0.132$, and $16$ of the $20$ write a region small enough to be a point on the probe grid --- the memorisation class again.
\end{itemize}

Outside: a plain algebraic least-squares circle fit on exactly the evidence the synthesizer was handed (\texttt{scripts/\allowbreak region\_\allowbreak fit\_\allowbreak baseline.py}). The contact landings are derivable from the sample by anyone who read the contract --- a contact is recognisable from the transition alone, and the landing follows from the contract's own integrator --- and they cover a median $111^\circ$ of the circle, partial coverage rather than a thin crescent. Even so the fit recovers both constants to within a tenth of the radius on $\mathbf{12}$ of the $20$ samples, and the centre alone on $13$, with no prior and no language model.

Together these place the failure precisely. It is not the evidence: on twelve of these samples three lines of linear algebra recover the region, and the synthesizer given the form recovers none of them. It is not an inability to fit constants: given the centres it fits the radius to float precision, every time. It is not representational: told the rule outright, every arm writes it at gate $1.000$ in zero iterations. What the synthesizer does not do is \emph{perform the fit} --- it substitutes a template and, when the template is refused, memorises the contacts. That is the region-template prior, stated as sharply as this instrument can state it: \textbf{a \emph{located} rule is what is not induced: the form alone does not rescue the failure, and once the form and its location are given the constants follow exactly}. On the eight samples where the trivial fit also fails, the negative is not attributable to the synthesizer, and we do not attribute it.

\paragraph{Within the reachable range ($111^\circ$--$185^\circ$), the prior does not yield to coverage --- and the direction was recorded before the run.} The controls above leave one mechanism question open. It is the only treatment in the 2D program whose two possible answers were both written down before its test --- which is weaker than pre-registration, since the question itself arose from the results it responds to, and Section~\ref{sup:prespec} classifies it that way (B21): if the obstruction is a prior over forms, enough evidence should overcome it; if more evidence changes nothing, it is a limit rather than a prior. The instrument can pose the question because the start distribution is a knob. Starting episodes on a ring around the near patch instead of in the box widens the \emph{angular coverage} of the contacts --- the arc they actually span, $360^\circ$ less the largest angular gap --- while the rollout count is lowered to hold the \emph{number} of contacts at the baseline's median of 14.5, so the dose is coverage and not quantity (\texttt{scripts/\allowbreak calibrate\_\allowbreak evidence\_\allowbreak dose.py}; the trap survives, play-cost 0.8701 and blind contact rate 0.92). Coverage saturates near 185$^\circ$, so the achievable range is 111$^\circ$ to 185$^\circ$, and the trivial fit traces it: it recovers both constants on 12 of 20 samples at the baseline's 111$^\circ$, 16 of 20 at 129$^\circ$, and \textbf{20 of 20} at 185$^\circ$. Thus every widest-dose sample supports recovery by that specified estimator; the stronger universal circle-class certificate holds in $18/20$ (Proposition~\ref{prop:discident}). The synthesizer recovers none. Given the form and asked only for its location and size it recovers \textbf{0 of 20} (best agreement 0.191); given no clause at all it recovers \textbf{0 of 20} (best 0.220), while the translation arm on the same wider sample still writes the rule at gate $1.000$ in zero iterations, $20$ of $20$. The narrower ring at 129$^\circ$ is the machinery control and behaves like the baseline (0 of 20), so the null at 185$^\circ$ is not an artifact of starting on a ring. Two consequences, one for each direction of the comparison. The estimator-level attribution becomes clean: the earlier control had to set aside the 8 samples on which the trivial fit also fails, and at 185$^\circ$ there are none to set aside --- on every sample in the dose arm the region is recoverable by the declared fit and the synthesizer does not recover it. At the stronger universal level the same contrast holds on $18$ blocks, with the two remaining blocks explicitly classified as underdetermined at tolerance rather than charged to synthesis. And the mechanism is bounded from the other side: within the range this instrument can reach, the failure does not respond to the dose at all, so ``prior'' should be read as a fixed disposition and not as a weight that more evidence overcomes. What lies beyond 185$^\circ$ this instrument cannot say: full coverage of the circle requires visiting the region's far side from every bearing, and the freeze semantics forbid it (Proposition~\ref{prop:entryclass}).

\paragraph{The two campaigns bound what an independent acceptance sample buys, and the bound is the evidence's identifiability.} On the disc, \emph{every} incomplete-arm artifact that reaches gate $1.000$ on its own sample is rejected by an independent acceptance sample: \textbf{8 of 8 draws over 6 distinct rollout-seed blocks}, at held-out accuracies $0.9944$ to $0.9997$. (Four are in the guided landing-prompt arm at $k = (3,7)$ and two in each size at $k = (5,9)$; the two sizes share their blocks, which is why the block count is $6$ and not $8$. The default disc arm reaches $1.000$ on none, so there is nothing there for a second sample to catch.) On the slab, nineteen reach gate $1.000$ and an independent sample rejects \textbf{0 of 19}. Nor is the independent gate merely strict: on the positive control of Section~\ref{sec:arity}, where the form and the centres are given and the artifacts are exactly right, it accepts \textbf{20 of 20} at $1.000$. That is not a difference in how carefully the two gates were run; it is Proposition~\ref{prop:entryclass} and the identifiability measurement above. Where the target is identified --- in the class-relative, finite-sample sense of Proposition~\ref{prop:discident} --- in-sample acceptance can be over-fitting and a held-out sample catches it; where the target is identified only up to the proposition's equivalence class, in-sample acceptance is the correct inference and there is nothing for a held-out sample to catch. An independent gate is exactly as effective as the evidence's identifiability permits, and no more --- which is the honest scope of Section~\ref{sec:heldout}.

Two consequences. The ablation cannot be read as an answer about arity, because its target is unidentifiable; we say so rather than reporting $0/40$ as if it were the disc's $0/156$. And the paper's repair criterion is not sufficient: a probe that fires only where the truth's mode is active cannot see a rule that over-covers it. Re-testing every 1D artifact the probe calls repaired against the truth on a dense grid (\texttt{scripts/\allowbreak repair\_\allowbreak exactness\_\allowbreak 1d.py}) finds four more of that kind, which is why Section~\ref{sec:patch2d-synthesis} reports $105$ of $111$: the four probe-passing phantom-stop artifacts are counted as non-repairs.

\section{Localization is representational: smooth learners and a mode-capable baseline on the same samples}
\label{sec:smooth}

``Representational'' is the narrowest of the four statements Remark~\ref{rem:fourstatements} keeps apart, and it is the only one this section tests. If localization is representational in that sense, two things must be checkable on non-code learners trained on the \emph{same data} (\texttt{scripts/continuous\_smooth\_probe.py}; the two most favorable smooth learners: closed-form linear least squares --- off the wall, the dynamics are \emph{exactly linear}, so this is the smooth best case --- and a small tanh MLP, probe-grade):

\begin{table}[ht]
\centering
\small
\begin{tabular}{llrrcc}
\toprule
model & trained on & off-mode err mean / max & probe err & gate $10^{-9}$ & gate $10^{-2}$ \\
\midrule
linear-LSQ & wall-free & $3.6\mathrm{e}{-15}$ / $1.7\mathrm{e}{-14}$ & 4.18 & \textbf{PASS} & \textbf{PASS} \\
linear-LSQ & wall-data & $1.9\mathrm{e}{-03}$ / $1.2\mathrm{e}{-02}$ & 4.17 & fail & fail \\
MLP h=8 & wall-free & $3.5\mathrm{e}{-03}$ / $5.0\mathrm{e}{-02}$ & 4.20 & fail & fail \\
MLP h=8 & wall-data & $6.0\mathrm{e}{-03}$ / $4.9\mathrm{e}{-02}$ & 4.19 & fail & fail \\
\bottomrule
\end{tabular}
\caption{Smooth learners on the synthesis samples ($x_\mathrm{wall} = 8$). ``probe err'' is the wall-region probe error (4.2 $\approx$ predicting straight through the wall).}
\label{tab:smooth}
\end{table}

\begin{figure}[ht]
\centering
\includegraphics[width=0.6\textwidth]{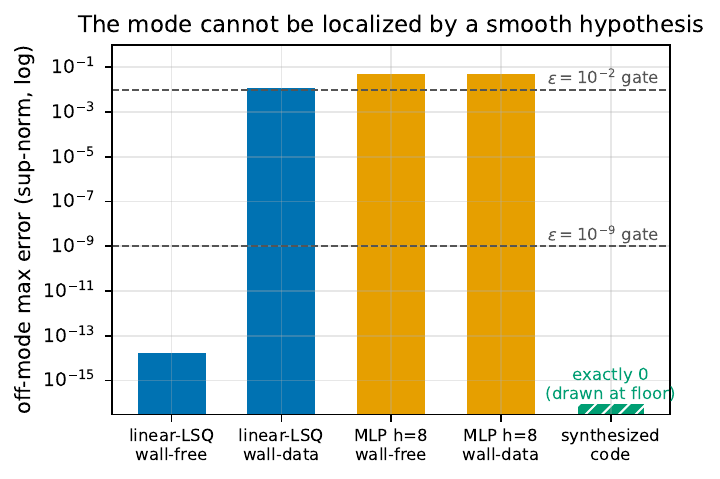}
\caption{Off-mode max error (log) by model and training data. The wall-free linear fit sits at float noise and passes both gates (blind); four contact rows tilt it twelve orders of magnitude; synthesized code is exactly zero off-mode.}
\label{fig:smooth}
\end{figure}

\textbf{Identifiability is learner-independent, live.} On the wall-free sample the linear model recovers the off-mode dynamics to $10^{-15}$, \textbf{passes the $\varepsilon = 10^{-9}$ gate}, and is exactly as wall-blind as the synthesized blind code (probe error 4.18: it predicts straight through the wall). Proposition~\ref{prop:ident} instantiated on a second hypothesis class. The proposition's exact content is that on the miss event no sample-measurable score separates a mode-blind model from the true one, so acceptance carries no information about the mode --- not that every accepted model is blind, which is false and which this paper's own data refute: Claude's pendulum artifact was accepted while carrying an \emph{invented} mode, its prior having supplied off-sample content that happens to be wrong in the other direction. Blindness is what a prior-free learner gets; the identifiability event is what makes any of it possible. The $(1-r)^N$ hole is not an LLM property.

\textbf{A mode-capable learned class closes the exactness gap where the evidence identifies the mode --- measured, not conceded.} The two learners above cannot contain the mode; the natural objection is a learner whose hypothesis class does. We test the most favourable such baseline: the instrument's own off-mode integrator and reward with only the event function learned from the same training samples --- a hard threshold on the cart, up to two separated circular event regions fitted by algebraic least squares on PatchField2D (\texttt{scripts/\allowbreak mode\_\allowbreak capable\_\allowbreak baseline\_\allowbreak h5.py}, \texttt{results/\allowbreak mode\_\allowbreak capable\_\allowbreak baseline\_\allowbreak h5.json}; the training seed block is the unit, $20$ blocks per arm). On the cart it behaves exactly like the code arm: from the wall-containing sample ($4$ contact transitions) it recovers the threshold exactly ($\hat x_\mathrm{wall} = 8.0$), is float-exact on all $3200$ held-out transitions, passes the $\varepsilon = 10^{-9}$ gate with zero bad transitions, and plays at truth parity (play\_cost $0.0$); three contacts already suffice, zero contacts never do; and from the wall-free sample it passes the same gate mode-blind --- Proposition~\ref{prop:ident} on a third hypothesis class. On PatchField2D, given the two-circle form, it recovers the near patch within $0.1$ of both constants on $12$ of $20$ baseline blocks and \emph{both} patches on $0$ of $20$ ($1$ of $20$ at the widest angular coverage) --- what it lacks is far-patch evidence, not representation. Three conclusions therefore separate cleanly: expressibility and float-level exactness are properties of any class containing the mode, code or not; induction from finite evidence succeeds for the numeric estimator exactly where the identifiability measurements of Section~\ref{sec:arity} say the evidence permits it; and the LLM synthesizer, handed the same form, recovers the located rule on none of its $20$ samples (Section~\ref{sec:arity}) --- the gap between the estimator's $12/20$ and the synthesizer's $0/20$ is the induction failure isolated from representation.

\textbf{With the mode in the data, code and smooth part ways.} Four contact rows out of 3200 tilt the linear fit by \textbf{twelve orders of magnitude} off-mode ($1.7\mathrm{e}{-14} \to 1.2\mathrm{e}{-02}$ max; Figure~\ref{fig:smooth}) --- it fails both gates \emph{and still} has the mode wrong (probe 4.17). The smooth hypothesis cannot put the error on the mode; it leaks everywhere (Proposition~\ref{prop:lipschitz}'s geometry, with Corollary~\ref{cor:locbudget} pricing the volume the leak must occupy, observed here as a least-squares tradeoff). The synthesized code, on the \emph{same sample}, wrote the exact clamp and passed at float precision. The MLP combines both axes: a pervasive ${\sim}5\mathrm{e}{-3}$ floor that no gate accepts, and a never-learned mode.

Consequence for the CWM paradigm: the paradigm \emph{creates} in continuous control the localized failure class that the learned-model literature says does not dominate there --- and, by the same representational fact, creates an exact-repair capability that no uniformly Lipschitz class can match (Corollary~\ref{cor:locbudget}). That capability is not exclusive to code: a learned event function whose class contains the mode matches code's float-exactness on the 1D clamp, measured above. What code supplies without being told is the class itself --- any computable rule, no pre-specified event family --- and the 2D campaigns show that this universality does not, by itself, deliver the induction of a located rule. Code cuts both ways, and the gate's sampling coverage decides which way.

\section{Related work}
\label{sec:related}


\paragraph{Counterexample-guided inductive synthesis: the same loop, a different oracle.}
CEGIS \citep{solarlezama2006sketching,solarlezama2008thesis} alternates a \emph{learner}, which proposes a candidate consistent with the counterexamples seen so far, and a \emph{verifier}, which either certifies the candidate against a specification or returns a new counterexample. The guarantee is inherited entirely from the verifier: one that is \emph{complete relative to a specification} either proves the candidate on the whole input domain or produces a witness. SyGuS standardized the format \citep{alur2013sygus}; the oracle-guided theory \citep{jha2010oracle,jha2017theory} makes what is synthesizable a function of which oracles are available, with programming by example \citep{gulwani2011flashfill,gulwani2017synthesis} as the bottom case, where generalization beyond the examples comes from a ranking function or a domain bias and not from a guarantee. Our synthesize--gate--refine pipeline \emph{is} a CEGIS loop --- LLM as learner, the gate's failing transitions as counterexamples, a refinement iteration as a CEGIS iteration --- with one load-bearing substitution: the counterexample oracle is random sampling of transitions, not a verifier. A verifier's ``no counterexample'' means \emph{there is none}; a sampling gate's means \emph{none was drawn in $N$ draws}, an event of probability exactly $(1-r)^N$ for a rule of rarity $r$ (Proposition~\ref{prop:gatemiss}), on which the loop is programming by example and the learner's prior decides everything off-sample. That weaker oracles buy weaker conclusions is not news to this literature and we do not re-derive it: it is the content of \citet{jha2017theory}, and the work that ran CEGIS against physical dynamics, where no complete verifier exists, substituted simulation and then recovered soundness by other means --- a bounded verifier, or an SMT check on the final candidate \citep{kapinski2014simulation,ravanbakhsh2019clf} --- or kept the verifier complete by pushing it into a theory solver \citep{abate2018cegist}. We add a closed-form law for how incomplete a sampling oracle is on a localized-rule failure, a proof that the miss event is unlearnable \emph{from the sample} by any learner (Proposition~\ref{prop:ident}), and a measurement of what a planner does with the accepted artifact. In the CEGIS sense our gate certifies one thing only --- the absence of counterexamples in a sample --- which is to say, nothing universal.

\paragraph{Statistical model checking.} SMC replaces exhaustive verification with simulation plus a statistical test and states what a sample size buys: acceptance sampling with Type-I/Type-II bounds \citep{younes2002acceptance}, Chernoff--Hoeffding complexity for an $(\epsilon,\delta)$ estimate of a property's probability \citep{herault2004apmc}, the black-box case \citep{sen2004blackbox}, and the standing difficulty that a rare property needs sample size scaling like the inverse of its probability \citep{legay2010smc}. Our gate is an SMC procedure for ``the model matches the plant to tolerance $\varepsilon$'' run without error bounds; we improve on none of that sample-complexity machinery, and \citet{herault2004apmc} is what would size a gate at a target confidence. SMC bounds the probability of a wrong verdict on a \emph{fixed} property, whereas we bound the probability that the sample never exhibits the phenomenon at all --- the verdict on the property actually tested is correct and uninformative --- and where SMC stops at the verdict we push it through a planner, whose loss is adversarially selected rather than proportional to the residual (Proposition~\ref{prop:playcost}).

\paragraph{Automata learning and active queries.} Passive data does not hand you the machine --- minimum-state automaton identification from given data is NP-hard \citep{gold1978complexity} --- whereas L$^*$ learns a regular language in polynomial time from membership queries \emph{plus an equivalence oracle} \citep{angluin1987learning}, which is what supplies a counterexample a random sample may never contain \citep{vaandrager2017model,settles2012active}. Proposition~\ref{prop:ident} is the continuous, hybrid-mode instance of that gap, and Angluin's decomposition reads our gate in one line: membership queries drawn at random, no equivalence oracle, which is why more refinement iterations cannot complete the loop. We inherit no query-complexity result --- their target is finite and exactly identifiable under exact queries, ours a program over a continuous space with tolerance $\varepsilon$ and a stochastic query distribution --- which is why we prove a probability law instead.

\paragraph{Falsification and counterexample-guided refinement.} Falsification searches for inputs driving a hybrid or cyber-physical system to violate a specification \citep{annpureddy2011staliro,corso2021survey,yamagata2021falsification}; CEGAR refines an abstraction at the spurious counterexample \citep{clarke2000cegar,clarke2003hybridcegar,alur2006predicate}; counterexample-guided data augmentation feeds falsified inputs back as training data \citep{dreossi2018cegda,dreossi2019verifai}. The message is uniform, and we do not improve on it: directed search finds the rare violating input random sampling misses. Our gate is the passive dual --- uniform random testing whose blind spot is precisely the mode a falsifier would hunt for, with $(1-r)^N$ quantifying the gap --- and the contrast with \citet{dreossi2018cegda} is our thesis in one sentence: there the counterexample that fixes the model comes from a falsifier, here from the gate's own random draw, and on the miss event there is none. CEGAR is also the template for our distrust-region mitigation, differing in the bookkeeping: it refines an abstraction over-approximate by construction, so refinement preserves the guarantee, whereas the fence refines a \emph{trust set} around an artifact with no such property --- hence a measured collapse of the exploitation and, at the farthest two-dimensional knob, $7$ of $20$ episodes pinned at blind-level return (Section~\ref{sec:patch2d-mitigation}).

\paragraph{Property-based testing with adaptive generators.} The gate is random testing of a program against an oracle and $(1-r)^N$ is the standard rare-input coverage gap \citep{claessen2000quickcheck,rubinstein2017montecarlo}, but uniform generation is not where that field stopped: generators can be steered by search toward inputs maximizing a utility \citep{loscher2017targeted,loscher2018automating}, derived from the property so constrained inputs arise by construction \citep{lampropoulos2017luck}, or guided by coverage feedback \citep{padhye2019zest}. Ours sits at the uniform, non-adaptive end of that spectrum, and each of the three has an instantiation here that we do not implement: a generator targeting distance travelled toward a mode boundary, one derived from the contract so mode-entering transitions occur by construction, and coverage-guided rollout mutation whose signal is which branches of the synthesized program executed --- available because the model is code. What we add is not a generator but that the planner is the adversary which finds the untested branch at deployment.

\paragraph{Experiment design and active system identification.} What is recoverable is a property of the input signal and not only of the estimator: optimal input design \citep{mehra1974optimal}, identification-for-control and the revival of experiment design \citep{gevers2005identification}, identifiability under a given excitation \citep{ljung1999sysid}, and the learning-theoretic result that actively chosen inputs achieve sample complexities passive excitation cannot \citep{wagenmaker2020active}. This literature implies our prescription and we re-derive none of it: our gate is a passive, non-adaptive experiment, and Proposition~\ref{prop:ident} is its identifiability statement specialized to a hybrid mode and a code hypothesis class. There the design reduces \emph{parameter} variance in a fixed parametric family and the failure is graded; ours is discrete --- a rule is present or absent in the program --- and its consequence is read through a planner rather than an estimation-error norm. The active boundary probing we list as follow-up work (Section~\ref{sec:limitations}) is an experiment-design problem and we name it as one.

\paragraph{Identification of hybrid and piecewise-affine systems.} Recovering a piecewise-affine or hybrid system is a joint classification-and-regression problem --- the partition of the state-input space as well as the affine law in each region --- and it is classical: clustering \citep{ferraritrecate2003clustering}, mixed-integer programming \citep{roll2004identification}, bounded-error set membership \citep{bemporad2005boundederror}, algebraic geometry \citep{vidal2003algebraic}, the tutorial and mixed logical dynamical formulations \citep{paoletti2007hybrid,bemporad1999control}, and trace-mining of hybrid automata \citep{medhat2015mining,soto2021synthesis}. We contribute no identification algorithm; our question is what a \emph{sampling verifier certifies} when the mode is missed, and what a planner then does --- the case these methods assume away, since none forms a region containing no samples. One is already active: \citet{soto2019membership} synthesizes linear hybrid automata from membership queries, the closest existing form of the fix we prescribe. Proposition~\ref{prop:lipschitz} and Corollary~\ref{cor:locbudget} are also a statement about this literature's hypothesis classes --- an exactly localized disagreement is unavailable to a bounded-Lipschitz class at any volume --- which is why we take code rather than piecewise-affine with a fixed mode count.

\paragraph{Contact-rich and discontinuous dynamics.} Contact is the physical archetype of our instruments and this literature's finding runs in the same direction: smooth end-to-end learning is the wrong parameterization and a complementarity-structured implicit loss does far better \citep{pfrommer2021contactnets}, stiff contact makes the loss landscape pathological for deep learners as a fundamental matter \citep{parmar2021stiff}, and even standard analytic planar contact models are limited on real data \citep{fazeli2020limitations,fazeli2017contact}. \citet{parmar2021stiff} is the empirical shadow of Proposition~\ref{prop:lipschitz} and \citet{pfrommer2021contactnets} takes structurally our remedy --- put the discontinuity in the hypothesis class --- neither of which we improve upon. The difference is what fails: there contact is present in the data and the failures are accuracy failures; here the mode is absent and the artifact is \emph{certified}. Our repair-from-data result is the code analogue --- given contact transitions in the sample the synthesizer recovered the exact switching rule on the one-dimensional clamps, whereas the two smooth fits we ran spread the error instead (Table~\ref{tab:smooth}) --- and, like this literature, we stop short of contact-rich manipulation.

\paragraph{Hypothesis classes that can represent modes.} Classes that \emph{can} express mode structure are mature: mixtures of experts with a learned gate \citep{jacobs1991moe}, recurrent switching linear dynamical systems whose discrete state depends on the continuous one \citep{linderman2017rslds}, neural ODEs with learned event functions \citep{chen2021eventfn}, and neural hybrid automata that infer the number of modes and their transitions \citep{poli2021nha}. This family scopes our representational claim rather than being scoped by it: Proposition~\ref{prop:lipschitz} quantifies over \emph{Lipschitz} pairs, and a hard gate or a learned event function reintroduces exactly the unbounded local structure it requires, so it is not an argument that neural models cannot represent hybrid modes --- our MLP is a probe at $h = 8$ scale, not a baseline (Section~\ref{sec:limitations}), and Section~\ref{sec:smooth} measures a favourable learned event-function baseline from this family directly. What survives is class-relative on the representation side and learner-independent on the identifiability side: no architecture here can infer a mode from a sample that never enters it (Proposition~\ref{prop:ident}), and float-level exactness --- zero off-mode error at $\varepsilon = 10^{-9}$ (Table~\ref{tab:smooth}) --- is a property of any class containing the mode: the measured event-function baseline attains it on the 1D clamp exactly as code does, while the smooth fits cannot (Section~\ref{sec:smooth}).

\paragraph{Conformal and PAC guarantees for learned dynamics.} Conformal prediction gives any predictor distribution-free finite-sample coverage under exchangeability \citep{vovk2022alrw} and has been carried to dynamics and control: prediction regions over trajectories, planned against \citep{lindemann2023conformal}; learned monitors that reject unreliable violation predictions \citep{bortolussi2019npm}; PAC-Bayes bounds for policies across environments \citep{majumdar2021pacbayes}. We derive no better certificate, and the comparison sharpens our result rather than softening it: their coverage holds with respect to the \emph{calibration} distribution, which here is the gate distribution, and the planner is not exchangeable with it --- the region our certificate covers carries $1.9\%$ of the exploited planner's queries, and a region containing every step-$t$ level set carries $7.8\%$ (Section~\ref{sec:limitations}). Calibrating conformally on gate rollouts would attach to the same wrong artifact a marginal guarantee that is honest over the calibration distribution and silent about the planner's queries; a planner could treat the conformal sets as uncertainty and react to their width, but sets calibrated where the mode never fires carry no widening at the mode, so the failure is not an artifact of choosing a metric certificate. Conversely ours is a sup-norm statement with an explicit $\varepsilon + 2L\rho$ constant (Propositions~\ref{prop:coverage} and~\ref{prop:partition}), which conformal prediction does not give, and it needs a finite Lipschitz constant, which conformal prediction does not --- where ours is weakest, the mode's local constant being unbounded.

\paragraph{Runtime assurance, safety filters, and robust MPC.} Every guarantee in this family is anchored to something known a priori: a verified conservative baseline plus a switching rule \citep{seto1998simplex,sha2001simplicity}, a shield synthesized from a safety automaton \citep{alshiekh2018shielding}, recursive feasibility across iterations of a learned controller \citep{rosolia2018lmpc,hewing2020learningmpc}, a minimally invasive projection keeping a known backup viable \citep{wabersich2021safetyfilter}, or a bounded model-error description \citep{rawlings2017mpc}. Our failure is invisible to every anchor \emph{derived from the learned model or its training distribution}: the error is unbounded but confined to a region the gate accepted as error-free, so a robustness margin tuned to the off-mode residual does not see it and a safety filter built on the same accepted model inherits the blind spot. An anchor specified independently of the model --- a verified conservative baseline, a shield synthesized from a safety automaton --- does not inherit it; it simply pays its conservatism whether or not a mode was omitted. The distrust-region fence has the shape of a safety filter with the anchor moved --- assembled from deployment-time refutations rather than an a-priori safe set, correcting a certified-but-wrong model rather than an unsafe policy --- which is why we report it as a mitigation with measured behaviour, its two-dimensional failure included (Section~\ref{sec:patch2d-mitigation}), and not as a certificate.

\paragraph{Decision-aware model learning and objective mismatch.} Prediction loss is the wrong objective for a model that will be planned with, and this literature has formalized the right one: losses weighting error by its effect on the Bellman update \citep{farahmand2017vaml,farahmand2018itervaml}, the value equivalence principle --- models are equivalent if they induce the same Bellman updates over a set of functions and policies \citep{grimm2020vep} --- an existence proof at scale for training purely for planning utility \citep{schrittwieser2020muzero}, and the finding that prediction accuracy and control performance diverge while planners exploit model error \citep{lambert2020objective,janner2019mbpo,hafner2020dreamer}. Our headline is a value-equivalence statement: the gate tests \emph{prediction} equivalence on the gate distribution, and we exhibit an artifact prediction-equivalent to within $10^{-9}$ everywhere the gate looked that is value-\emph{anti}-equivalent, capturing essentially none of the attainable return at the knobs swept. Their constructive programme is to \emph{train} for value equivalence, which needs a value function or a policy at training time; here the model is written before the planner runs, so our prescription is a change of \emph{test distribution}, not of loss --- and their results concern approximation quality within a parametric class, whereas our failure is an all-or-nothing missing rule that no reweighting of a loss recovers from a sample that never touched it. Our contribution to that conversation is the \emph{verified-but-wrong localized} regime, which a bounded-Lipschitz model can represent only at the volume price of Corollary~\ref{cor:locbudget} and which neither learner we fit represents at all (Table~\ref{tab:smooth}), plus a closed-form acceptance-failure law confirmed at gate scale.

\paragraph{Code world models.} Code World Models \citep{lehrach2025cwm} and code-as-policy \citep{liang2023codeaspolicies,gao2023pal} supply the paradigm, and the companion paper \citep{aguilar2026verified} is the discrete study this paper extends: we transfer its provable core, show that its empirical residual (translation, not inference) does not transfer on the one-dimensional clamps, and add the localization budget separating code from bounded-Lipschitz hypotheses.

\section{Limitations}
\label{sec:limitations}

\paragraph{Three instruments; repair is geometry-scoped, mitigation fails outright on a 2D boundary at distance.} Cart-with-wall, pendulum-with-stop, and PatchField2D span 1D and 4D state and one and two modes. The \emph{mechanism} survived every extension --- it is measure-theoretic --- and so did the gate's strict sample consistency (the all-or-nothing 2D gate accepted no partially-right artifact). The residue is sharper and different. \textbf{Repair-from-data collapses on 2D region boundaries, curved or flat}: GPT-5.x recovered the rule in 105 of 111 revealed 1D clamps (four more passed the mode probe while carrying an invented second stop) but recovered the 2D region rule in 0/156 mode-containing attempts over 20 distinct gate samples pooled over the disc, the square ablation, and the guided/$3\times$-budget treatment (Section~\ref{sec:patch2d-synthesis}). The prompting/budget and curvature outs are now \emph{measured negatives}; what could restore repair --- two-sided (inside) evidence, active boundary probing, or richer evidence summaries --- is the subject of ongoing follow-up work. The \textbf{2D mitigation} likewise generalizes but only partially, and its failure mode at the farthest knob is not a longer transient but lock-in: $7$ of $20$ episodes end pinned at blind-level return, because the fence's tie-break is an unsigned distance and a curved boundary lets the planner be pinned against it (Section~\ref{sec:patch2d-mitigation}). A signed outward normal is the obvious fix and is untested. Contact-rich manipulation, moving boundaries, and $3{+}$ modes remain future work.

\paragraph{Two planner families, one fixed configuration each.} Random-shooting MPC sits at the high-query-reach/high-play-cost end and CEM at the low-query-reach/near-zero-play-cost end (Section~\ref{sec:cem}); only the latter is \emph{implied} by Proposition~\ref{prop:playcost}, since the bound forces low play cost from low reach but does not force the converse. This is stronger than a one-family result but not planner-universal: CEM was measured at one prototype-fixed setting, with no hyperparameter sweep, and its pendulum truth returns expose local optima. Its non-exploitation is reach-limited, not knowledgeable or safe --- the same search that misses phantom reward can miss real distant reward. Gradient-based shooting and tree search \citep{coulom2006mcts,kocsis2006bandit} remain untested. Distrust-region replanning (Section~\ref{sec:mitigation}) is a mitigation layered on MPC, not a third family; its hard-boundary guarantee and its behavior with other planners remain untested.

\paragraph{Synthesis cells are modest, and the two sizes share their samples.} The gate sample is drawn from the seed index alone, so mini and large are synthesized from identical samples at every instrument and knob (Section~\ref{sec:synthesis}): every pooled ``both sizes'' count is $n$ distinct gate samples with two synthesis draws each, and the distinct-sample $n$ is half the pooled one --- 10 mode-absent samples on the cart headline cell, 31 knob-resolved mode-present samples across the two pendulum knobs for the shared-block pools (34 distinct pendulum blocks once the \texttt{-{}-seed-offset} arm is included), 38 mode-containing on the PatchField2D disc cells, 20 on each ablation campaign. Per-size Wilson bounds are quoted for exactly this reason wherever only shared blocks exist. On the two \emph{headline} cells the constraint was removed instead: the large arm was re-run on a disjoint seed block (\texttt{-{}-seed-offset 20}), so there the counts are over independent blocks and the bounds are legitimate at that unit (exact lower bounds $0.832$ cart, $0.782$ pendulum), and even the $(1-r)^N$ gate-miss check becomes poolable ($20/40$ independent cart samples lacked the mode against a predicted $0.63$). With that accounting: 20 seeds/cell on the headline $x_\mathrm{wall} = 8$ cart cell for both GPT-5.x sizes, plus a 3-seed Qwen cross-family spot-check; the pendulum arm adds two knobs (headline and caught) at the same 20 seeds/cell, both sizes, plus its own 3-seed Qwen spot-check. The wall/mode-absent conditional is 20/20 across cart sizes and three runs and 18/18 across pendulum sizes: per size that is 10/10 (Wilson lower bound 0.72) and 9/9 (0.70), the bounds we rely on, against pooled 0.84 and 0.824 whose trials share samples. The GPT-5.x repair rate is 20/20 on the cart headline cell and 62/62 pooled on the pendulum (31 distinct shared-block samples) by the probe criterion, of which 58 are exact on a dense grid and four carry an invented second stop. The cross-family arms are small-$n$ (3 seeds plus one control per instrument per family) and cover two alternate families (Qwen, and Claude agent-relayed), not a sweep of models --- and their coverage is uneven across instruments: both families were run on the two 1D instruments, and on PatchField2D both families carry both arms --- Claude's agent-relayed, and Qwen's as a backend-matched campaign (both arms on one pinned vLLM configuration, provenance versioned), with the incomplete cells refused $3/3$ (Section~\ref{sup:ablations}); they show repair is model-dependent \emph{in mechanism} --- Qwen 0/2 mode-present (superstitious patches), Claude repairing most but via a symmetry prior that oscillated on two seeds and had one phantom, unfalsifiable mode accepted. The PatchField2D synthesis arm adds two bi-knob cells (20 seeds each, both GPT-5.x sizes) plus two ablation campaigns (region-guidance $3\times$ budget; axis-aligned square with flat edges --- 20 seeds each, both sizes); its finding is a \emph{negative} one --- 0/76 repair of the circular mode, 0/66 partial repair, 0/156 pooled, plus 0/3 in a Claude cross-family arm on the same cell whose failure classes match (Section~\ref{sec:patch2d-synthesis}) --- so repair-from-data is model-dependent and evidence-dependent; Section~\ref{sec:arity} isolates which, and Proposition~\ref{prop:entryclass} says when the target is not identifiable at all.

\paragraph{The mode-blindness probe covers the true mode's region and no other.} By construction \texttt{mode\_blindness} (\texttt{src/cwm/continuous/contract.py}) fires probes where the truth's mode is active, so it establishes whether the \emph{revealed} mode is encoded but is blind to a mode the model \emph{invents} elsewhere: Claude's phantom pendulum stop (Section~\ref{sec:synthesis}) scored blindness 0.0 while carrying an extra, non-existent stop at $\theta = -1.4$. Detecting invented modes needs code inspection, or probes seeded outside the sampled region --- the classification alone does not suffice, and our phantom-mode finding rests on reading the accepted code, as the companion paper's artifact-level analysis does.

\paragraph{The MLP is a probe, not a baseline; the mode-capable baseline is favourable by design.} The MLP substantiates the representational point at $h{=}8$/pure-Python scale; a tuned modern dynamics model would have a lower floor but the same structural inability to be bit-exact off a mode at $\varepsilon = 10^{-9}$ (an argument, not yet a measurement, at scale). The learned event-function baseline of Section~\ref{sec:smooth} removes that objection where its class contains the mode, but it is deliberately favourable --- the off-mode plant is pinned and known, only the event function is learned --- so it bounds what representation explains, not what an end-to-end hybrid learner (mixture-of-experts, learned event functions at scale) would achieve; those remain untested.

\paragraph{Proposition~\ref{prop:lipschitz} bounds geometry, not probability.} The ball is metric; converting to gate-visitation probability needs the gate measure on that ball, which is instrument-specific. The empirical complement (Table~\ref{tab:smooth}'s twelve-orders tilt) carries the quantitative weight here. The budget itself is unconditional at boundary points (Corollaries~\ref{cor:locbudget}--\ref{cor:kappabudget}), but its exemption clause is scoped: ``no finite local constant'' is proved for the discontinuous reset modes studied --- the velocity-zeroing wall and stop and the freezing patch --- while a hybrid boundary that is continuous (a kink, a saturation) remains Lipschitz and inside the budget, so no claim is made that every hybrid boundary escapes it.

\paragraph{The coverage certificate reaches only the smooth case.} What transfers from the companion paper's enumeration is a covering-number statement for Lipschitz pairs, and on the deployed cart gate it is informative: no pair with $L \leq 5.77$ can carry the wall's error of $4.2$ past the gate on the certified region. The residue is that this is exactly the case the paper is \emph{not} about --- a discontinuous reset mode has no finite local Lipschitz constant, so no $L$ makes the certificate apply to it --- and that the certified region carries $1.9\%$ of the exploited planner's queries. Section~\ref{sup:certscope} states the scope in full, including what closing the within-rollout dependence and the step-$t$ density did and did not buy.

\paragraph{The instruments are designed and deterministic; bounded observation noise is measured, process noise and external benchmarks are not.} All three instruments were built around the law, so prevalence on externally specified dynamics is not established here. One robustness axis is now measured rather than assumed: under known coordinatewise $\mathrm{Uniform}[-\eta,\eta]$ observation noise, a support-compatibility gate preserves the acceptance law on a fixed $200$-block CartWall@4 panel through the pre-specified primary level $\eta = 0.1$ (blind-pass increase exactly $0$) and first crosses the frozen $+0.05$ masking boundary at $\eta = 1$ (blind pass $32/200$ blocks, exact CP95 $[0.1121, 0.2183]$; Section~\ref{sup:noisygate}, with the analysis pre-specified and hash-pinned). Process noise, an unknown noise law, planning under noise, and an externally authored contact or saturation benchmark with its own critical event remain open; the last is the decisive follow-up, and the analysis for it should be frozen before any outcome is inspected, as the noise analysis was.

\paragraph{play\_cost $> 1$ is a normalization artifact} (blind $<$ random), reported unclamped and explained; the headline claims never depend on the excess over 1.

\section{Conclusion}
\label{sec:conclusion}

The companion paper ended by diagnosing the gate failure as a reach-distribution shift and prescribing verification on the distribution the planner visits. This paper shows the diagnosis is not about discreteness. The gate-miss law, the identifiability argument, and the play-cost bound are measure-theoretic and survive the move to continuous state spaces intact; what the move changes is \emph{where the localized failure can live} (discontinuous reset boundaries --- an omitted mode is a rare rule) and \emph{who can express it} (programs place it exactly, while any class with a finite Lipschitz bound pays the volume price of Corollary~\ref{cor:locbudget}, now priced at boundary witnesses too via the interior-volume constant of Corollary~\ref{cor:kappabudget} --- Proposition~\ref{prop:lipschitz} and Table~\ref{tab:smooth}; smooth compactly supported errors and continuous hybrid boundaries are not excluded, and a measured learned event function attains code's float-exactness where its class contains the mode, Section~\ref{sec:smooth}). On two minimal hybrid instruments the entire discrete phenomenology reproduces --- threshold law, knob-invariant exploitation, the reach mechanism --- with the gate's acceptance rate closely tracking $(1-r)^N$ at gate scale (within sampling noise; Section~\ref{sec:axes}).

The synthesis experiment then delivers this paper's own finding: on the 1D instruments the continuous regime is the one where the danger law is \emph{dominant} --- near-exhaustive for the shipment of mode-blind artifacts. Given the mode in its sample, the LLM repairs it exactly in $105$ of $111$ draws; every attempt is exact on $50$ of $56$ instrument--stream blocks (exact 95\% interval $[0.781,0.960]$), and the treatment-resolved census is $64/70$ knob-level block cells --- including from the synthesis examples alone. Of the six others, the gate caught two superstitious patches, and four artifacts repaired the true stop while adding an \emph{invented} one their samples could not refute --- a class their own acceptance samples could not catch --- the independent gate caught one of the four, by the luck of its draw --- and the behavioural grid convicts in full; given the mode absent, every learner tested (LLM, exact linear regression, a mode-capable event-function baseline) is accepted blind and the planner is exploited at a regret of essentially the whole attainable return --- below the random baseline's mean where measured, though on the 1D instruments that comparison rides the baseline's heavy tail (Section~\ref{sec:mechanism}). That repair, though, is geometry-dependent: on a third, 4D bi-modal instrument with 2D modes it inverts --- neither GPT-5.x size recovers the region rule from data (0/156, including an axis-aligned square with flat edges and a guided $3\times$-budget treatment), two further model-family spot-checks at tiny $n$ fail in the same classes, and eight interventions plus two positive controls place the obstruction in the induction of a \emph{located} rule --- not in curvature, the tested prompting and budget, variable ambiguity, the interior's censoring, the evidence's coverage (raised to where a three-line fit is exact on every sample), or an inability to fit constants, each shown insufficient by measurement, with the arity intervention recorded as unanswerable on this instrument rather than excluded (Section~\ref{sec:arity}), and the exhaustiveness claim is scoped back to the 1D clamps, though the gate remains sample-consistent (Section~\ref{sec:patch2d-synthesis}). The provable core is the gate-miss and identifiability propositions (Section~\ref{sec:theory}): the mode-absent event is a hole no learner can close \emph{from the sample} --- a prior or the specification could still supply the mode --- and its probability is in closed form. Sample-covered exactness is the acceptance criterion itself; that accepted artifacts were also right off-sample --- with the known phantom-mode exceptions --- four pendulum artifacts of GPT-5.x and the Claude artifact, each wrong only where its own sample is silent --- is an empirical regularity of the models tested (GPT-5.x repairs; Qwen stalls and is refused; Claude repairs most through a symmetry prior), not a theorem. For practitioners the message is one clause sharper than the discrete one: on the 1D clamps measured here, \emph{sample coverage of the mode boundaries is the whole game} --- on the 2D regions coverage was raised and did not suffice (Section~\ref{sec:arity}), so the residual there is the located-rule induction gap --- the synthesis stack will translate what it is told and repair what it is shown, and no stage of it can \emph{learn from the sample} what the sample never touched (a prior or the specification still can, and Claude's did).

\appendix

\part*{Supplementary Material}
\addcontentsline{toc}{part}{Supplementary Material}

\noindent The material below supports the main article and is not needed to read it. It
contains the results the main text states in one sentence and cites: the $\varepsilon$-axis
propositions, the derived play-cost normalizers, the coverage certificates and their
statistical accounting, the detectability rate, the second planner family, the planner-side
mitigation, the complete LLM protocol, the reproducibility manifest, and the
pre-specification ledger.

\section{The reveal-rarity's \texorpdfstring{$\varepsilon$}{eps}-invariance, and the rate that replaces its threshold}
\label{sup:epsflat}

\begin{proposition}[$\varepsilon$-invariance of a hard mode's reveal-rarity]
\label{prop:epsinv}
Fix a hard mode, a gate policy, and a model that agrees with the truth \emph{exactly} off the mode --- which the mode-blind model does by construction, the two sharing an integrator. For a rollout $\omega$ put
\[
  D(\omega) \;=\; \max\{\,\|f(s,a) - \hat f(s,a)\|_\infty \;:\; (s,a) \text{ a mode contact of } \omega\,\},
\]
with $D(\omega) = 0$ if $\omega$ never contacts the mode. Then the reveal-rarity at tolerance $\varepsilon$ --- the probability that a gate rollout contains a transition on which truth and model differ by more than $\varepsilon$ --- is exactly $P_\rho(D > \varepsilon)$, non-increasing in $\varepsilon$. Consequently, with $\varepsilon^\ast = \operatorname*{ess\,inf}\{D : D > 0\}$,
\[
  \text{reveal-rarity}(\varepsilon) \;=\; \underbrace{P_\rho(D > 0)}_{\text{mode-firing rarity } r} \quad\text{for every } \varepsilon < \varepsilon^\ast,
\]
and $\mathrm{pass@}N = (1-r)^N$ is \emph{exactly} $\varepsilon$-invariant on that range.
\end{proposition}

\begin{proof}
Off the mode the two models agree exactly, so any transition with $\|f - \hat f\|_\infty > \varepsilon$ is a mode contact; a rollout therefore contains one iff its largest contact disagreement exceeds $\varepsilon$, i.e.\ iff $D(\omega) > \varepsilon$. The display is then the definition of the essential infimum, and $\mathrm{pass@}N$ follows from Proposition~\ref{prop:gatemiss}.
\end{proof}

Two readings of $\varepsilon^\ast$ must be kept apart. Against the \emph{empirical} measure of one gate sample, $\varepsilon^\ast$ is the smallest positive $D$ observed and the identity above is exactly what Table~\ref{tab:epsstar} reports --- in-sample, and no more. Against the \emph{population} measure, $\varepsilon^\ast = 0$ under Proposition~\ref{prop:epsrate}'s approachability hypothesis --- which on these instruments is an assumption the data are consistent with, not a theorem about them --- so the display is empty there and the population content of the flatness is not a threshold at all but the quadratic rate proved next. Writing one symbol $r$ for the two rarities is therefore licensed exactly in-sample, and in population to within $C\varepsilon^2$.

The threshold does not converge to a positive constant: the running minimum falls monotonically and does not settle --- on \texttt{wall@4}, $0.420$ at $25$ firing rollouts, $0.123$ at $200$, $0.041$ at $3200$. The population threshold is zero, and what stands in its place is a rate, which holds at every $\varepsilon$ rather than below an unstable cutoff.

\begin{proposition}[flatness at a rate for clamped semi-implicit plants; no population threshold]
\label{prop:epsrate}
Consider any plant of the semi-implicit form
\[
  \omega_{t+1} = \omega_t + \bigl(\mathrm{gain}\cdot a_t + F(\theta_t, \omega_t)\bigr)dt,
  \qquad \theta_{t+1} = \theta_t + \omega_{t+1}\,dt,
\]
with $F \in C^1$ \emph{arbitrary}, run over a horizon $T$ under the gate policy ($a_t$ i.i.d.\ $U(-a_{\max}, a_{\max})$, drawn independently of the state), and carrying a hard mode that clamps the integrated coordinate: if $\theta_{t+1} \geq \theta_{\mathrm{stop}}$ the next state is $(\theta_{\mathrm{stop}}, 0)$. Let the model be the same plant with the clamp removed, let $D(\omega)$ be a rollout's largest sup-norm disagreement over its mode contacts, and let $r_{\mathrm{fire}} = P(D > 0)$ be the mode-firing rarity. Then for every $\varepsilon > 0$,
\[
  0 \;\leq\; r_{\mathrm{fire}} - \mathrm{reveal\text{-}rarity}(\varepsilon) \;=\; P(0 < D \leq \varepsilon) \;\leq\; C\,\varepsilon^2,
  \qquad C = \frac{T\,M}{2\,\mathrm{gain}\,a_{\max}},
\]
where $M = \max\bigl(\|p_{\theta_0}\|_\infty,\ (2\,\mathrm{gain}\,dt^2 a_{\max})^{-1}\bigr)$ bounds the density of the clamped coordinate's absolutely continuous part below the stop at every step (the law also carries an atom \emph{at} the stop, which the proof handles separately). If moreover ($T \geq 3$) the mode is \emph{slowly approachable at the horizon's end} --- there is a boundary state $(\theta_{\mathrm{stop}}, \omega^\star)$ with two-sided margin $|\omega^\star + F(\theta_{\mathrm{stop}},\omega^\star)\,dt| < \mathrm{gain}\,dt\,a_{\max}$ such that for every $\delta > 0$ there are an event $H_\delta$ of positive probability on the history through step $T-3$ on which \emph{no clamp has yet fired}, and a nonempty open action rectangle $U_\delta \subseteq (-a_{\max}, a_{\max})^2$, with steps $T-2$ and $T-1$ also clamp-free for histories in $H_\delta$ and $(a_{T-3}, a_{T-2}) \in U_\delta$ and $(\theta_{T-1}, \omega_{T-1})$ landing within $\delta$ of $(\theta_{\mathrm{stop}}, \omega^\star)$ --- then $\operatorname*{ess\,inf}\{D : D > 0\} = 0$, so no positive $\varepsilon$ makes the two rarities equal: they agree only in the limit, at the quadratic rate above.
\end{proposition}

\begin{proof}
At a contact the truth clamps the velocity to $0$ while the model returns the pre-clamp post-action velocity $\omega'_{t+1}$, so $D \geq |\omega'_{t+1}|$; and the clamp fires only if $\theta_t + \omega'_{t+1}\,dt \geq \theta_{\mathrm{stop}}$ with $\theta_t < \theta_{\mathrm{stop}}$, so $\omega'_{t+1} \leq \varepsilon$ forces $\theta_t \in [\theta_{\mathrm{stop}} - \varepsilon\,dt,\, \theta_{\mathrm{stop}})$. Hence
\[
  \{0 < D \leq \varepsilon\} \;\subseteq\; \bigcup_{t < T}\ \{0 < \omega'_{t+1} \leq \varepsilon\} \cap \{\theta_t \in [\theta_{\mathrm{stop}} - \varepsilon dt,\, \theta_{\mathrm{stop}})\},
\]
and two facts independent of $F$ bound each term. First, $a_t$ is drawn independently of $(\theta_t, \omega_t)$ and enters $\omega'_{t+1} = \omega_t + F(\theta_t,\omega_t)dt + \mathrm{gain}\,dt\,a_t$ additively, so conditionally on $(\theta_t,\omega_t)$ it is uniform on an interval of length $2\,\mathrm{gain}\,dt\,a_{\max}$ and $P(0 < \omega'_{t+1} \leq \varepsilon \mid \theta_t, \omega_t) \leq \varepsilon/(2\,\mathrm{gain}\,dt\,a_{\max})$. Second, for $t \geq 1$ separate the two branches of step $t$ itself. Write
\[
  \theta^{\mathrm{pre}}_t \;=\; \underbrace{\theta_{t-1} + \omega_{t-1}dt + F(\theta_{t-1},\omega_{t-1})\,dt^2}_{\text{free of } a_{t-1}} \;+\; \mathrm{gain}\,dt^2\,a_{t-1}
\]
for the unclamped landing --- an expression valid whatever clamps fired at \emph{earlier} steps, since those only change $(\theta_{t-1}, \omega_{t-1})$, a history-measurable quantity (a clamp resets it to $(\theta_{\mathrm{stop}}, 0)$, still free of $a_{t-1}$). The realized coordinate is $\theta_t = \theta^{\mathrm{pre}}_t$ on $\{\theta^{\mathrm{pre}}_t < \theta_{\mathrm{stop}}\}$ and $\theta_t = \theta_{\mathrm{stop}}$ otherwise, so the conditional law of $\theta_t$ given the history through $t-1$ is \emph{not} uniform: it is a uniform density below the stop plus an atom at $\theta_{\mathrm{stop}}$ itself. The target strip $[\theta_{\mathrm{stop}} - \varepsilon dt, \theta_{\mathrm{stop}})$ excludes the atom, and on $\{\theta_t < \theta_{\mathrm{stop}}\}$ the two coordinates coincide, so $P(\theta_t \in [\theta_{\mathrm{stop}} - \varepsilon dt, \theta_{\mathrm{stop}})) = P(\theta^{\mathrm{pre}}_t \in [\theta_{\mathrm{stop}} - \varepsilon dt, \theta_{\mathrm{stop}}))$; and $\theta^{\mathrm{pre}}_t$, conditionally on that history, \emph{is} uniform on an interval of width $2\,\mathrm{gain}\,dt^2a_{\max}$ (the fresh $a_{t-1}$ enters affinely with slope $\mathrm{gain}\,dt^2$), with density at most $(2\,\mathrm{gain}\,dt^2a_{\max})^{-1}$; at $t = 0$ the bound is $\|p_{\theta_0}\|_\infty$. Either way it is at most $M$, so $P(\theta_t \in [\theta_{\mathrm{stop}} - \varepsilon dt, \theta_{\mathrm{stop}})) \leq M\varepsilon\,dt$. Multiplying the two and union-bounding over the $T$ steps gives the display. The second fact is where the semi-implicit \emph{ordering} is used: the action reaches the integrated coordinate within the same step, at order $dt^2$.

For the essential infimum, fix $\varepsilon > 0$. By continuity of $F$ there are $\delta_0 > 0$ and $\kappa > 0$ such that every state within $\delta_0$ of $(\theta_{\mathrm{stop}}, \omega^\star)$ has $m := \omega + F(\theta,\omega)\,dt \in (-\mathrm{gain}\,dt\,a_{\max} + \kappa,\ \mathrm{gain}\,dt\,a_{\max} - \kappa)$. Take $\delta < \min(\delta_0,\ \varepsilon\,dt/\max(1,dt),\ \kappa\,dt)$ and the $H_\delta$, $U_\delta$ the hypothesis supplies. On that event no clamp fires at any step before the final transition, so $(\theta_{T-1},\omega_{T-1})$ is the \emph{unclamped} two-step image of $(a_{T-3},a_{T-2})$ --- and only on such an event may the family-wide two-action Jacobian identity (Section~\ref{sup:coverage}), $\det \partial(\theta_{T-1},\omega_{T-1})/\partial(a_{T-3},a_{T-2}) = \mathrm{gain}^2 dt^3 \neq 0$, be invoked at all, since a clamp at step $T-1$ collapses the coordinate to the constant $\theta_{\mathrm{stop}}$ and a clamp at step $T-2$ confines the pair to a one-parameter affine image, a Lebesgue-null set in the plane. The actions being uniform with positive density on $U_\delta$, the pair has a conditional density bounded below on the open set $W := \mathrm{image} \cap B_\delta(\theta_{\mathrm{stop}},\omega^\star)$, nonempty by the hypothesis; every $(\theta,\omega) \in W$ has $\theta \in (\theta_{\mathrm{stop}} - \delta, \theta_{\mathrm{stop}})$ --- inside the strip, since $\delta \leq \varepsilon\,dt$ --- and carries the margin. So the event $\{(\theta_{T-1},\omega_{T-1}) \in W\} \cap \{\text{no clamp before the final transition}\}$ has positive probability. Conditionally on it the final action is fresh, so $\omega'_T = m + \mathrm{gain}\,dt\,a_{T-1}$ is uniform on an interval containing $(0, \kappa)$, and the final-step clamp fires with $\omega'_T \leq \varepsilon/\max(1,dt)$ exactly when $\omega'_T \in [(\theta_{\mathrm{stop}} - \theta_{T-1})/dt,\ \varepsilon/\max(1,dt)]$ --- a subinterval of positive length, its lower end being below $\delta/dt \leq \min(\varepsilon/\max(1,dt),\ \kappa)$. Because no clamp fired earlier, this final-transition contact is the rollout's \emph{only} contact; this is why the window sits at the horizon's end, since $D$ is a maximum over the rollout's contacts and an earlier or later coarse contact would push it past $\varepsilon$. On the constructed event the contact's sup-norm disagreement carries \emph{two} components --- the clamped velocity, $\omega'_T$, and the position excess $\theta_{T-1} + \omega'_T dt - \theta_{\mathrm{stop}} \leq \omega'_T\,dt$ (positive part; $\theta_{T-1} < \theta_{\mathrm{stop}}$) --- so $D = \max(\omega'_T,\ \theta^{\mathrm{pre}}_T - \theta_{\mathrm{stop}}) \leq \omega'_T \max(1, dt) \leq \varepsilon$, and $D \geq \omega'_T > 0$; hence $P(0 < D \leq \varepsilon) > 0$ for every $\varepsilon > 0$.
\end{proof}

Two constraints multiply, which is the whole content: a contact revealed only faintly needs both a nearly-zero post-action velocity \emph{and} a position already inside a strip of width $\varepsilon\,dt$. The constant is the same object the coverage certificate uses: substituting $M$ gives $C = T\,dt\,c$ with $c = (4a_{\max}^2\mathrm{gain}^2dt^3)^{-1}$ the universal step-$t$ density constant of the same family, so one constant serves both the flatness rate and the certificate. On this paper's instruments that is $M \leq 16.67$ and $C = 222$ for the cart \emph{and} the pendulum (identical, since they share $T$, $dt$, $\mathrm{gain}$ and $a_{\max}$; \texttt{scripts/\allowbreak eps\_\allowbreak flatness\_\allowbreak rate.py} derives both from the env objects). The bound is loose by orders of magnitude --- a union bound over eighty steps times two worst-case densities --- so the exponent is the content, and it is measured on both arms (60{,}000 rollouts each): $2.57$ on \texttt{wall@4} and $2.10$ on \texttt{stop@1.0}, both at or above the proved $2$, with the bound holding at every grid point. Because the proposition covers the nonlinear plant too, the pendulum's $2.10$ is a confirmation of a proved rate rather than the only evidence for it; what the near-equality adds is that the two-constraint mechanism is close to tight there, the union bound and the worst-case densities together costing only the gap between $2.10$ and $2$. The approachability hypothesis remains an \emph{assumption} about the plant --- no finite sample witnesses a for-every-$\varepsilon$ statement. What the data supply is consistency: observed single-contact disagreements reach down to $0.0018$ over the $919$ contacts on \texttt{wall@4}, three orders below the $\mathrm{gain}\,dt\,a_{\max} = 0.3$ margin the hypothesis needs, with nothing suggesting a positive floor. PatchField2D is outside the family, its action setting a force \emph{direction} rather than entering additively.

The rate also explains the threshold's instability: the minimum of $n$ draws from a law whose tail vanishes like $\varepsilon^{p}$ falls like $n^{-1/p}$, so $\varepsilon^\ast$ cannot converge to a positive limit. Table~\ref{tab:epsstar} accordingly reports an in-sample identity, and the population statement is Proposition~\ref{prop:epsrate}.

Note what $\varepsilon^\ast$ is: a per-rollout \emph{minimum} of a per-contact \emph{maximum}. A pinned rollout contacts the mode many times, and it only has to contact it coarsely \emph{once} to be revealed, which is why $\varepsilon^\ast$ sits two orders of magnitude above the smallest single-contact disagreement ($0.0018$ over the $919$ contacts on \texttt{wall@4}; \texttt{scripts/eps\_invariance\_threshold.py}).

Two constraints govern how the threshold may be used. First, $\varepsilon^\ast$ is a property of a \emph{sample}, so comparing a threshold computed on one rollout stream against dips measured on another compares two random quantities rather than testing a prediction; on four arms it disagrees on two. Computed on the sweep's own stream the relation is an \emph{identity}: reveal-rarity equals firing-rarity at every grid $\varepsilon$ below $\varepsilon^\ast$, and the first grid point at or above it is where the sweep dips. Table~\ref{tab:epsstar} reports it that way, in-sample, and claims nothing beyond the identity.

Second, a minimum is an unstable statistic --- upward-biased for the population infimum and monotonically decreasing in $n$ --- so no conclusion may rest on the last digit or on a comparison against the grid's edge. Measured across independent streams of the same size, $\varepsilon^\ast$ moves by a factor of $1.8$ on \texttt{wall@4} ($0.114$, $0.146$, $0.201$) and $1.4$ on \texttt{wall@8} ($0.396$, $0.438$, $0.538$, each from only $21$ to $25$ firing rollouts). And at ten times the sample it drops well below all of them: $0.065$ on \texttt{wall@4} and $0.272$ on \texttt{wall@8}. That last figure matters, because $0.272 < 0.3$: the claim that \texttt{wall@8} is flat \emph{throughout} the grid is a statement about the $2000$-rollout sample, not about the instrument, and at $20{,}000$ rollouts the grid's top point would dip. We therefore state flatness as holding on the range each sample resolves, which is what the proposition licenses, and not as a property of the arm.

\begin{table}[ht]
\centering
\small
\begin{tabular}{lrrrl}
\toprule
mode arm & $\varepsilon^\ast$ & firing & first grid $\varepsilon \geq \varepsilon^\ast$ & measured in the same sample \\
\midrule
cart \texttt{wall@8} & 0.3959 & 25 & none in grid & flat through $\varepsilon = 0.3$ \\
cart \texttt{wall@4} & 0.1137 & 277 & 0.3 & $0.1385$ to $\varepsilon = 0.1$, $0.1355$ at 0.3 \\
pendulum \texttt{stop@1.0} & 0.0791 & 282 & 0.1 & $0.1410$ to $3{\times}10^{-2}$, $0.1400$ at 0.1, $0.1240$ at 0.3 \\
pendulum \texttt{stop@1.4} & 0.0805 & 35 & 0.1 & $0.0175$ to $3{\times}10^{-2}$, $0.0170$ at 0.1, $0.0155$ at 0.3 \\
\bottomrule
\end{tabular}
\caption{The $\varepsilon$-invariance threshold, computed on the sweep's own rollout stream (2000 rollouts/arm, gate policy only, \texttt{--seed 10000} to match \texttt{continuous\_eps\_sweep.py}). Within one sample the relation is an identity, not a prediction: reveal-rarity is flat below $\varepsilon^\ast$ and the first grid point at or above it is where the sweep departs. The ``firing'' column is the number of rollouts the minimum is taken over --- $25$ on \texttt{wall@8} --- which is why the threshold itself carries the sample-to-sample spread discussed above and must not be read to the last digit.}
\label{tab:epsstar}
\end{table}

\section{The play-cost normalizers, and why knob-invariance is arithmetic}
\label{sup:normalizers}

\begin{proposition}[the normalizers, derived]
\label{prop:normalizers}
Fix the cart instrument with $x_\mathrm{wall} < x_\mathrm{right}$, and let $x_{\min}(t)$ be the position at step $t$ of the trajectory that starts at $x_0 = -x_0^{\mathrm{range}}$ and takes $a_s = -a_{\max}$ throughout. Then for \emph{every} policy and every admissible $x_0$,
\[
\begin{gathered}
  J \;\leq\; \sum_{t=1}^{T}\left[\frac{a_\mathrm{left}}{1 + e^{(x_{\min}(t) - x_\mathrm{left})/w}} \;+\; \frac{a_\mathrm{right}}{1 + e^{(x_\mathrm{right} - x_\mathrm{wall})/w}}\right]
  \;=:\; \bar J, \\[2pt]
  J \;\geq\; T \min_{x \in [x_{\min}(T),\, x_\mathrm{wall}]} r(x) \;=:\; \underline J,
\end{gathered}
\]
and both sides are explicit numbers.
\end{proposition}

\begin{proof}
\emph{Step 1: $x_t \leq x_\mathrm{wall}$.} Immediate from the clamp.

\emph{Step 2: $x_t \geq x_{\min}(t)$, clamp or no clamp.} Absent clamping, $x_t$ is affine in $(a_0,\dots,a_{t-1})$ with all coefficients positive --- each action enters through $\mathrm{gain}\,dt^2$ times a positive sum of powers of $1-\mathrm{drag}\,dt$ --- so it is minimised over admissible action sequences and starting points exactly at $a_s \equiv -a_{\max}$ from $x_0 = -x_0^{\mathrm{range}}$, which is $x_{\min}(t)$. The clamp needs a separate argument, because it lowers \emph{both} coordinates (to $x_\mathrm{wall}$ from a larger $x$, and to $0$ from a positive $v$), so a clamped trajectory is not dominated by its unclamped counterpart. Two elementary facts. The velocity recursion $v_{s+1} = (1-\mathrm{drag}\,dt)v_s - \mathrm{gain}\,a_{\max}dt$ under full left thrust does not involve $x$, so the all-left trajectory started at rest from any $x_0$ is $x_0 - L(t)$ for one and the same $L(t) = -dt\sum_{s \leq t} v_s$; in particular $x_{\min}(t) = -x_0^{\mathrm{range}} - L(t)$. And $L$ is increasing, since $v_0 = 0$ and $v_s \leq 0$ implies $v_{s+1} \leq -\mathrm{gain}\,a_{\max}dt < 0$, the factor $1-\mathrm{drag}\,dt$ lying in $(0,1)$. Now let the clamp last fire at time $\tau$, leaving the state $(x_\mathrm{wall}, 0)$; after $\tau$ the dynamics are affine, so the minimisation above applies to the remaining $t-\tau$ steps and
\[
  x_t \;\geq\; x_\mathrm{wall} - L(t-\tau) \;\geq\; -x_0^{\mathrm{range}} - L(t-\tau) \;\geq\; -x_0^{\mathrm{range}} - L(t) \;=\; x_{\min}(t),
\]
the middle step by $x_\mathrm{wall} \geq -x_0^{\mathrm{range}}$ and the last by monotonicity of $L$. If the clamp never fires, the first sentence of this step already gives it.

\emph{Step 3: from the envelope to the bounds.} $r$ is the sum of a term decreasing in $x$ and a term increasing in $x$, so on any interval it is at most the decreasing term at the left endpoint plus the increasing term at the right endpoint. Applying this on $[x_{\min}(t), x_\mathrm{wall}]$, which contains $x_t$ by Steps 1 and 2, and summing over $t$ gives $\bar J$. The lower bound is immediate from $r > 0$ and $x_t \in [x_{\min}(T), x_\mathrm{wall}]$, using that $x_{\min}$ is decreasing.
\end{proof}

so the bound reads $|\mathrm{play\_cost}| \lesssim q_{\mathrm{hit}}(E)$ exactly when $J_{\mathrm{rand}} \approx J_{\min}$ and $J_{\mathrm{truth}} \approx J_{\max}$ --- the random policy near the reward floor, the truth planner near the ceiling. Both instruments are in this regime: on the cart $J_{\mathrm{rand}} = 0.53$ against a pinned $J_{\mathrm{blind}} \approx 0$ and $J_{\mathrm{truth}} = 17.77$; on the pendulum $J_{\mathrm{rand}} = 0.06$ and $J_{\mathrm{truth}} = 20.08$. The blind planner queries the wall region in every episode, so $q_{\mathrm{hit}}(E) \approx 1$ and the normalized bound is saturated --- $\mathrm{play\_cost} \approx 1$, as observed in Sections~\ref{sec:mechanism}--\ref{sec:axes}.

On the cart the bound is nearly but not exactly saturated, and it is worth being precise about the gap rather than calling it attainment. With the derived normalizers of Proposition~\ref{prop:normalizers} --- $J_{\max} \leq 18.0359$ and $J_{\min} \geq 1.33\times10^{-6}$ at $x_\mathrm{wall} = 8$ --- and $q_{\mathrm{hit}} = 1$, the ceiling reads $18.0359/17.238 = \mathbf{1.0463}$ against the play\_cost measured \emph{at that knob}, $1.0299$: the measurement sits at $98.4\%$ of a bound with nothing measured in it. Deriving the normalizers rather than estimating them costs almost nothing --- the measured-supremum version of the same ceiling was $1.0445$, so the proof is $1.0017\times$ weaker --- and it leaves no step in this display where an estimate stands in for a supremum. What is \emph{measured} rather than implied is the informative part --- the exploited planner reaches the realizable floor, $J_{\mathrm{blind}} \approx J_{\min}$ --- and that is what pushes play\_cost to the top of the range; the proposition supplies the ceiling, not the observation.

The bound is not merely valid but essentially attained, because the optimal policy really is ``push left and stay'': at $x_\mathrm{wall} \leq 6$ the push-left trajectory realises $\bar J$ to within $2\times10^{-7}$, so $J_{\max}$ is \emph{known} there rather than bounded, (\texttt{scripts/\allowbreak play\_\allowbreak cost\_\allowbreak proved\_\allowbreak bounds.py}, which also searches bang-bang, constant, single-switch and $4000$ random block policies at each knob as a redundancy check on the derivation). The slack grows only as the wall approaches the right plateau and lets its tail in: $1.0000$ at $x_\mathrm{wall} \leq 6$, $1.0015$ at $8$, $1.0799$ at $10$. The floor $\underline J$ is attained near $x = 2.70$, where the left plateau's tail has died and the right one has not begun; at $x_\mathrm{wall} = 2$ it is $2.87\times10^{-6}$, which is the measured $J_{\min}$ to two digits.

The same script asks how much of the \emph{pinning} is derivable, since $\mathrm{contact} = 1.00$ at every knob is the other measured input. Two of its three ingredients are structural. The wall is reachable in a small fraction of the horizon --- $14$ steps at $x_\mathrm{wall} = 2$ rising to $30$ at $10$, against $T = 80$, and this is the earliest possible contact because $x_t$ is maximised by full right thrust for the same monotonicity reason as above. And the blind model offers a return of $8.31$ for going right against $5.70$ for going left over the planner's $40$-step horizon, a ratio of $1.46$: the asymmetry that makes right the argmax is a property of the reward, not of the sampling. The third ingredient is not derivable: nothing forces the argmax over $200$ random block candidates to execute a \emph{positive first action}, since a trajectory may go left before turning.

That obstruction is real but it is about the wrong quantity. What the \emph{plan} targets is a different question, and it admits a bound (\texttt{scripts/\allowbreak phantom\_\allowbreak targeting\_\allowbreak probability.py}). Fix a state $s$ and let $\bar x = x_\mathrm{right} - 2w$. The same envelope argument as Proposition~\ref{prop:normalizers} --- applied now under the \emph{blind} model, over the planner's horizon --- bounds the imagined return of any candidate that never exceeds $\bar x$ by an explicit $U(s)$, while the full-right-thrust candidate achieves an explicit $L(s)$. Wherever $L(s) > U(s)$, any candidate scoring at least $L(s)$ must exceed $\bar x$, so if the sampled set holds one then \emph{the argmax's trajectory reaches the phantom plateau}: the planner is provably planning against reward the truth makes unreachable, which is the exploitation stated as a property of the plan rather than of the outcome.

The gap closes at every state we checked, from $+0.33$ at $(-\tfrac12, 0)$ to $+26.25$ at $(6,5)$. What does not is the probability that the sampled set holds such a candidate, and the shape of that is the point. From rest it is negligible --- $2\times10^{-4}$ from an exactly computable sub-event (all four blocks above a threshold, where the score is monotone so the corner is the worst case), $5\times10^{-3}$ estimated with a Wilson lower limit on the planner's own sampling law. Once the cart carries rightward velocity it is $1.000$ by both routes, from $(2,3)$ onward. So the statement is conditional: \textbf{phantom-targeting is self-reinforcing, not self-starting}. How the lure begins, the executed first action's sign, and hence $\mathrm{contact} = 1.00$ all stay measured.

\begin{proposition}[knob-invariance of play\_cost is arithmetic, not a regularity]
\label{prop:knobinv}
Let the rarity knob $k$ range over values for which the truth planner's realized return $J_{\mathrm{truth}}$ and the random policy's return $J_{\mathrm{rand}}$ are \emph{exactly} independent of $k$ (an idealization: the measured departure and its propagated effect are quantified after the proof). Then, directly from the definition of Section~\ref{sec:planner},
\[
  \mathrm{play\_cost}(k) \;=\; \underbrace{\frac{J_{\mathrm{truth}}}{J_{\mathrm{truth}} - J_{\mathrm{rand}}}}_{\text{knob-free}} \;-\; \frac{J_{\mathrm{blind}}(k)}{J_{\mathrm{truth}} - J_{\mathrm{rand}}},
\]
an affine function of the exploited planner's own return with knob-independent coefficients, and hence constant up to
\[
  \max_k \mathrm{play\_cost}(k) - \min_k \mathrm{play\_cost}(k) \;=\; \frac{\max_k J_{\mathrm{blind}}(k) - \min_k J_{\mathrm{blind}}(k)}{J_{\mathrm{truth}} - J_{\mathrm{rand}}},
\]
\end{proposition}

\begin{proof}
$\mathrm{play\_cost}(k) = (J_{\mathrm{truth}} - J_{\mathrm{blind}}(k))/(J_{\mathrm{truth}} - J_{\mathrm{rand}})$ by definition; splitting the numerator gives the first display, whose first term is knob-free by hypothesis, and the second display is the range of an affine function with a negative knob-free slope.
\end{proof}

So the entire knob-dependence is the exploited planner's \emph{residual} reward divided by the truth-minus-random margin, and a planner pinned with none gives play\_cost exactly $J_{\mathrm{truth}}/(J_{\mathrm{truth}} - J_{\mathrm{rand}})$.

The two hypotheses are measured, and measured sharply: on the cart the truth planner's return is knob-independent \emph{to twelve digits} ($J_{\mathrm{truth}} = 17.757356407381$ at all seven knobs of the sharp variant, $17.772246981024$ at all seven of the default), and $J_{\mathrm{rand}}$ varies across knobs by $3.9\times10^{-9}$ on the sharp variant and $1.2\times10^{-4}$ on the default --- the random policy does occasionally enter the mode region, so its knob-independence is approximate rather than structural, and measurably less approximate once the plateau tail is removed. With those in hand the proposition is not decoration: it predicts every measured play\_cost to $2.4\times10^{-10}$ on the sharp variant, and that residue is itself accounted for rather than left over: $\partial\,\mathrm{play\_cost}/\partial J_{\mathrm{rand}} = (J_{\mathrm{truth}} - J_{\mathrm{blind}})/(J_{\mathrm{truth}} - J_{\mathrm{rand}})^2 \approx 0.060$, which against $J_{\mathrm{rand}}$'s measured $3.9\times10^{-9}$ knob-variation predicts $2.3\times10^{-10}$ --- the observed residue to one digit. On the default the same propagation of a $1.2\times10^{-4}$ variation gives $7\times10^{-6}$, matching the $7.3\times10^{-6}$ observed there. So the ``knob-invariant $\mathrm{play\_cost} \approx 1$'' of Section~\ref{sec:mechanism} is an \emph{identity} once the planner is pinned, and what the sharp-plateau variant buys is not invariance itself but the vanishing of $J_{\mathrm{blind}}$ that makes the identity's residual term negligible. What remains empirical is the pinning ($\mathrm{contact} = 1.00$ and $J_{\mathrm{blind}} \approx 0$) and the truth planner's indifference to the knob. The latter is not automatic --- the clamp does change the imagined value of right-going candidates --- so we do not leave it at bit-identity. A candidate's imagined return under truth depends on the knob \emph{only if} its imagined trajectory reaches the clamp, so it suffices to check, at every knob and every replanning step, that (C1) the argmax candidate's imagined trajectory stays strictly below the smallest wall in the sweep, and (C2) every candidate that does reach the wall scores strictly below the argmax. C1 alone pins the argmax \emph{value} but not the argmax \emph{candidate}, and it is the candidate that determines the executed action, so the gap has to be closed explicitly. It is closed by the tie-break rather than by uniqueness: the planner keeps the \emph{first} maximiser, the candidate enumeration depends only on $(a_{\max}, \mathrm{rng}, \mathrm{horizon}, n_{\mathrm{samples}}, \mathrm{block})$ and not on the knob, and C2 keeps every clamping candidate strictly below the top --- so the first maximiser is the same knob-free candidate at every knob whether or not the top is attained twice. That distinction is not hypothetical: uniqueness \emph{fails} in the sweep, because the reward saturates to $1.0$ in floating point on the plateau and many candidates then score identically, so the script reports uniqueness as a diagnostic and the certificate rests on C1, C2 and the enumeration-order argument. The result is a machine-checked certificate over the planner's own candidate set rather than an observation about outputs. \texttt{scripts/\allowbreak truth\_\allowbreak plan\_\allowbreak invariance\_\allowbreak certificate.py} runs it over the harness's own $20$ episodes per knob, the same scope as the runs the paper reports: the certificate holds at \emph{every} knob of the sweep, with the argmax candidate never exceeding $x = 0.344$ and clamping candidates losing by a margin of $5.25$ or more. Pushed past the sweep, the margin contracts monotonically --- $4.12$ at $x_\mathrm{wall} = 11$, $0.51$ at $12$ --- and the certificate fails at $12.5$, where the argmax candidate itself reaches the wall. So the regime has a \emph{derived} boundary at $x_\mathrm{wall} \approx x_\mathrm{right} = 12$: the truth planner switches to the right plateau exactly once the wall stops blocking it, which is also where the mode stops being a trap and there is no danger left to measure. The sweep $[2,10]$ sits inside the certified region with a wide margin.

\section{What the gate does certify: coverage certificates for Lipschitz pairs}
\label{sup:coverage}

\paragraph{What the gate \emph{does} certify: the continuous coverage analogue.} The companion paper's coverage certificates enumerate finite information-set spaces, and continuous state spaces admit no such enumeration. A covering-number analogue is available in its place, built from the disagreement ball of Proposition~\ref{prop:lipschitz} and the visitation density of Remark~\ref{rem:densitygeneral}.

\begin{proposition}[coverage certificate for Lipschitz models]
\label{prop:coverage}
Let $f, \hat f$ be $L$-Lipschitz on $U \subseteq S \times A$ in the sup-metric, and suppose $\hat f$ passes the gate at tolerance $\varepsilon$ on every visited transition. If the visited set is a $\rho$-net of $U$, then
\[
  \sup_{u \in U} \|f(u) - \hat f(u)\|_\infty \;\leq\; \varepsilon + 2L\rho .
\]
Moreover, if the gate's per-step visitation density is at least $c$ on $U$ and $M$ independent samples are drawn, then the visited set is a $\rho$-net with probability at least $1-\delta$ as soon as
\[
  M \;\geq\; \frac{\ln\!\big(N_{\mathrm{pack}}(U, \rho/2)/\delta\big)}{c \cdot \inf_{u \in U}\mathrm{vol}\big(B(u,\rho/2) \cap U\big)},
\]
where $N_{\mathrm{pack}}(U,\rho/2) \leq \mathrm{vol}(U \oplus B_{\rho/4})/\mathrm{vol}(B_{\rho/4})$.
\end{proposition}

\begin{proof}
For the first part, take $u \in U$ and let $v$ be a visited point with $\|u - v\|_\infty \leq \rho$; then $\|f(u) - \hat f(u)\| \leq \|f(v) - \hat f(v)\| + \|f(u)-f(v)\| + \|\hat f(u) - \hat f(v)\| \leq \varepsilon + 2L\rho$. For the second, fix a maximal $(\rho/2)$-packing of $U$; it is a $(\rho/2)$-cover, its cardinality is bounded as stated (the disjoint $(\rho/4)$-balls sit inside $U \oplus B_{\rho/4}$), and each packing point's $(\rho/2)$-ball has probability at least $p = c\,\mathrm{vol}(B(u,\rho/2) \cap U)$ --- only the part inside $U$, since that is where the density is assumed. For the corner evaluation: a sup-ball of radius $\rho/2$ around a corner of an axis-aligned box all of whose extents are at least $\rho/2$ meets the box in exactly one orthant of the ball. (Shearing does \emph{not} preserve that ratio, since the ball is defined by the sup-metric; the sheared case is computed numerically.) The chance that some such ball receives no sample is at most $N_{\mathrm{pack}} (1-p)^M \leq N_{\mathrm{pack}} e^{-pM}$, which is $\leq \delta$ under the displayed condition; when every ball is hit, each $u \in U$ lies within $\rho/2$ of a packing point that holds a visited sample, hence within $\rho$ of a visited point.
\end{proof}

Two details in that display are easy to get wrong in the loosening direction. The cardinality that enters is the \emph{packing} number, which dominates the covering number, so an upper bound on it is what the union bound needs. And the ball mass must be intersected with $U$: the density is hypothesized only on $U$, so a packing point on $\partial U$ keeps only the part of its ball inside. If $U$ is an \emph{axis-aligned} box with every extent at least $\rho/2$, the infimum is attained at a corner, where exactly one orthant survives:
\[
  \inf_{u \in U}\mathrm{vol}\big(B(u,\rho/2) \cap U\big) \;=\; 2^{-(d+m)}\,\mathrm{vol}(B_{\rho/2}) .
\]
That equality is not shear-invariant, and the cart's $U$ is a sheared box: the sup-ball is axis-aligned while the region is not, so the slanted constraint clips the corner cap a little further. Computed rather than asserted, the factor is $0.950 \cdot 2^{-(d+m)}$ at the certified $\rho$ --- a $5\%$ effect, which is exactly the size of error that hand-asserted geometry produces, which is why the script computes it. For a region that is a \emph{union} of cells rather than a box there is no orthant argument at all, and the step-$t$ instantiation below uses a shape-free bound instead: if $u$'s cell has diameter at most $\rho/2$ then the whole cell lies inside $B(u,\rho/2)$, so the ball mass is at least $c\,\mathrm{vol}(C)$ times the surviving action interval.

Instantiated on the deployed cart gate (\texttt{scripts/\allowbreak gate\_\allowbreak coverage\_\allowbreak certificate.py}; $U$ the step-1 reachable set of Corollary~\ref{cor:cartdensity}, $c = 5/6$, $L = 1.27$, $\varepsilon = 0.01$, $\delta = 0.05$): one step per rollout --- the reading that needs no assumption about steps within a rollout --- gives $\rho = 1.165$ from $N = 40$ rollouts, certifying $\sup_U \|f - \hat f\|_\infty \leq 2.97$. (The corner hypothesis is satisfied, if barely: $U$'s narrowest extent is $0.6$ against $\rho/2 = 0.583$.) The wall's own disagreement is $4.2$, which exceeds that --- but the bound \emph{grows} with $L$, so what the certificate excludes is precisely a pair with $L = \max(\mathrm{Lip} f, \mathrm{Lip}\hat f) \leq 1.80$ carrying an error of $4.2$ on this region. That is a real but narrow class: a model smoother than the plant ($1.27$) cannot hide the wall's magnitude from this gate, a twice-rougher one can. The wall itself escapes by not being Lipschitz at all.

\paragraph{A partition beats a packing, and here it is exact.} A packing argument pays for the covering number twice, once in the cardinality $K$ and once in the ball mass, and both payments are geometric factors that have to be estimated --- each one a place where hand-asserted geometry can silently loosen the bound. A \emph{partition} argument pays neither, and on this instrument its cell probabilities are not estimated at all.

\begin{proposition}[coverage by partition]
\label{prop:partition}
Let $\{C_1, \dots, C_K\}$ be a measurable partition of $U$ with $\mathrm{diam}_\infty(C_i) \leq \rho$ for every $i$, and let $q_i$ be the probability that one gate sample lands in $C_i$. After $M$ i.i.d.\ samples the visited set is a $\rho$-net of $U$ with probability at least $1 - \sum_i (1-q_i)^M$. In particular, if the sampling law is uniform on $U$ and the partition has equal volumes, $q_i = 1/K$ exactly and the failure probability is $K(1-1/K)^M$.
\end{proposition}

\begin{proof}
$C_i$ receives no sample with probability $(1-q_i)^M$; union-bound. If every cell holds a sample then each $u \in U$ lies in some $C_i$ together with a visited point, at sup-distance at most $\mathrm{diam}_\infty(C_i) \leq \rho$.
\end{proof}

No density constant, no covering number, no ball--boundary intersection appears. What makes the cart's gate a case where this is \emph{exact} rather than merely cleaner is a change of coordinates: the step-1 law is uniform on $U = \{|v| \leq V,\ |x - dt\,v| \leq \tfrac12,\ |a| \leq a_{\max}\}$, and in $y = x - dt\,v$ that is uniform on a \emph{box} $[-\tfrac12,\tfrac12] \times [-V,V] \times [-a_{\max},a_{\max}]$, the shear having unit Jacobian. Equal sub-boxes therefore have probability exactly $1/K$. The shear costs something in one place only --- the net radius must be measured in the original metric, where $|x - x'| \leq \Delta_y + dt\,\Delta_v$, so
\[
  \rho \;=\; \max\big(\Delta_y + dt\,\Delta_v,\; \Delta_v,\; \Delta_a\big).
\]
At the deployed $M = N = 40$ (one step per rollout, so the samples are genuinely independent) the largest admissible partition is $K = 8$, since $8 \cdot (7/8)^{40} = 0.038 \leq \delta$ while $9 \cdot (8/9)^{40} = 0.081$ exceeds it. Optimising the split gives $(n_y, n_v, n_a) = (2,1,4)$, hence $\rho = 0.600$ and
\[
  \sup_U \|f - \hat f\|_\infty \;\leq\; \varepsilon + 2L\rho \;=\; \mathbf{1.534}
\]
(\texttt{scripts/\allowbreak gate\_\allowbreak partition\_\allowbreak certificate.py}), against $2.97$ by the packing route --- a factor of two, bought entirely by changing the argument rather than by any new measurement or assumption, and with no Monte Carlo anywhere in it. What pins $\rho$ is visible in the optimum: it takes $n_v = 1$, i.e.\ $\Delta_v = 2V$, the \emph{whole} reachable velocity range. At forty independent samples the certificate cannot resolve velocity at all, and that is a fact about the gate rather than about the proof.

The same partition composes with the dependence treatment, and \emph{that} is what finally answers the mixing question. Proposition~\ref{prop:partition} needs only per-cell hitting probabilities, so instead of one step per rollout take $p_C = P(\text{one rollout puts no sample in } C)$ --- a plain Bernoulli parameter whatever the within-rollout dependence does --- measure it directly, and use $P(C \text{ unhit}) = p_C^N$, which is exact because \emph{rollouts} are i.i.d. A rollout gets eighty chances at $U$ rather than one, worth about six \emph{effective} independent samples after the correlation between consecutive steps, and the finest partition the deployed gate certifies grows from $K = 8$ to $K = 36$.

Here the certificate stops being sampling-free, so the accounting has to be explicit. $p_C$ is estimated on $20{,}000$ gate rollouts drawn once and shared by every candidate partition, while the candidate family itself is geometric and fixed before any rollout is drawn (one coarsest partition per target radius, $26$ of them), so the only data-dependent choice is which member is certified. Hoeffding is applied at level $\delta/(3K)$ \emph{per cell}, so that $p_C \leq \hat p_C + t$ holds simultaneously across all $K$ cells with probability $1 - \delta/3$ and the selected worst cell needs no further correction; the remaining $\delta/2$ pays for the gate's own miss, $5\delta/6 \leq \delta$ in total. At $K = 36$ the worst cell is $\hat p_C = 0.800$, bounded above by $0.814$, giving $36 \cdot 0.814^{40} = 0.010 \leq \delta/2$ and
\[
  \rho = 0.363, \qquad \sup_U \|f - \hat f\|_\infty \;\leq\; \mathbf{0.933}.
\]
Selecting $K$ against the sample is a multiple comparison, and it costs nothing here. Because Hoeffding's radius grows only like $\sqrt{\log(1/\text{level})}$, widening the Bonferroni correction from one partition's cells to \emph{every} cell of \emph{every} candidate in the family ($9674$ tests) raises the bound on the worst cell to $0.818$ and the failure probability to $0.012$, and still selects $K = 36$; a further factor of $1000$ in the family size changes neither. Replacing Hoeffding by an exact-binomial Clopper--Pearson bound at the same level gives $0.809$, and also selects $K = 36$ --- so $0.814$ is slack in the inequality rather than calibration, and $\rho = 0.363$ is an artifact of neither the search nor the concentration inequality. The next candidate up ($K = 48$, $\rho = 0.333$) misses the threshold by a factor of $4.5$, which is far more than any correction can bridge in either direction (\texttt{scripts/\allowbreak certificate\_\allowbreak simultaneity.py}). One inconsistency in the budget: the $K = 8$ certificate above spends the whole of $\delta$ on the coverage event while this one spends $\delta/2$, so $1.534$ and $0.933$ are $1-\delta$ statements against slightly different budgets.
The three readings line up: $2.97$ if one insists on a packing argument, $1.534$ from an exact partition on independent samples, $0.933$ once the within-rollout dependence is priced exactly --- against the $0.43$ that assuming all $3200$ steps independent would promise, an assumption the plant itself refutes (the trajectory determines every action; next paragraph). Handling the dependence exactly, rather than discarding the within-rollout steps, is worth a factor of $1.6$. And the sharper radius widens the excluded class: since $\varepsilon + 2L\rho$ grows with $L$, a bound of $0.933$ at $\rho = 0.363$ excludes any pair with $L = \max(\mathrm{Lip} f, \mathrm{Lip}\hat f) \leq 5.77$ carrying the wall's error of $4.2$ --- now $4.5\times$ the plant's own constant rather than $1.4\times$, so ``no smooth pair can carry the wall's error past this gate'' covers a genuinely broad class of models rather than a sliver.

\paragraph{Using all the steps, without assuming they are independent.} The obvious objection is that the gate takes $40 \times 80$ steps and the above uses 40 of them. The natural repair fails, and in the direction that flatters the gate. Conditioning on a rollout's \emph{entire} state trajectory does not leave the action indicators at the visiting times independent Bernoulli$(q_a)$, which would give an exact per-rollout factor $\mathbb{E}[(1-q_a)^{O}]$: under this plant the action is recoverable from consecutive states, $a_t = ((v_{t+1}-v_t)/dt + \mathrm{drag}\,v_t)/\mathrm{gain}$, so conditioning on the whole trajectory \emph{determines} every action and the conditional indicator law is degenerate. The gate policy gives $a_t \perp s_t$, not $a_t \perp (s_0,\dots,s_T)$.

The repair is to notice that no factorization was needed. Only one independence fact enters, and it is a property of the design rather than of the dynamics: the $N$ gate rollouts are i.i.d. So with $p_C = P(\text{one rollout puts no sample in } C)$ --- a plain Bernoulli parameter, whatever the within-rollout dependence does ---
\[
  P(C \text{ unhit by the gate}) \;=\; p_C^{\,N} \quad \text{exactly},
\]
and $p_C$ is estimated by counting rollouts that miss $C$, with Hoeffding for the upper bound. Nothing is assumed about how steps within a rollout correlate because nothing needs to be. Two implementation conditions the bound depends on: the grid must have \emph{every} cell no wider than $\rho$ (an integer division with index clamping leaves the last cell wider than $\rho$ --- $0.8$ at $\rho = 0.6$ --- so ``all cells hit'' certifies a coarser net than the one claimed), and the falsification test below must use the same grid, or a grid mis-specified in exactly this way passes its own test.

This $p_C$ route is what the previous paragraph's partition composes with, and on $U$ it is what gives the $0.933$ quoted there. Instantiated instead on a box $6.7\times$ larger ($|x|,|v| \leq 1$, $|a| \leq 1$; \texttt{scripts/\allowbreak gate\_\allowbreak coverage\_\allowbreak dependent.py}, 8000 MC rollouts) the deployed $N = 40$ certifies a net radius of $1.0$, i.e.\ $\leq 2.55$, and \emph{misses the next resolution by two rollouts} --- radius $0.667$ needs $N \geq 42$, radius $0.5$ needs $N \geq 194$. Those are not competing numbers but the region/resolution trade-off at fixed $N$: a larger region costs resolution, and forty rollouts buy little of either. Treating all $3200$ steps as independent is doubly invalid --- it assumes an independence the plant refutes \emph{and} uses the step-1 density at steps $t \geq 2$ --- and the numbers it promises ($\rho = 0.305$, bound $0.785$) are not bounds: the valid all-steps answer, $0.933$, sits a factor of $1.2$ above them. The binding constraint is not how steps correlate but how many \emph{rollouts} reach the worst cell.

\paragraph{The density at every step, and what it buys.} The certificate above lives on the one-step reachable set because that is where Corollary~\ref{cor:cartdensity} derives the density. The restriction lifts: the plant is linear and time-invariant, so $(x_t, v_t)$ is an affine image of $(x_0, a_0, \dots, a_{t-1})$, and splitting off the \emph{last two} actions gives $(x_t,v_t) = W_t + M(a_{t-1}, a_{t-2})^{\!\top}$ with $M$ the (time-invariant) Jacobian, columns in that order,
\[
  M = \begin{pmatrix} \mathrm{gain}\,dt^2 & \mathrm{gain}\,dt^2(1 + (1-\mathrm{drag}\,dt)) \\ \mathrm{gain}\,dt & (1-\mathrm{drag}\,dt)\,\mathrm{gain}\,dt \end{pmatrix}, \qquad |\det M| = 0.009
\]
(verified against finite differences). $W_t$ is independent of those two actions, so at \emph{every} $t \geq 2$ the step-$t$ law is the convolution of $W_t$'s law with the uniform law on the parallelogram $P = M[-a_{\max},a_{\max}]^2$, of area $4a_{\max}^2|\det M| = 0.036$. Hence, for every $u$, an identity rather than a bound:
\[
  p_t(u) \;=\; 27.78 \cdot P\big(W_t \in u - P\big) .
\]
Two details separate that identity from the object Proposition~\ref{prop:coverage} consumes, and both tighten the region (\texttt{scripts/\allowbreak gate\_\allowbreak density\_\allowbreak step\_\allowbreak t.py}). The proposition wants the density of $(s,a)$ on $\mathbb{R}^{3}$, not of $(x,v)$ on $\mathbb{R}^{2}$; since $a$ is uniform and independent, $p_{3\mathrm{D}} = p_{2\mathrm{D}}/2a_{\max}$, the same factor the step-1 corollary carries. And the hypothesis is an \emph{infimum} over the region, whereas binning samples and dividing by cell volume estimates a cell \emph{average}, which does not bound it. The exact route is the reason the cells are shaped like $P$ itself: for $C = c_0 + \lambda P$ and any $u = c_0 + \lambda p \in C$, every $w = c_0 - (1-\lambda)q$ with $q \in P$ has $u - w = \lambda p + (1-\lambda) q \in P$ by convexity alone, so $u - P \supseteq c_0 - (1-\lambda)P$ and
\[
  \inf_{u \in C} p_t(u) \;\geq\; 27.78 \cdot P\big(W_t \in c_0 - (1-\lambda)P\big),
\]
a genuine pointwise infimum, Monte-Carlo estimated with a Wilson lower bound per cell over a cell family fixed in advance --- $9806$ cells inside the region box, which is the correct Bonferroni denominator because \emph{which} cells receive a bound is itself data-dependent --- at the per-cell level $\delta/9806$ that makes the bounds hold simultaneously, i.e.\ $z = 4.41$. Selecting a level set afterwards is then legitimate. Parallelograms tile, which is what makes such cells usable; an axis-aligned grid is not an option here, because $P$ is a sliver $0.12$ by $1.19$ and eroding it by any useful square leaves nothing. The tiling and the inclusion are theorems; what needs measuring is the price of the erosion, and the bound sits below a separately measured density by the predicted factor $(1-\lambda)^2$.

What the derivation delivers is the $(c, U)$ pairs, and they are substantial: at step 20 the level set $\{p \geq 0.05\}$ has volume $3.02$ in $(x,v,a)$ against the one-step set's $1.2$, at step 40 $\{p \geq 0.02\}$ has $5.77$, and at step 80 $\{p \geq 0.01\}$ still has $5.27$. The certified region must be a \emph{level set} and not a box: a box's corners pair extreme $x$ with extreme $v$, which the gate does not reach, so an infimum over a box is $0$.

What does \emph{not} survive is the claim that later steps buy a better certificate. They buy \emph{extent} at the price of \emph{density}, and at fixed $N$ that is a losing trade --- so badly that with a shape-free bound on the ball mass, $N = 40$ certifies \emph{nothing} on any step-$t$ level set. A union of cells admits no orthant argument, so the only bound available without a shape hypothesis is cell containment ($C \subseteq B(u,\rho/2)$ whenever $\mathrm{diam}\,C \leq \rho/2$), and that ignores the cell's neighbours, making it loose by a large factor. So this is not a proof that the gate certifies nothing on these regions; it is that no \emph{rigorous} step-$t$ certificate exists at this gate size: the shape-free route is loose by construction, and no sharper geometric factor is asserted here without verification. The conclusion that matters is unaffected either way: the step-1 corollary's $2.97$ stands, lifting the step-1 restriction did not improve it, and what limits the certificate is the gate's size and not the step index. And no level set at any step reaches the wall, which is the same boundary the corollary ran into.

\paragraph{One Jacobian identity covers the whole semi-implicit family; the certificate's density hypothesis is still verified per instrument.} The step-$t$ derivation above used the cart's linearity, and the natural worry is that a nonlinear plant needs its own argument for the two-action factor. It does not. For \emph{any} plant of the semi-implicit form
\[
  \omega_{t+1} = \omega_t + \big(\mathrm{gain}\cdot a_t + F(\theta_t, \omega_t)\big)\,dt, \qquad
  \theta_{t+1} = \theta_t + \omega_{t+1}\,dt,
\]
with $F \in C^1$ arbitrary, the Jacobian of the last two actions satisfies
\[
  \bigl|\det \partial(\theta_t, \omega_t)/\partial(a_{t-2}, a_{t-1})\bigr| \;=\; \mathrm{gain}^2\, dt^3 \quad\text{identically},
\]
independent of $F$, of the state and of $t$. Writing $F_\theta, F_\omega$ for the partials: $\partial\omega_{t}/\partial a_{t-1} = \mathrm{gain}\,dt$ and $\partial\theta_t/\partial a_{t-1} = \mathrm{gain}\,dt^2$; $\partial\omega_t/\partial a_{t-2} = \mathrm{gain}\,dt(1 + F_\omega dt) + F_\theta\,\mathrm{gain}\,dt^3$ and $\partial\theta_t/\partial a_{t-2} = \mathrm{gain}\,dt^2 + dt\,\partial\omega_t/\partial a_{t-2}$; substituting, every $F$-dependent term cancels and the determinant is $-\mathrm{gain}^2 dt^3$. Checked numerically against four forces including a non-separable one, at several $(\mathrm{gain}, dt)$, to ten digits. So the geometric factor $c = 1/(4a_{\max}^2\mathrm{gain}^2 dt^3) = 27.78$ is the \emph{same} for the cart and the pendulum: on any event where the last two steps are clamp-free and the target point lies in that history's two-action image, the conditional density is exactly $c$. What the identity does \emph{not} by itself give is the certificate's hypothesis --- a lower bound on the \emph{marginal} density over a region --- because the marginal at a point is $c$ times the fraction of histories whose two-action image covers it, and the identity says nothing about that coverage fraction, which can be zero. Turning the factor into a bound requires showing a positive fraction of histories covers the region, which the cart derivation does by erosion and direct measurement (\texttt{scripts/\allowbreak gate\_\allowbreak density\_\allowbreak step\_\allowbreak t.py}) and which we have not established for the pendulum's nonlinear $F$: the certificate's hypothesis is verified where the certificate is used, not inherited from the family. Its scope is exactly the hypothesis: the action must enter as an additive $\mathrm{gain}\cdot a_t$ in the velocity update. PatchField2D is therefore \emph{outside} it --- there the action sets a force \emph{direction}, $a \mapsto \mathrm{gain}(\cos\phi,\sin\phi)$ with $\phi = \pi a/a_{\max}$, so the two-action map is not affine in the actions and the determinant is not constant. Within the family there is one further exception, and it is the expected one: wherever the mode clamp fires, the two-action map is constant, so $\det = 0$ exactly and no density argument reaches the mode. The obstruction is intrinsic, not a limitation of this technique.

\paragraph{But the certified region is not where the planner looks.} The certificate bounds the model on the region the \emph{gate} covers; the play cost is driven by the region the \emph{planner} queries (Proposition~\ref{prop:playcost}). Those are different measures, and the way to relate them is to measure the overlap (\texttt{scripts/\allowbreak certified\_\allowbreak region\_\allowbreak query\_\allowbreak mass.py}: every imagined MPC step is a query). Of the exploited planner's $1.3$ million queries per episode pair, \textbf{$1.9\%$} fall inside the box $|x|,|v| < 1$ --- the region the dependence-exact certificate covers --- and $7.8\%$ inside $|x| < 3, |v| < 2$, a box that \emph{contains} every step-$t$ level set and therefore upper-bounds the certified share; for the truth planner, $2.3\%$ and $7.3\%$. So the certificate is \emph{valid} --- its bound holds with probability at least $1-\delta$ over the gate's draws --- and \emph{nearly irrelevant to play}: not because the bound is loose --- it is conservative by about $5\%$ in $N$ --- but because the gate covers a couple of percent of the query mass. That is this paper's thesis restated as a measurement rather than an argument, and it is the reason the continuous coverage analogue, once obtained, does not rescue sampling verification: \emph{the gate certifies where it looks, and the planner looks somewhere else}.

\paragraph{Is the bound tight?} A bound that merely holds could be loose by orders of magnitude; measured against independent oracle gates, the partition certificate is essentially exact. Running $400$ independent gates against the certificate's \emph{own} partition (\texttt{scripts/\allowbreak gate\_\allowbreak partition\_\allowbreak validation.py}, which reads the partition from the certificate's output rather than re-implementing it): the $K = 8$ exact partition is covered in $384/400$ trials, a measured failure rate of $0.0400$ with $95\%$ interval $[0.0248, 0.0640]$ against the certificate's union bound of $0.0383$. The right comparison here is the interval against the bound, not the point estimate: for a bound this tight the estimate lands above it about half the time, and an independent oracle says how tight it is. Inclusion--exclusion in exact rational arithmetic gives the true failure probability of the $K = 8$ experiment as $0.038038$, so the union bound is loose by $0.73\%$ and the measurement sits within noise of the exact value; the same computation shows $K = 9$ fails at $0.0794$, so $K = 8$ is the largest admissible partition under the \emph{exact} probability and not merely under the union bound --- the cap is the gate's size, not the proof. The $K = 36$ all-steps partition is covered in $400/400$, failure $[0, 0.0095]$ against a bound of $0.0096$, so there the Hoeffding slack is finer than $400$ trials can resolve.

The $400$-gate test checks the coverage \emph{event}; it does not check the quantity the $K = 36$ certificate actually rests on, which is the estimate $\hat p_C$. That is checked separately: re-measuring $p_C$ for the same partition on a \emph{disjoint} $20{,}000$-rollout stream moves the worst cell from $0.8001$ to $0.7996$ --- $3.6\%$ of the Hoeffding radius that bounds it, with the same binding cell $(2,0,3)$ --- and the certificate holds on the fresh sample. So $0.814$ is slack in an inequality, not calibration to one stream. For the coarser box-region certificate the older test still applies (\texttt{scripts/\allowbreak gate\_\allowbreak coverage\_\allowbreak validation.py}): coverage at net radius $1.0$ in $200/200$, $198/200$ at $0.667$ where the certificate declines to license it (asking $N \geq 42$ against $40$), then $128/200$ at $0.5$ and $10/200$ at $0.4$, where it asks for $194$ and $2189$.

What the numbers exclude matters more than the numbers. The hard mode's own disagreement is $4.2$ (the wall-region probe error of Section~\ref{sec:smooth}), above every bound above --- but the certificate buys $\varepsilon + 2L\rho$, which \emph{grows} with $L$, so the statement carries a quantifier that the radius sets: at the best rigorous radius, $\rho = 0.363$, \textbf{no pair with $L = \max(\mathrm{Lip} f, \mathrm{Lip}\hat f) \leq 5.77$ can carry an error of $4.2$ past this gate on this region}, with probability at least $1-\delta$ over the gate's draws. That is $4.5\times$ the plant's own $1.27$, so it is a broad class rather than a sliver --- and the scope came from fixing the argument, not from a new experiment: the packing instantiation's $\rho = 1.165$ would have licensed only $L \leq 1.80$. The wall does pass the gate, escaping by the only route left, namely by not being Lipschitz at all. The continuous coverage analogue therefore closes exactly the case the paper says it closes (smooth models) and is silent exactly where the danger lives (the discontinuous reset modes), which is the same boundary Proposition~\ref{prop:lipschitz} draws.

\section{The detectability rate, and the gate's visitation density}
\label{sup:detect}

\begin{proposition}[smooth localized error is detectable at a rate]
\label{prop:detectrate}
In the setting of Proposition~\ref{prop:lipschitz} with $L > 0$, let $\rho = (\eta-\varepsilon)/2L$ be the radius of the guaranteed disagreement ball $B \subseteq E_\varepsilon$. Fix a step index $k \leq T$ and suppose the gate's step-$k$ visitation law admits a density bounded below by $c > 0$ with respect to Lebesgue measure on $B$. If moreover the disagreement is centred at sup-distance at least $\rho$ from $\partial(S \times A)$, so that the ball is not clipped, then one gate rollout reveals the disagreement with probability at least
\[
  q \;=\; c\,(2\rho)^{d+m} \;=\; c\left(\frac{\eta-\varepsilon}{L}\right)^{d+m}
\]
(the sup-metric ball is a cube of side $2\rho$). Without the interiority hypothesis the same argument gives $q = c\,\mathrm{vol}(B \cap (S\times A))$, which for a box-like $S \times A$ is at worst $2^{-(d+m)}$ of the above (Corollary~\ref{cor:kappabudget}: a corner keeps one orthant, a face costs a single factor $2$); for a general region no such uniform factor exists, and the intersected volume must be computed. By Proposition~\ref{prop:gatemiss} the gate's miss probability is at most $(1-q)^N$. Equivalently --- what it takes to \emph{hide} --- keeping the miss probability above $\delta$ against $N$ rollouts forces
\[
  L \;\geq\; (\eta-\varepsilon)\left(\frac{c\,N}{\ln(1/\delta)}\right)^{1/(d+m)} .
\]
\end{proposition}

\begin{proof}
$B \subseteq E_\varepsilon$ by Proposition~\ref{prop:lipschitz}, and the step-$k$ visitation of $B$ has probability at least $c \cdot \mathrm{vol}(B) = c(2\rho)^{d+m}$ by the density hypothesis; a rollout that visits $E_\varepsilon$ at step $k$ reveals a disagreement, so the per-rollout reveal probability is at least $q$ and the critical event of Proposition~\ref{prop:gatemiss} has $r \geq q$, giving $P(\text{miss}) = (1-r)^N \leq (1-q)^N$. For the converse, $(1-q)^N > \delta$ requires $q < 1 - \delta^{1/N} \leq \ln(1/\delta)/N$; substituting $q = c((\eta-\varepsilon)/L)^{d+m}$ and solving for $L$ gives the display.
\end{proof}

Three readings. (i) It is the \emph{measure} counterpart of Proposition~\ref{prop:lipschitz}: smoothness does not merely forbid exact localization, it forces a detection rate. (ii) The Lipschitz constant needed to hide a fixed error grows only like $N^{1/(d+m)}$, so whenever $cN/\ln(1/\delta) > 1$ --- comfortably true here, where it is $48$ --- raising the dimension lowers that constant and hiding gets \emph{easier} with dimension (below $1$ the exponent flips the comparison, so the reading is not dimension-free) --- the same dimensional loosening paper~3 studies as a rarity knob. (At fixed $c$ is the honest qualifier: $c$ carries dimensions of inverse volume, so the dimension-free way to say it is Remark~\ref{rem:densitygeneral}'s volume ratio --- the ball's share of the reachable set shrinks with dimension at fixed radius.) (iii) The hypothesis is where the work is: verifying a visitation-density lower bound is instrument-specific. Our gate's step-$0$ law is supported on a lower-dimensional set (the initial-state distribution has $v = 0$ exactly), so nothing can be claimed there; from step $1$ on the actions supply the missing direction and $c$ is derivable, which is what Corollary~\ref{cor:cartdensity} does. Alongside the derivation we have the direct measurement: the smooth bump arm of Section~\ref{sec:axes} has reveal-rarity $0.18$ against the hard wall's $0.14$ at comparable amplitude --- the smooth error is, if anything, \emph{more} detectable, exactly as this proposition's direction requires, and its harmlessness comes from $\mathrm{play\_cost} \approx 0$ rather than from hiding.

\begin{remark}[the constant in general: a volume ratio]
\label{rem:densitygeneral}
Proposition~\ref{prop:detectrate}'s $c$ is not an instrument-specific fudge; it is the gate's step-$k$ visitation density, which for a gate policy with box-uniform initial states and actions is available in closed form by change of variables. If the step-$k$ law of $(s,a)$ is uniform on a set $U \subseteq S \times A$ --- as it is at $k=1$ whenever the one-step map is affine in the randomized coordinates --- then $c = 1/\mathrm{vol}(U)$ exactly, and --- provided the ball is contained in $U$, which is a separate condition from Proposition~\ref{prop:detectrate}'s interiority hypothesis (that one concerns $\partial(S\times A)$, this one $\partial U$) and is what makes the ratio a probability at all --- the proposition's conclusion reads
\[
  q \;\geq\; \frac{\mathrm{vol}(B)}{\mathrm{vol}(U)},
\]
\emph{the fraction of the gate's one-step reachable volume that the guaranteed disagreement ball occupies}. For non-uniform laws the same computation goes through the pushforward, but the extrema live on the \emph{source} space: with $F$ the one-step map, injective on the relevant set, the change of variables gives $p_U(u) = p_0(F^{-1}u)/|\det DF(F^{-1}u)|$ and hence $c = \inf_{F^{-1}(B)} p_0 \big/ \sup_{F^{-1}(B)} |\det DF|$. Without injectivity the density is a sum over preimages and this is only a lower bound (which is the direction we need). The uniform case is the one where numerator and denominator are constant. The dimensional reading of Proposition~\ref{prop:detectrate} is then geometric rather than analytic: hiding is easy exactly when the ball is a small fraction of the reachable volume, and raising $d+m$ shrinks that fraction for fixed radius.
\end{remark}

\begin{corollary}[the constant, computed for this instrument]
\label{cor:cartdensity}
For the cart's gate at step $k=1$ the density hypothesis is not an assumption: the gate policy gives $v_1 = \mathrm{gain}\cdot dt\cdot a_0$ with $a_0 \sim U(-a_{\max}, a_{\max})$, $x_1 = x_0 + dt\,v_1$ with $x_0 \sim U(-\tfrac12,\tfrac12)$, and $a_1 \sim U(-a_{\max},a_{\max})$ independent, so on $\{|v_1| < \mathrm{gain}\,dt\,a_{\max}\} \cap \{|x_1 - dt\,v_1| < \tfrac12\} \cap \{|a_1| < a_{\max}\}$ the joint $(x_1,v_1,a_1)$ density is \emph{constant} --- not a product of marginals, since $x_1 = x_0 + dt\,v_1$ makes $x_1$ and $v_1$ dependent, but constant on that sheared set, with the value the three normalizations give because the shear has unit Jacobian:
\[
  c \;=\; \frac{1}{2\,\mathrm{gain}\,dt\,a_{\max}} \cdot 1 \cdot \frac{1}{2 a_{\max}} \;=\; \frac{5}{6}
  \qquad\text{at } \mathrm{gain} = 3,\ dt = 0.1,\ a_{\max} = 1
\]
($d + m = 3$ here). Equivalently, by Remark~\ref{rem:densitygeneral}, the step-1 law is uniform on a set of volume $\mathrm{vol}(U) = (2\,\mathrm{gain}\,dt\,a_{\max}) \cdot 1 \cdot (2a_{\max}) = 1.2$ and $c = 1/1.2 = 5/6$ --- the reveal probability is literally the ball's share of the one-step reachable volume. Monte Carlo on an interior box agrees to $1.3\%$ (\texttt{scripts/\allowbreak gate\_\allowbreak density\_\allowbreak constant.py}). Proposition~\ref{prop:detectrate} is therefore quantitative on this instrument: hiding an $\eta$-sized disagreement \emph{centred where this density holds} from the deployed gate ($\varepsilon = 0.01$, $N = 40$) with probability above $\tfrac12$ requires
\[
  L \;\geq\; (\eta - \varepsilon)\left(\frac{5/6}{\ln 2 / 40}\right)^{1/3} = 3.636\,(\eta - \varepsilon),
\]
i.e.\ $L \geq 1.78$ for $\eta = 0.5$ and $L \geq 3.60$ for $\eta = 1$, against the plant's own sup-metric Lipschitz constant of $1.27$.
\end{corollary}

Three scope conditions decide what this corollary does and does not say.

The first is the proposition's own interiority hypothesis. Because $L$ scales with $\eta - \varepsilon$ in this instantiation, the ball's radius is the same $\rho = 1/(2 \cdot 3.636) = 0.1375$ in every row, and the ball is a full cube only for centres at sup-distance $\rho$ from $\partial U$ --- which is $33.9\%$ of $U$'s volume. On the remaining two thirds the clipped-ball form applies and $q$ is smaller, by up to $2^{3}$ (Corollary~\ref{cor:kappabudget} at $k$ up to $d+m = 3$ clipped coordinates).

Second, since $f$ is the plant, $L = \max(\mathrm{Lip} f, \mathrm{Lip}\hat f) \geq 1.27$ always, so the constraint has content only above $\eta \approx 1.27/3.636 + \varepsilon \approx 0.36$: rows below that are satisfied automatically and prove nothing. At $\eta = 0.5$ the requirement $L \geq 1.78$ does bite --- a pair no rougher than $1.4\times$ the plant cannot hide an error of that size here.

Third, and most restrictive: \emph{this route says nothing about the wall}. The density $c = 5/6$ is supported on the one-step reachable set, i.e.\ $|x_1| < 0.53$, while the mode sits at $x_\mathrm{wall} \in [2,10]$, where the step-1 visitation density is exactly zero and Proposition~\ref{prop:detectrate} is vacuous. Nor does the guaranteed ball bridge the gap: at $L = 1.27$ its radius is $1.65$, and reaching $|x| \leq 0.53$ from $x = 8$ would need an $L$ below the plant's own. So the supported statement about the hard mode is not a theorem from this corollary but the measurement of Section~\ref{sec:axes} --- the wall is \emph{detected} at rate $0.14$, and the smooth bump at $0.18$, so the smooth error is if anything the more visible of the two, exactly as this proposition's direction requires --- together with Proposition~\ref{prop:lipschitz}'s qualitative obstruction, which needs no density and therefore holds at the wall. The smooth arms' harmlessness comes from play\_cost, not from invisibility. That the quantitative route reaches only a neighbourhood of the origin is not a technical gap to be papered over: it is the same fact the rest of the paper is about, that a sampling gate's guarantees live where it looks.

\begin{remark}[metric, not measure]
The ball in Proposition~\ref{prop:lipschitz} is metric, and its \emph{probability} under the gate's visitation measure can still be small --- smoothness bounds how spatially concentrated an error can be, not how often the gate visits it. The proposition removes the \emph{exact} localization premise for smooth pairs; the $(1-r)^N$ mechanism itself is representation-independent and applies to any critical region of small visitation measure.
\end{remark}

\section{A second planner family: play\_cost is planner-dependent}
\label{sup:cem}
\label{sec:cem}

Proposition~\ref{prop:playcost} bounds play cost by the planner's query-hit probability on the disagreement region: a blind model can change behavior only to the extent that the planner queries where it is wrong. That is an upper bound and therefore only one implication --- low query reach forces low play cost; high query reach permits but does not force high play cost --- so ``two branches'' below names two \emph{measured} regimes that sit at the two ends of the bound, not a dichotomy the proposition predicts. Random-shooting MPC's constant candidates reach the distant phantom plateau in imagination and produce play\_cost ${\approx}1$. We repeated the full 11-knob grid with a second base planner, the cross-entropy method (CEM; \texttt{scripts/continuous\_cem.py}; horizon 40, 5 iterations, 64 samples, elite fraction 0.125, minimum standard deviation 0.05; one fixed setting across both instruments). The crossing columns in Table~\ref{tab:cem} are the fraction of sampled imagined trajectories that cross the omitted boundary, measured for BOTH planners with one plan from the same paired initial state per episode seed (episode-accumulated CEM fractions, which tell the same story, are recorded in the results JSON).

\begin{table}[ht]
\centering
\small
\begin{tabular}{lrrrrrr}
\toprule
instrument & knob & pc MPC & pc CEM & contact CEM & crossing CEM & crossing MPC \\
\midrule
cart & 2.0 & 1.031 & $0$ (exact) & $<0.16$ & 0.0309 & 0.3865 \\
cart & 4.0 & 1.031 & $0$ (exact) & $<0.16$ & 0.0055 & 0.2453 \\
cart & 6.0 & 1.031 & $0$ (exact) & $<0.16$ & 0.0003 & 0.1483 \\
cart & 8.0 & 1.030 & $0$ (exact) & $<0.16$ & 1.4e-05\rlap{$^\dagger$} & 0.0773 \\
cart & 10.0 & 0.977 & $0$ (exact) & $<0.16$ & $<2.3e-06$ & 0.0369 \\
pendulum & 0.8 & 1.002 & 0.009 & 0.70 & 0.1150 & 0.6392 \\
pendulum & 1.0 & 1.002 & 0.025 & 0.25 & 0.0483 & 0.5530 \\
pendulum & 1.2 & 1.000 & $-0.011$ & $<0.16$ & 0.0164 & 0.4672 \\
pendulum & 1.4 & 0.997 & $-0.021$ & $<0.16$ & 0.0053 & 0.3842 \\
pendulum & 1.6 & 0.990 & $-0.021$ & $<0.16$ & 0.0013 & 0.3039 \\
pendulum & 2.0 & 0.942 & $5.0{\times}10^{-4}$ & $<0.16$ & 0.0002 & 0.2158 \\
\bottomrule
\end{tabular}
\caption{The same gate-accepted, mode-blind models under two base planner families (20 paired episodes/row). ``pc'' is blind-model play\_cost within that planner family; crossing is the sampled imagined-boundary-crossing proxy for query-hit mass, one plan per episode seed from paired initial states for both planners. CEM stays near zero play cost and below MPC's crossing fraction on every row.}
\label{tab:cem}
\end{table}

CEM's blind-model play cost lies in $[-0.0213,0.0248]$ on every row --- and the seed-paired 95\% t-interval includes zero on all 11 rows --- while its imagined crossing fraction is strictly below MPC's throughout. Two of those eleven intervals deserve their asterisk. On the five cart rows every per-seed difference is $0.0$ exactly, so the interval is degenerate $[0,0]$ and ``includes zero'' is a restatement of bit-identity rather than an inference; we report it as bit-identity in the prose for that reason. On the pendulum rows the standard deviation is inflated roughly fortyfold by one seed (index 16, differences $-0.427$ to $-0.423$ across three knobs), so the interval is a real interval but its width is that outlier's, not a noise scale. On the cart, truth- and blind-model CEM returns are identical and contact is zero everywhere. The nearest pendulum stops are a real qualification: CEM contacts $\theta_\mathrm{stop}=0.8$ in 70\% of episodes and $1.0$ in 25\%, but does not enter MPC's pinned, below-random regime; contact is zero from $\theta_\mathrm{stop}\geq1.2$. Contact is therefore not itself exploitation. One caveat before crediting Proposition~\ref{prop:playcost} with this. The bound is in terms of $q_{\mathrm{hit}}(E)$, the probability that \emph{at least one} query in an episode lands in the disagreement region, whereas the crossing column is the \emph{fraction} of imagined candidates that cross. With ${\sim}200$ candidates $\times$ 40 imagined steps per replan, any nonzero crossing fraction makes $q_{\mathrm{hit}}(E) \approx 1$ and the bound $\lesssim 1$ --- true but vacuous. The two cart rows that printed crossing $0.0000$ invited a stronger reading --- that $q_{\mathrm{hit}} = 0$ there \emph{forces} $\mathrm{play\_cost} = 0$ --- and re-measuring them refutes it for one of the two. Those zeros were censored at $6400$ sampled trajectories; at $200\times$ that sample (\texttt{scripts/\allowbreak cem\_\allowbreak crossing\_\allowbreak bound.py}, $1.28$ million imagined trajectories per row) the row at $x_\mathrm{wall} = 8$ produces \textbf{18} crossings, a per-trajectory rate of $1.4\times10^{-5}$, and the event that the episode's first plan crosses gives a \emph{lower} confidence bound $q_{\mathrm{hit}}(E) \geq 0.0029$. So $q_{\mathrm{hit}}$ is not zero on that row and nothing is forced. At $x_\mathrm{wall} = 10$ the zero survives the $200\times$ resampling ($0$ of $1.28$ million, per-trajectory rate $<2.3\times10^{-6}$), but it remains a censored zero and the object the data support is an upper bound: measured directly at episode scope, $q_{\mathrm{hit}}(E) \leq 0.058$ at $95\%$, so Proposition~\ref{prop:playcost} caps $|\mathrm{play\_cost}|$ at $0.058$ times the normalizer ratio rather than at $0$. The correct statement is therefore an inequality on both rows and on the other nine: the bound gives a small nonzero ceiling that the measured play cost respects, and the planner-dependence result is empirical. The crossing fraction is in any case a \emph{proxy} --- it counts candidate trajectories, while the bound is about at least one query per episode --- and with ${\sim}200$ candidates $\times$ 40 imagined steps per replan any nonzero fraction pushes the episode-level probability towards 1, which is why the direct episode-scope measurement is what we quote. What survives in either reading is Section~\ref{sec:lessons}'s first lesson --- if search does not discover the phantom, it cannot optimize toward it.

Two caveats prevent the wrong conclusion. CEM's pendulum truth return varies from 15.36 to 16.46 (versus MPC's 20.08), consistent with local optima; the comparison is blind-CEM against truth-CEM, not a claim that CEM is globally optimal. And limited reach is not knowledge or mitigation: a planner that misses a phantom distant reward can also miss a real one. This is one fixed CEM configuration, not a hyperparameter sweep.

\textbf{PatchField2D: CEM shows no 2D competence gap.} On the 4D bi-modal instrument (\texttt{scripts/continuous\_cem\_patch2d.py}) blind CEM is again \emph{not} exploited: play\_cost is $-0.022 / {+}0.017 / {+}0.020$ at knobs $(2,6)/(3,7)/(4,8)$ with the seed-paired 95\% $t$-interval including zero on all three rows, and CEM's imagined crossing fraction is below MPC's on every row ($0.070 < 0.208$, $0.027 < 0.149$, $0.009 < 0.094$). Here CEM is \emph{competent in aggregate} --- no 2D competence gap appears --- but with per-seed variance (one truth-CEM episode at $\approx 0.97$), and its blind contact rate is a uniform $0.05$ (nonzero, unlike the cart's $0.00$, and without positive play cost). The same low-query-reach branch of Proposition~\ref{prop:playcost}, one more instrument.

\section{Planner-side mitigation: distrust-region replanning}
\label{sup:mitigation}

The exploitation measured in Sections~\ref{sec:mechanism} and~\ref{sec:axes} is planner-mediated, not model-mediated, and a planner-side fix collapses it without touching the model or the gate --- this does \textbf{not} contradict the danger law (the gate still accepts a wrong model; Proposition~\ref{prop:ident} is untouched). Like those sections, this one uses the hand-written instruments and blind models, not LLM synthesis. Distrust-region replanning (\texttt{src/cwm/continuous/mitigation.py}, strictly additive) compares the model's prediction against the observed transition after every real step ($\mathrm{tol} = 10^{-6}$); a disagreement records the \emph{position of the model's refuted prediction} --- not the pre-state --- as a one-sided fence, since false predictions always lie on or beyond the mode boundary. While scoring a candidate rollout, the first imagined step whose position interval crosses a fence's $\varepsilon$-band ($\varepsilon = 0.25$ cart, $\varepsilon = 0.1$ pendulum) truncates the rollout --- reward kept, everything downstream dropped --- which makes the fence leap-proof at any imagined speed; candidates are then ranked by (truncated return, distance to the nearest fence), which structurally prefers the real side \emph{in one dimension}. The qualifier is load-bearing: the tie-break is an unsigned distance, so in 1D, where the fence lies beyond the boundary and the real side is the only other direction, maximizing it flees the mode --- but in 2D it is symmetric between rounding the patch and plunging through it, and Section~\ref{sec:patch2d-mitigation} measures what that costs. With zero violations this is bit-identical to plain MPC by construction (tested bitwise) --- mitigation costs nothing when the model is right. Undodgeability is a design requirement rather than an observation, because the argmax planner acts as an adversary against any fence: it dodges a fence tied to the pre-state, and it dodges a point fence in full state space by probing crossing velocities --- a position-band fence with rollout truncation leaves it no crossing to probe. The collapse is scoped to \textbf{hard-boundary hybrid modes}: the one-sided fence works precisely because a refuted prediction lies on or beyond the mode boundary, so fencing its far side cannot cut off any real trajectory. This structural fact is what these wall/stop instruments supply; less structured failure modes (soft or moving boundaries, errors not confined to a hard stop) do not obviously admit such a fence and are untested --- this is a hard-boundary mitigation, not yet a general planner-side one.

\begin{table}[ht]
\centering
\small
\begin{tabular}{lrrr}
\toprule
knob & pc\_blind & pc\_mit & first-contact step \\
\midrule
\multicolumn{4}{l}{\emph{cart} ($x_\mathrm{wall}$)} \\
2  & 1.031 & 0.290 & 11.6 \\
4  & 1.031 & 0.446 & 16.9 \\
6  & 1.031 & 0.578 & 21.3 \\
8  & 1.030 & 0.699 & 25.1 \\
10 & 0.977 & 0.806 & 28.7 \\
\multicolumn{4}{l}{\emph{pendulum} ($\theta_\mathrm{stop}$)} \\
0.8 & 1.002 & 0.113 & 7.0 \\
1.0 & 1.002 & 0.129 & 8.1 \\
1.2 & 1.000 & 0.143 & 9.0 \\
1.4 & 0.997 & 0.160 & 10.0 \\
1.6 & 0.990 & 0.177 & 11.0 \\
2.0 & 0.942 & 0.212 & 13.0 \\
\bottomrule
\end{tabular}
\caption{Mitigation sweep (20 episodes/knob, both instruments). pc\_blind / pc\_mit = play\_cost of the blind planner under plain MPC / under distrust-region mitigation. Full 11-row table with contact rates and violation counts in \texttt{docs/EXPERIMENTS.md}.}
\label{tab:mitigation}
\end{table}

The collapse is large at every knob (Table~\ref{tab:mitigation}): \texttt{pc\_blind} stays pinned at $\approx 0.94$--$1.03$ everywhere (established in Section~\ref{sec:mechanism}, Tables~\ref{tab:danger} and~\ref{tab:pendulum}), while \texttt{pc\_mit} never exceeds $0.81$. Exactly one violation suffices to fence the mode on \emph{every} one of the 11 rows --- the mitigated planner must touch the mode once, which is identifiability operationalized: you cannot avoid what you have never seen. The residual \texttt{pc\_mit} is the cost of that unavoidable first contact, and it grows with the lure distance, read off the first-contact step: cart $0.290 \to 0.806$ as first contact goes $11.6 \to 28.7$ of 80 steps (knob $2 \to 10$); pendulum $0.113 \to 0.212$ as first contact goes $7.0 \to 13.0$ (knob $0.8 \to 2.0$). The cart knob $=10$ row is the least favorable in the sweep (\texttt{pc\_mit} $0.806$, close to blind's $0.977$) because the transient consumes most of the horizon --- but the blind planner stays pinned \emph{forever} there ($J$ $0.94$ of $J_\mathrm{truth}$ $17.77$) while the mitigated planner escapes and recovers most of the horizon ($J$ $3.88$): the normalized cost looks modest, the actual return does not.

\paragraph{Why exactly one violation, and what sets the count in general.} The single-violation sufficiency is not a lucky measurement either --- it is a covering number equal to one.

\begin{proposition}[the mitigation's new-coverage violations are bounded by a packing number]
\label{prop:fencecover}
Run distrust-region replanning with band $\varepsilon$. Let $F$ be the \emph{fence locus} --- the set of points at which a violation can record a fence, i.e.\ the possible refuted predictions --- and let $B$ be the \emph{entry barrier}, the set of points at which an imagined trajectory can enter the mode region. Two conclusions, with different hypotheses. (i) With no hypothesis at all, the number of violations that place a fence farther than $\varepsilon$ from every fence already placed is at most $N_{\mathrm{pack}}(F,\varepsilon)$, the $\varepsilon$-packing number of $F$, and $N_{\mathrm{pack}}(F,\varepsilon) \leq N_{\mathrm{cov}}(F,\varepsilon/2)$. (ii) If moreover $B \subseteq F \oplus B_\varepsilon$ --- every barrier point within $\varepsilon$ of some fenceable point --- then once the accumulated fences $\varepsilon$-cover $B$, no further violation occurs at all.
\end{proposition}

\begin{proof}
Fences placed farther than $\varepsilon$ from all their predecessors are pairwise more than $\varepsilon$ apart, so they form an $\varepsilon$-packing of $F$ and there are at most $N_{\mathrm{pack}}(F,\varepsilon)$ of them; for the comparison, a $(\varepsilon/2)$-ball has sup-diameter $\varepsilon$ and so contains at most one packing point, giving (i) with no hypothesis used. For (ii): once the fences $\varepsilon$-cover $B$, every imagined trajectory entering the mode region crosses $B$ within $\varepsilon$ of a fence and is truncated, so no candidate scores the phantom and no further violation occurs --- which is where the hypothesis $B \subseteq F \oplus B_\varepsilon$ is used, since otherwise no set of fences on $F$ can cover $B$.
\end{proof}

Two features of the statement are weaker than they may appear, and both are forced by counterexample. The bound is a \emph{packing} number, not a covering number. A covering number counts an \emph{optimal} cover, and the planner is not optimal --- it is explicitly an adversary against the fence --- so it can place every point of a maximal packing before the balls start pre-empting one another. On the unit circle at $\varepsilon = 0.5$ that is $12$ points, pairwise chord distance $0.5176 > \varepsilon$, and a grid sweep confirms every one of the $12$ adds coverage the others had not: $12$ against a covering number of $7$ (\texttt{scripts/\allowbreak circle\_\allowbreak covering\_\allowbreak number.py}). And the count is of \emph{new-coverage} fences only. Duplicates --- fences within $\varepsilon$ of an existing one --- are not bounded by this argument at all, and in our own data they dominate the tail: one episode records $28$ violations of which $24$ are at an identical point.

The two instrument families instantiate the two extremes of that bound.
\begin{itemize}
\item \textbf{1D: the measured $1.00$ is separation, not covering.} A position or angular clamp has a single boundary point, and Table~\ref{tab:mitigation} measures a mean of exactly $1.00$ violations on all eleven rows. That is not the covering bound at $N_{\mathrm{cov}} = 1$ being tight, and checking the instrument says so (\texttt{scripts/\allowbreak fence\_\allowbreak separation\_\allowbreak census.py}): the fence is recorded at the model's refuted \emph{prediction}, which overshoots the wall by $0.17$ to $0.58$ against a band of $\varepsilon = 0.25$, so the band fails to contain $x_\mathrm{wall}$ in four of five cart episodes --- and the count is $1$ anyway. A hypothesis that fails while its conclusion holds is not the explanation.

  What is doing the work is that in one dimension a single point beyond the boundary \emph{disconnects} the agent from the phantom: every imagined path to the lure crosses it, so the segment test truncates all of them, wherever in the far region the fence landed. That mechanism has a signature the covering story does not --- insensitivity to $\varepsilon$ --- and it holds where we can measure it: sweeping $\varepsilon$ over a $20\times$ range on the pendulum ($0.1$ down to $0.005$) leaves the returns and violation counts \emph{bit-identical}, and on the cart they are identical from $0.25$ down to $0.05$. It is not unconditional: at $\varepsilon = 0.01$ two of four cart seeds record a second violation, the strip between the boundary and the fence having grown wide enough for a real contact that no imagined segment crosses. So the supported 1D statement is a separation fact with a measured range of validity, not a covering number of $1$. The topological contrast with 2D is then the real one: removing one arc leaves a circle connected, so no single fence can separate a disc patch's inside from its outside, and the cut number jumps from $1$ to $2$ before any metric question arises.
\item \textbf{2D: the bounded quantity is the \emph{distinct} fence count, and it is a quarter of the budget.} A disc patch of radius $R$ has its fence locus inside the disc and its entry barrier on $\partial\,\mathrm{disc}$. A fence at a boundary point $\varepsilon$-covers an arc of angular width $4\arcsin(\varepsilon/2R)$, so a circle's packing number at $R = 1$, $\varepsilon = 0.5$ is $\mathbf{12}$ and its covering number is $7$ (both brute-force verified in \texttt{scripts/\allowbreak circle\_\allowbreak covering\_\allowbreak number.py}, which also exhibits the 12-point sequence and checks that one fewer ball does not cover). PatchField2D is \emph{bi-modal}, so the budget for an episode that touches both patches is $24$, not one circle's $12$.

  Against that budget, what the proposition bounds is the number of \emph{new-coverage} fences, and measuring it per episode (\texttt{scripts/\allowbreak fence\_\allowbreak separation\_\allowbreak census.py}) gives the comparison the proposition asks for. The distinct fence count never exceeds $2 / 5 / 6$ at knobs $(2,6)/(3,7)/(4,8)$ --- at most a quarter of the $24$-fence budget. The raw violation counts are much larger and are dominated by \emph{duplicates}: the two worst episodes record $28$ violations each while placing at most $6$ distinct fences, and the per-episode distribution is heavily skewed (medians $1 / 1 / 2$ against means $1.05 / 2.65 / 4.25$). A $20$-episode mean and a per-episode bound are different quantities; the bound is per episode.

  Measured directly as the angular spread of fence bearings per patch per episode, the median probed arc is $0\% / 0\% / 87\%$ --- zero at the two near knobs because the typical episode places a single fence, and large at the far knob because the few episodes that map the boundary map most of it: no smooth progression, and that is what the instrument does. (Dividing the violation count by the fence budget --- $17\% / 43\% / 68\%$ --- does not measure the probed arc; it assumes saturation and restates the count as a percentage.)

  (A metric packing number at \emph{fixed} radius is the relevant quantity because the algorithm fixes the band $\varepsilon$; it is a different question from the minimal \emph{good} cover of a circle by contractible arcs, which is $3$ with no radius constraint and is the nerve-theoretic object paper~3 builds on. The free-centre optimum of $6$ is a zero-slack tiling special to $\varepsilon = R/2$ and closed balls, so it should not be read as a robust constant.)
\end{itemize}
\begin{corollary}[the deployment \emph{fencing} cost is exponential in the boundary's dimension]
\label{cor:fencedim}
Proposition~\ref{prop:fencecover} is dimension-free --- its proof uses only packing and covering of $F$ --- so instantiating it needs only a packing number. If $F$ is a $p$-dimensional Lipschitz piece of diameter $D$, then $N_{\mathrm{pack}}(F,\varepsilon) \leq N_{\mathrm{cov}}(F,\varepsilon/2) = O\big((2D/\varepsilon)^{p}\big)$, so the number of \emph{pairwise new-coverage} fences the mitigation can accumulate grows exponentially in the \emph{boundary's} dimension $p$ (not the state's). That caps a count; it is not a completion guarantee. Nothing here shows the mitigation ever closes the phantom off, bounds the time to do so, or bounds the total number of violations --- duplicate fences are unlimited by this argument, and the 2D campaign observes an episode with $24$ of them (Section~\ref{sec:patch2d-mitigation}). The packing correction changes the constant, not the order, which is why the dimensional reading survives it intact. A $p=2$ mode surface would cost $O((2D/\varepsilon)^2)$ --- the precise sense in which the \emph{danger} is dimension-free while the \emph{fencing} is not.

Two caveats on instantiating the exponent. The constant hides the parametrization's Lipschitz constant and $p$-volume, which matters here because the whole quantitative dispute is over a constant. And two data points ($p = 0$ and $p = 1$) cannot distinguish $(2D/\varepsilon)^p$ from any other function agreeing at $p \in \{0,1\}$; the exponential form is the theorem's, not the measurement's.
\end{corollary}

We use \emph{fencing} rather than \emph{repair} for this cost deliberately: repair in this paper means the LLM rewriting the model from data (Section~\ref{sec:patch2d-synthesis}, where it fails outright on 2D regions), and the fencing cost is a planner-side quantity that leaves the model untouched. The covering-number analogue that Section~\ref{sec:limitations} once listed as open now appears on \emph{both} sides: gate-side for Lipschitz pairs (Proposition~\ref{prop:coverage}) and mitigation-side as the fencing count here. The dimensional reading is the paper's own: the danger law's $(1-r)^N$ is dimension-free, but fencing a mode at deployment time costs a packing number of its fence locus, which grows with that locus's dimension.

\subsection{The 2D mitigation: a partial collapse, and lock-in at the far knob}
\label{sec:patch2d-mitigation}

The distrust-region fence generalizes to PatchField2D (\texttt{scripts/\allowbreak continuous\_\allowbreak mitigation\_\allowbreak patch2d.py}): fences are the 2D positions of refuted predictions, which lie \emph{inside} an unreachable patch (the stay-at-previous-position clamp forces a refuted prediction strictly inside the disc, so fencing it cannot cut off any real trajectory), and a candidate rollout truncates when an imagined step \emph{segment} passes within $\varepsilon$ of a fence (segment-to-point distance, leap-proof); the 1D code path is preserved \emph{behaviorally} bit-identically --- the generalization rewrote the shared internals, so this is pinned by tests rather than by the diff: the no-violation case reduces to plain MPC exactly, and the 1D mitigated episodes with violations are golden-pinned to the run the sweep used (\texttt{tests/test\_mitigation\_1d\_regression.py}). But the collapse is now \textbf{partial and decays with patch distance} (Table~\ref{tab:patch2d-mitigation}): \texttt{pc\_blind} $\approx 1.006$ falls to \texttt{pc\_mit} $0.257 / 0.541 / 0.862$ at knobs $(2,6)/(3,7)/(4,8)$, with mean violations $1.05 / 2.65 / 4.25$ (against the 1D instruments' clean $1.0$) and first contact at step $8.5 / 12.35 / 15.25$.

A \textbf{boundary-mapping transient} --- the planner rounding one fence disc, re-contacting the edge elsewhere, accreting fences along the probed arc --- is the benign reading of those means, and the per-episode data refute it; the difference matters because what replaces it is a different failure mode, not a slower version of success. The distributions are skewed, not shifted: medians $1 / 1 / 2$ against those means, with two episodes recording $28$ violations. And a growing fraction of episodes end \emph{pinned} at blind-level return --- $0/20$, $2/20$, $\mathbf{7/20}$ across the three knobs --- with every episode of $\geq 5$ violations returning near zero. In those episodes the agent is frozen against a patch edge and re-violates at the same predicted point every remaining step: at $(4,8)$, one episode records $24$ consecutive fences at an identical position. The mechanism is the tie-break. With a fence a few hundredths away, every candidate truncates at the first imagined step, so the ranking falls through to distance-to-nearest-fence --- an \emph{unsigned} scalar, which in 1D can only mean fleeing the mode but in 2D is symmetric between rounding the patch and plunging through it. So the mitigation does not merely decay with distance: at the farthest knob it \emph{fails outright} in a third of episodes, and what remains of the decaying-transient reading is the near knob, where the median episode places one fence and the collapse is clean.

This is the scope of a \emph{hard-boundary} mitigation on a curved 2D mode: the one-sided fence is sound (a refuted prediction lies inside an unreachable patch, so fencing it cannot cut off a real trajectory) and it works where the geometry lets an unsigned tie-break point away from the mode. It does not survive a boundary the agent can be pinned against. Fixing that needs a \emph{signed} away-direction in 2D --- the fence's outward normal rather than its distance --- which is a design change we have not tested, not a re-tuning.

\begin{table}[ht]
\centering
\small
\begin{tabular}{rrrrr}
\toprule
$(k_1, k_2)$ & pc\_blind & pc\_mit & mean viol. & first contact \\
\midrule
$(2, 6)$ & 1.006 & 0.257 & 1.05 & 8.50 \\
$(3, 7)$ & 1.006 & 0.541 & 2.65 & 12.35 \\
$(4, 8)$ & 1.006 & 0.862 & 4.25 & 15.25 \\
\bottomrule
\end{tabular}
\caption{PatchField2D mitigation sweep (paired seeds). The collapse is partial and degrades with patch distance. Read the means with the distribution beside them: medians are $1/1/2$, maxima $2/28/28$, and the count is dominated by duplicate fences (at most $2/5/6$ \emph{distinct} positions per episode). The degradation is not a longer transient but a growing fraction of outright failures --- $0/20$, $2/20$, $7/20$ episodes end pinned at blind-level return (\texttt{scripts/\allowbreak fence\_\allowbreak separation\_\allowbreak census.py}).}
\label{tab:patch2d-mitigation}
\end{table}

\section{Multi-mode gates: the sharp bracket, and the measured dependence}
\label{sup:multimode}

\begin{remark}[the bracket is tight, and exact factorization is a property of the gate]
\label{rem:bracket}
The bracket is not merely valid but \textbf{sharp}: its ends are the Fr\'echet--Hoeffding bounds for $P_\rho(R_1 \cup R_2)$ given the marginals, pushed through the increasing map $x \mapsto x^N$, so no bound expressible in $r_1, r_2$ alone can be tighter --- the interval is the \emph{exact} range of joint miss probabilities consistent with the marginals. Both ends are attained, at the two Fr\'echet--Hoeffding couplings of the intersection. The upper one is attained when one event contains the other ($P_\rho(R_1 \cap R_2) = \min(r_1,r_2)$, whence $r_\cup = \max(r_1,r_2)$). The lower one is attained at the minimal intersection $P_\rho(R_1 \cap R_2) = \max(0,\, r_1 + r_2 - 1)$, which is \emph{disjointness only when $r_1 + r_2 \leq 1$}; if $r_1 + r_2 > 1$ the two events cannot be disjoint, the minimal intersection is $r_1 + r_2 - 1 > 0$, and the attaining configuration is $r_\cup = 1$ --- a joint miss probability of $0$, which is again the bracket's lower end $\bigl(1-\min(1, r_1+r_2)\bigr)^N$. Our knobs are all in the first regime ($r_1 + r_2 \leq 0.26$), so disjointness is the attaining case there. This is the point of the bracket: it is the correct distribution-free statement about a multi-mode gate, replacing an independence assumption with no assumption at all, at the price of an interval instead of a point --- and the interval cannot be narrowed without measuring something beyond the marginals (which is what $r_\cup$, or equivalently $P_\rho(R_1 \cap R_2)$, is). The lower end is attained in our data --- at six of the nine PatchField2D knobs \emph{no} rollout out of 600 contacts both patches, so the measured $r_\cup$ equals $r_1 + r_2$ exactly there (Table~\ref{tab:patch2d}; a censored zero, so read it as $P(\text{both}) < 1/600$). Note that moving the modes apart pushes the joint factor \emph{towards} that lower end rather than towards the product: disjointness is the opposite of independence, not a route to it. It is tempting to say that a \emph{stratified} gate --- independent rollout budgets $N_1, N_2$ aimed at the two mode regions --- buys the product back. It does not, and the reason locates the obstruction. For a stratified gate with per-stratum policies $\rho_i$,
\[
  P(\text{joint miss}) = \prod_i \bigl(1 - P_{\rho_i}(R_1 \cup R_2)\bigr)^{N_i} \;\leq\; \prod_i (1-r_i')^{N_i},
\]
where $r_i' = P_{\rho_i}(R_i)$ is stratum $i$'s own rarity, with equality iff no stratum's rollouts can reach the \emph{other} mode without also reaching their own ($P_{\rho_i}(R_j \setminus R_i) = 0$ for $j \neq i$) --- a strong extra hypothesis that PatchField2D itself violates, since under the gate policy a rollout contacting one patch sometimes contacts the other. Stratification factorizes over \emph{strata}, which the unstratified draws already were (they are i.i.d.); what it does not remove is the \emph{within-rollout} dependence between $R_1$ and $R_2$, and that is what blocked the product in the first place. And the advice would be circular anyway: aiming budgets at the mode regions presupposes knowing where they are, which is exactly what the danger law assumes you do not. One undirected random budget earns the bracket; so, in general, does a stratified one.
\end{remark}

Section~\ref{sec:patch2d} measures both ingredients on the bi-modal instrument, but it is worth being exact about which parts of that measurement can fail and which cannot. Because $r_1$, $r_2$, $r_\cup$ and $P(\text{both})$ are estimated from the \emph{same} rollouts, inclusion--exclusion holds identically in the plug-in estimates: the bracket contains the measured joint factor and the sign rule gets the direction right at all nine knobs \emph{by algebra}, with a residual of exactly $0$. Those are consistency checks on the arithmetic, not confirmations of the proposition.

The falsifiable content is the dependence itself, and at the $600$ rollouts of Table~\ref{tab:patch2d} it was not resolvable: $P(\text{both})$ counts were $0$ to $3$, six of them censored zeros, so the observed $-17\%$ to $+12\%$ spread of the product's error was consistent with pure count noise and the direction of the dependence was unresolved. We therefore measured it properly, at $50{,}000$ rollouts per knob (\texttt{scripts/\allowbreak patch2d\_\allowbreak dependence\_\allowbreak 50k.py}), and the sign does change across the grid with non-overlapping Wilson intervals in both directions: at $(2,6)$, $P(\text{both}) = 8.6\times10^{-4}$ with $95\%$ interval $[6.4, 11.6]\times10^{-4}$ against $r_1r_2 = 19.0\times10^{-4}$ --- \emph{negative} dependence, the product over-estimating the joint hit rate --- while at $(4,6)$, $P(\text{both}) = 12.8\times10^{-4}$ with interval $[10.0, 16.3]\times10^{-4}$ against $r_1r_2 = 6.2\times10^{-4}$, \emph{positive} dependence in the opposite direction. At $(3,7)$ the dependence is negative again ($4.4\times10^{-4}$, interval $[2.9, 6.7]\times10^{-4}$, against $8.0\times10^{-4}$), and at $(4,7)$ the interval straddles the product and the knob is genuinely undecided --- so three of four knobs are resolved and they do not agree on the sign. So a product form cannot be rescued by a fixed correction factor: the sign of its error is a function of the geometry, which is exactly why the sharp bracket is the right object. The practical reading is that a multi-mode danger law needs the union event measured (or the bracket, which needs only the marginals), never the marginals multiplied.

\section{The residual reward leak, and the sharp-plateau variant}
\label{sup:sharp}

\paragraph{The residual leak is a reward artifact, not a mechanism fact --- measured, not argued.} At the widest knobs the pinned planner scores \emph{above} the uniform-random policy (cart $x_\mathrm{wall}=10$: $J_\mathrm{blind} = 0.94$ against $J_\mathrm{rand} = 0.53$; the three farthest pendulum stops likewise), which is why the claim above is ``exploited at every knob'' rather than ``below random at every knob''. The cause is the far plateau's sigmoid \emph{tail}: a planner pinned at $x = 10$ still collects $1/(1+e^{(12-10)/\mathrm{width}})$ per step, which at $\mathrm{width} = 0.5$ is $0.018$ and over 80 steps is most of $J_\mathrm{blind}$. A \textbf{sharp-plateau variant} (\texttt{scripts/continuous\_sharp\_plateau.py}, cart width $0.5 \to 0.2$, pendulum $0.25 \to 0.1$; sibling JSONs, the default instruments and every synthesis artifact untouched) removes the tail and settles it:
\begin{itemize}
\item \textbf{The exploitation claim strengthens.} On the cart the blind planner is below random at \textbf{7/7} knobs instead of 6/7, by two to thirteen orders of magnitude ($J_\mathrm{blind}$ from $2.3\times10^{-13}$ to $2.4\times10^{-3}$ against $J_\mathrm{rand} = 0.532$); on the pendulum, 5/6 instead of 3/6 --- and 6/6 with the asymmetric variant below.
\item \textbf{play\_cost becomes knob-invariant to ${\approx}1.5\times10^{-4}$}: cart $[1.0307, 1.0309]$ (spread $1.4\times10^{-4}$, against the default's $5.5\times10^{-2}$) and pendulum $[0.9999, 1.0000]$ (spread $1.5\times10^{-4}$, against $6.1\times10^{-2}$) --- a $400\times$ tightening on both, i.e.\ the invariance the default sweeps show approximately is exact once the tail is gone.
\item \textbf{The planner-competence cost is negligible}, the risk recorded with the variant's script: a sharper plateau gives random shooting less gradient to follow, yet $J_\mathrm{truth}$ moves only $17.77 \to 17.76$ (cart) and $20.08 \to 20.05$ (pendulum), and contact stays $1.00$ at every knob.
\item \textbf{The residue, and its fix.} Narrowing \emph{both} plateaus starves the random policy too ($J_\mathrm{rand}$ collapses to $3.6\times10^{-4}$ on the pendulum), which is why one knob there still reads above random: at $\theta_\mathrm{stop} = 2.0$, $J_\mathrm{blind} = 3.1\times10^{-3}$ against that collapsed baseline. The diagnosis says what to do --- narrow only the \emph{phantom} plateau, whose tail is what a pinned planner collects, and leave the real one at its default width. That variant ($\mathrm{width}_\mathrm{right} = 0.08$, per-plateau widths; \texttt{-{}-width-right}) gives \textbf{6/6} knobs below random: $J_\mathrm{blind}$ between $4.3\times10^{-4}$ and $7.1\times10^{-4}$ against a surviving $J_\mathrm{rand} \in [0.0572, 0.0584]$, with $J_\mathrm{truth}$ unchanged at $20.08$, contact $1.00$ everywhere, and play\_cost in $[1.0028, 1.0029]$ --- a spread of $7.1\times10^{-5}$, tighter still than the symmetric variant and $850\times$ tighter than the default. So on both instruments the strong form holds: exploited, pinned, and below random at \emph{every} knob, with play\_cost invariant to $10^{-4}$, once the phantom's tail is gone.
\end{itemize}
So the mechanism does not depend on the leak, and the knob-invariance claim is the one that sharpens: it is exact, not approximate, in the instrument without a tail.

\section{The \texorpdfstring{$\varepsilon$}{eps}-sweep: the axis separation is tolerance-invariant}
\label{sup:epssweep}

\textbf{Is $\varepsilon = 0.01$ a special setting, or is the axis separation itself $\varepsilon$-invariant?} A sweep over $\varepsilon \in \{10^{-9}, 10^{-6}, 10^{-4}, 10^{-3}, 10^{-2}, 3{\times}10^{-2}, 0.1, 0.3\}$ (\texttt{scripts/continuous\_eps\_sweep.py}; Table~\ref{tab:eps-sweep} for the cart, full grid in \texttt{docs/EXPERIMENTS.md}) answers: $\varepsilon$-invariant. Mode-arm reveal-rarity is flat across the entire grid --- on the cart, \texttt{wall@8} is bit-identically flat through the whole grid including $\varepsilon = 0.3$, and \texttt{wall@4} only dips slightly (never widens) at the top of the grid. The pervasive bias arms switch sharply at their own error scale on both instruments instead.

\begin{table}[ht]
\centering
\small
\begin{tabular}{rrrr}
\toprule
$\varepsilon$ & wall@8 rarity & bias $\times$1.03 rarity & bias $\times$2.0 rarity \\
\midrule
$10^{-6}$ & 0.0125 & 1.0000 & 1.0000 \\
$10^{-2}$ & 0.0125 & $<0.0019$ & 1.0000 \\
0.1 & 0.0125 & $<0.0019$ & 0.0040 \\
0.3 & 0.0125 & $<0.0019$ & 0.0040 \\
\bottomrule
\end{tabular}
\caption{Reveal-rarity vs. $\varepsilon$ (cart). The mode arm is flat across the whole grid; the pervasive arms switch at their own error scale.}
\label{tab:eps-sweep}
\end{table}

The pendulum replicates both halves: its mode arms are flat too (only a slight dip at the top of the grid --- \texttt{stop@1.0} rarity 0.1410 at $\varepsilon \le 3{\times}10^{-2}$, dipping to 0.1400 at $\varepsilon = 0.1$ and 0.1240 at $\varepsilon = 0.3$), and its bias arms switch at the same error-scale boundaries as the cart's. pass@40 $\approx (1-r)^{40}$ continues to hold for the mode arms at every $\varepsilon$ in the grid, on both instruments. (The sweep's own rarity column is measured on 2000 rollouts per cell, since what it has to establish is \emph{flatness in $\varepsilon$} rather than a third digit; the 20{,}000-rollout estimate of the same quantity in Table~\ref{tab:axes} is $0.0103$ against this sweep's $0.0125$, one estimate's noise apart.) play\_cost is not re-measured across $\varepsilon$: the model under test does not depend on the gate's tolerance, so play behavior is $\varepsilon$-independent by construction --- the sweep varies only what the gate can see. The gate's $\varepsilon$ is a pervasive-error dial, not a mode-detection dial: tightening it cannot catch the hard mode, and loosening it does not widen the hole.

\textbf{The 4D bi-modal instrument replicates the flat mode arm, once per mode.} On PatchField2D the $\varepsilon$-sweep (\texttt{scripts/continuous\_eps\_sweep\_patch2d.py}) gives \emph{exactly} flat mode-arm reveal-rarity across the entire grid: patches-omitted $0.147$, patch-1-only $0.142$, patch-2-only $0.005$ --- each constant at every $\varepsilon$. The per-mode arms recover the per-mode rarities of the mechanism sweep (Section~\ref{sec:patch2d}), so the tolerance axis is orthogonal to the mode hole on the bi-modal instrument too, and separately for each mode.

\section{Bounded observation noise: the gate law survives to a measured masking boundary}
\label{sup:noisygate}

The gate above accepts on exact equality, which is well-defined only because the instruments are deterministic and observed without noise. This section measures what replaces the exact statements under known bounded observation noise, with the design, critical event, unit, thresholds and seed mapping frozen before the run (\texttt{results/\allowbreak h7\_\allowbreak noisy\_\allowbreak gate\_\allowbreak prespec\_\allowbreak v1.json}; the result JSON records that pre-specification's SHA-256).

\textbf{Design.} Observed next states and rewards carry independent coordinatewise $\mathrm{Uniform}[-\eta,\eta]$ noise at the seven pre-specified levels $\eta \in \{0, 0.01, 0.03, 0.1, 0.3, 1, 3\}$; exact matching is replaced by support compatibility (the noisy observation must lie in the candidate's bounded support). The candidates are the CartWall@4 truth and its analytic mode-blind proxy --- this is a gate-law measurement, not a new synthesis campaign. Under the uniform-noise hypothesis the conditional per-scalar overlap probability is exact, $q(\delta,\eta) = \max(0,\, 1 - |\delta|/(2\eta))$ for a truth-vs-candidate discrepancy $\delta$ (with the $\eta = 0$ equality limit), and multiplies over coordinates and transitions; independent noise then gives one empirical Bernoulli verdict per block. The unit is the disjoint $20$-rollout seed block, $200$ blocks on a fixed latent panel ($189$ of $200$ contain the mode; $1816$ contact transitions); neither the $4000$ rollouts nor the $320{,}000$ transitions are treated as independent (\texttt{scripts/\allowbreak h7\_\allowbreak noisy\_\allowbreak observation\_\allowbreak gate.py}, \texttt{results/\allowbreak h7\_\allowbreak noisy\_\allowbreak observation\_\allowbreak gate\_\allowbreak v1.json}).

\begin{table}[ht]
\centering
\small
\begin{tabular}{rrrrr}
\toprule
$\eta$ & truth passes & blind passes & exact CP$95$ & analytic conditional \\
\midrule
0--0.1 & 200/200 & 11/200 & $[0.0278, 0.0963]$ & 0.0550 \\
0.3 & 200/200 & 11/200 & $[0.0278, 0.0963]$ & 0.0556 \\
1 & 200/200 & 32/200 & $[0.1121, 0.2183]$ & 0.1582 \\
3 & 200/200 & 106/200 & $[0.4583, 0.6008]$ & 0.5029 \\
\bottomrule
\end{tabular}
\caption{Support-compatibility gate under bounded observation noise, per seed block. The four levels $\eta \in \{0, 0.01, 0.03, 0.1\}$ are identical to four digits and are shown as one row.}
\label{tab:noisygate}
\end{table}

\textbf{Readings.} The correct model passes every block at every level. The pre-specified primary criterion --- an analytic blind-pass increase of at most $0.01$ at $\eta = 0.1$ relative to $\eta = 0$ --- is met with an increase of exactly $0$ on this panel, so bounded noise through the primary level masks no additional block: the noiseless equality statement becomes a support-probability statement with the same content. The first pre-specified level whose increase exceeds the frozen $+0.05$ boundary is $\eta = 1$, where the analytic conditional blind-pass probability is $0.1582$ and the independent noise realization accepts $32/200$ blocks --- sufficiently wide support measurably hides mode discrepancies, and the boundary at which it starts is now a number rather than a caveat. These are conditional evaluations on the fixed latent panel, not population-prevalence estimates; process noise, an unknown noise law, a planner rerun under noise, and an externally specified benchmark remain open (Section~\ref{sec:limitations}).

\section{Cross-family arms, and a fourth artifact class}
\label{sup:crossfamily}

\paragraph{Cross-family spot-checks (two families).} \emph{Qwen} (HF router, \texttt{Qwen/\allowbreak Qwen3-\allowbreak Coder-\allowbreak 30B-\allowbreak A3B-\allowbreak Instruct}, 3 seeds): the full arm is clean (3/3 gate 1.000, blind 0.0), so the pinned-integrator premise is not GPT-specific; on the incomplete arm the identifiability branch reproduces (the 1/3 wall-absent seed is gate-1.000 wall-blind and exploited, play\_cost 0.999), but Qwen \emph{repaired neither} of its two wall-present seeds (gate 0.999 and 0.491, both superstitious patches the gate refused). \emph{Claude} (Sonnet, agent-relayed via the companion paper's protocol on \texttt{scripts/continuous\_claude\_step.py} --- verbatim pipeline messages relayed to fresh, context-free instances per message, same $\varepsilon = 10^{-9}$ gate and MPC play as the API arms; seeds 10000/20000/30000 plus one full control per instrument; \texttt{results/continuous\_claude\_relay.json}): both full controls are clean (2/2 gate 1.000, blind 0.0, play\_cost 0.0), and both mode-absent seeds are accepted fully blind and exploited (cart play\_cost 0.999, pendulum 0.995) --- the identifiability event fires family-independently, as Proposition~\ref{prop:ident} requires. On repair Claude is neither GPT-5.x nor Qwen: cart seed 20000 repaired the exact one-sided rule in 1 iteration, and cart seed 30000 repaired at iteration 5 after a period-2 oscillation (symmetric $\pm 8$ walls $\to$ both removed $\to$ symmetric again $\to$ removed $\to$ one-sided correct), but pendulum seed 20000 was accepted in 1 iteration carrying an invented mode (next paragraph) and pendulum seed 30000 stalled at 5 iterations (gate 0.9972), oscillating between the symmetric-stop and no-stop artifacts without ever finding the one-sided rule. So across three families the mode-absent blind-and-exploited event fired for every one (it is a property of the sample; Proposition~\ref{prop:ident}), while repair-from-data is model-dependent \emph{in mechanism, not merely in rate}: GPT-5.x recovers the revealed clamp in almost every case (105 of 111 draws exactly; four more add a phantom stop the probe cannot see, two stall), Qwen recovers none (superstitious local patches), and Claude recovers most but through a \emph{symmetry prior} that generalizes one-sided evidence into a symmetric pair of boundaries --- a prior GPT-5.x also exhibits, in four artifacts across both sizes, so the class is not family-specific.

\paragraph{A fourth artifact class: accepted while carrying an invented mode its gate sample cannot refute.} A symmetry-or-landmark prior produces an outcome beyond the correct/blind/superstitious-patch trichotomy. It is clearest in the Claude arm, and it is not confined to it: four GPT-5.x artifacts on the pendulum's headline knob do the same thing (Section~\ref{sec:patch2d-synthesis}). Where the training sample covers the invented side, the gate refutes the phantom and the memoryless refine loop oscillates (cart seed 30000's period-2 cycle above; pendulum seed 30000's stall at gate 0.9972). But where the sample is \emph{silent} on the invented side, the phantom is unfalsifiable: pendulum seed 20000 was certified at gate 1.000 in a single iteration while carrying a phantom symmetric stop at $\theta = -1.4$ that this seed's rollouts never reach (verbatim \texttt{elif th2 < -th\_max: th2 = -th\_max; om2 = 0.0}, the second branch of a symmetric pair, with \texttt{th\_max = 1.4} its own constant) --- a verified artifact, exact on every sampled transition, that nonetheless encodes a hard stop which does not exist. This is a clean natural experiment: the \emph{same} symmetric artifact was refuted at pendulum seed 30000 and certified at seed 20000, the only difference being whether the sample covered $\theta < -1.4$. It is Proposition~\ref{prop:ident}'s prior caveat measured directly --- on inputs the sample never touches, the artifact's content comes from the model's prior, and the gate cannot police it. Note the classification did not flag it: the \texttt{mode\_blindness} probe scored 0.0 (correct) because it probes only the true $+1.4$ mode's region; code inspection, not the probe, caught the invented mode (Section~\ref{sec:limitations}).

\paragraph{Qualifications on the Claude arm.} (i) Agent-relayed means an agent scaffold over a subscription transport, not an API, though the relayed messages are byte-identical to the pipeline's; in the two multi-iteration cells a handful of refinement replies prefixed a one-line explanation before the code block (two distinct sentences, repeated across the oscillation) despite the output-only-code instruction; the code block still parsed, and no relay was refused. (ii) The refine loop is memoryless in the API arms too --- \texttt{refine\_continuous} sends a single user message per iteration (\texttt{src/cwm/continuous/contract.py}) --- so the oscillation is protocol behavior, not a relay artifact. (iii) $n$ is small: 3 seeds plus one control per instrument, one alternate family.

The branches compose into the paper's central claim. With the mode in the data, the synthesize--refine--gate loop recovers the exact mode or refuses the artifact --- correct relative to every transition the sample covered, which is the only sense of correctness the gate defines (sample-covered exactness is the acceptance criterion itself; the empirical, code-inspected regularity --- not a theorem --- is that every accepted artifact's off-sample content was also the true mode, except Claude's phantom pendulum stop, which errs only where its sample is silent: the Proposition~\ref{prop:ident} prior caveat, not a gate failure). With the mode absent, no loop can help (Proposition~\ref{prop:ident}: the sample carries no evidence for the mode --- though a prior or the specification still could supply it), and the acceptance of a blind artifact is not a failure of the LLM, the refinement, or the gate implementation --- it is the sampling event whose probability is exactly $(1-r)^N$. The discrete danger law had a provable core plus an empirical residual (the rule not learned even when shown); on the continuous 1D instruments the residual \emph{can} vanish --- it does for GPT-5.x on all but four of the clamps it was shown --- so \textbf{the law becomes the entire failure surface} for a capable-enough synthesizer, and shrinks toward it even for a weaker one (this is geometry-scoped: Section~\ref{sec:patch2d-synthesis} shows the residual returns on a 2D circular mode, where even GPT-5.x does not repair). The actionable consequence sharpens correspondingly: in this regime, spec completeness and gate-sample coverage are not two independent worries --- coverage \emph{is} the dominant worry, because the synthesis loop repairs what the sample reveals.

(Scope: 20 seeds/cell on the headline cell, both GPT-5.x sizes; the wall-absent conditional is 20/20 across sizes and consistent across three independent runs --- the 5-seed first run, the 20-seed tightened run, and the 3-seed Qwen run; the wall-present repair rate is 20/20 for GPT-5.x on this cell, 0/2 for Qwen, and mechanism-dependent for Claude (exact on both cart seeds, one via a period-2 oscillation) --- model-dependent, small-$n$ on the cross-family arms. LLM synthesis is stochastic across calls; the three-way structure and the identifiability conditional are stable run-to-run, per-seed iteration counts are not.)

\section{The 2D artifacts: a code and behavioural audit, and three ablations}
\label{sup:artifacts}

\paragraph{The mechanism: translation succeeds, induction of the boundary collapses.} A 76-artifact code inspection locates the failure precisely. \emph{Plant translation succeeds}: ${\approx}74/76$ artifacts reproduce the exact 4D semi-implicit integrator and the reward. What collapses is \emph{induction of the 2D circular boundary}. The dominant modal failure ($38/76$) is \textbf{dimensional reduction}: the disc is modeled as a \emph{half-plane} --- a 1D CartWall-style clamp at the right location but the wrong shape, e.g.\ substituting \texttt{if x2 > 4.0: ...} for a circle --- learning the freeze mechanic but not its circular trigger. The remaining classes: $20/76$ pure-blind (no patch logic), $9/76$ superstitious local patches, and $9/76$ disc-form attempts, none of them correct. The $\varepsilon$-exactness alternative --- that correct-form discs merely failed the $10^{-9}$ arithmetic gate --- is \textbf{falsified}: no correct-form disc failed on arithmetic; the discs that failed were the wrong \emph{shape}. A \emph{behavioral} audit confirms this hand inspection on every claim (\texttt{scripts/patch2d\_artifact\_audit.py}: probe each artifact's \texttt{step()} on a state grid, classify the shape of its deviation-from-integrator set): 39/76 half-plane-form, 13 behaviorally blind, 8 disc-form, 7 measure-zero textual traps, 4 point, 3 square-form, 2 bounded-other; integrator arithmetic exact in 74/76; and --- the point that matters for partial repair --- \textbf{no artifact \emph{encodes} its seen patch}. On the see-one-miss-the-other branch (the 66 seeds where a partial repair is even definable), $28$ freeze sets do \emph{contain} the seen patch (coverage $>0.9$), but they contain it the way a half-plane contains a disc: their frozen area runs from $15\times$ to $81\times$ the patch's own ${\approx}1.2\%$ share of the probed box (median $61\times$). Exactly one artifact is even patch-\emph{selective} --- seed 180000 covers its seen patch and leaves the unseen one below $0.1$ --- and it does so with a \emph{half-plane} freezing $31\times$ the patch area, i.e.\ by where its threshold happens to fall between the two patches, not by encoding a region; its gate score is below $1$ like every other see-one artifact. The missing partial repairs were therefore never attempted as \emph{regions} at all, rather than attempted and refused by the gate. The two class counts differ by exactly one artifact (39 vs 38 half-plane, 8 vs 9 disc-form) and the difference is a \emph{syntax-versus-behavior} reading, not a disagreement about any artifact's correctness: two artifacts (mini $k=(3,7)$, seeds 130000 and 150000) write a \emph{radial} predicate anchored at the reward lodes rather than at a patch --- \texttt{if hypot(x{-}(-6),y) <= 2 or hypot(x{-}12,y) <= 2} in one, a reward-threshold freeze in the other --- so hand inspection reads the disc in the source while the audit reads the deviation set, which is unbounded inside the probed box. Both readings agree that neither artifact encodes a patch; only the label moves.

\paragraph{Ablation 1: richer prompting and $3\times$ budget do not restore repair --- they change the failure class.} The natural confound is that the prompt was too poor and the budget too small. The strongest joint treatment (120 examples, 40 failure lines, describe-the-region-first guidance with an explicit ``the region need not be a 1D threshold'' de-bias, and 15 refine iterations; identical samples to the original run) yields \textbf{0/40 repair} at $k=(3,7)$, both sizes. The guidance \emph{works as prompt engineering}: the half-plane reduction disappears (0/40, versus 21/40 in the matching base cells) and the artifacts now write bounded 2D regions --- rotated ellipses (${\sim}15/40$), rectangles (${\sim}10/40$), unions of micro-discs (${\sim}5/40$) --- but none is the true disc. The new failure class is \emph{evidence-hull fitting} with two compounding errors: they fit the hull of the \emph{observed} freeze positions (the pre-freeze crescent hugging the disc's reachable west boundary from outside: fitted centers pulled west to $x \approx 2.3$ against the true $x = 3$, i.e.\ dragged onto the evidence at the patch's west edge), and 36/40 condition on the current position rather than the landing $(x_2, y_2)$ --- the causal variable of the rule. The sample itself then refuses every wrong shape.

\paragraph{Ablation 2: an axis-aligned square with flat edges --- curvature did not suffice to explain the failure.} This ablation was added after the disc result it responds to (Section~\ref{sup:prespec} dates every such addition). If the failure were about \emph{curved} boundaries, an axis-aligned square patch (a Chebyshev ball, $\max(|x-c_x|,|y-c_y|) \le R$ --- a \texttt{max}/\texttt{abs} predicate, no quadratic) should be repairable. It is not: with the same pipeline at $k=(3,7)$ the full arm is clean (20/20 at zero refine iterations, both sizes --- translating the square clause is as easy as the disc's) and the incomplete arm gives \textbf{0/40 repair} (every sample mode-containing; best gate $0.9962$). The classes mirror the disc's, \emph{reflected}: the dominant failure is again dimensional reduction (the square's west edge as a 1D threshold, \texttt{if x2 >= 2.0}), and several artifacts commit the \emph{inverse} error --- writing \emph{discs} on square evidence (radial \texttt{hypot} guards at the edge midpoint, micro-disc unions), plus reward-anchored superstitions (freeze zones invented at the lodes). Zero artifacts write the true box or its \texttt{max}/\texttt{abs} form. Across the four treatments (two disc knobs, the square ablation, the guided/$3\times$-budget treatment) that is \textbf{0 of 156 mode-containing synthesis draws}, and the sampling unit makes those $156$ draws \textbf{20 distinct gate-sample blocks}: changing the disc knob or the prompt variant reuses a block's random stream byte for byte, so $k=(3,7)$ and $k=(5,9)$ are two treatments over the same blocks rather than $38$ samples. The bound the evidence supports is therefore per block: an exact (Clopper--Pearson) $95\%$ upper bound of $\mathbf{0.168}$ on the probability that a fresh mode-containing gate sample is repaired by any attempt run on it (Wilson $0.161$). Pooling the $156$ draws as independent trials would have claimed $0.023$, $7.2\times$ tighter than the evidence supports; we report the per-block figure and the per-treatment counts separately (\texttt{scripts/\allowbreak paper2\_\allowbreak statistics.py}), never a single pooled interval.

\paragraph{Ablation 3: two further model-family spot-checks, at tiny $n$, fail in the same classes.} The three campaigns above are GPT-5.x only, so the collapse could be one family's idiosyncrasy. A Claude arm on the headline cell says otherwise (agent-relayed as in Section~\ref{sec:synthesis}; Sonnet, $k = (3,7)$, three mode-containing seeds --- two see-$P_1$-miss-$P_2$, one seeing both --- plus one full control, same $N = 40$ sample, same $\varepsilon = 10^{-9}$ gate, same five-iteration memoryless refine loop). The full control is clean at iteration 0 (gate 1.000, both discs written on the \emph{landing} position, blindness 0.0, play\_cost 0.0): translation is not the problem here either. No incomplete seed repaired (0/3), and the \emph{trajectory} carries more information than the outcome: each seed spends its five iterations in a \textbf{period-2 cycle} between the pure-blind artifact and a wrong template, returning to its iteration-0 gate value exactly on every even iteration (seed 10000: $0.9934$ blind $\to$ $0.9591$ disc-on-\emph{current}-position $\to$ $0.9934$ $\to$ $0.9284$ half-plane $\to$ $0.9934$ $\to$ $0.7044$ reward-threshold; seed 20000: blind $0.9962 \leftrightarrow$ half-plane $0.9406$, three times over; seed 30000: $0.9966 \to 0.6750$ reward-zone $\to 0.9966 \to 0.8406$ half-plane $\to 0.9966 \to 0.5328$ $y$-band). The template set is the one the GPT-5.x campaigns exhibited --- 1D threshold, radial disc on the \emph{wrong} (current, not landing) variable, reward-landmark zone, axis-aligned band --- and \textbf{no incomplete iteration in any seed wrote a disc on the landing position}, the form this same model writes immediately when the contract states it. The \emph{constants} replicate Ablation 1 quantitatively: every half-plane sits at $x = 2.0$, the disc's west edge --- ``the right location, the wrong shape'' --- and the one radial attempt fits center $(2.328, -0.135)$ with radius $0.824$, i.e.\ the hull of the observed freeze crescent: the true patch is the disc of radius $1$ about $(3,0)$, spanning $x \in [2,4]$, while the fitted one spans $x \in [1.50, 3.15]$ --- \emph{dragged west onto the evidence}, so it covers the patch's west half, misses $x \in [3.15, 4]$ entirely, and freezes free space over $x \in [1.50, 2]$. Its center sits inside the true patch but at the wrong place ($0.69$ from the true center), with radius $0.824$ against $1$. That is the same $x \approx 2.3$ signature the guided GPT-5.x artifacts produced. Per-iteration gate accuracies and rule classes are re-derived from the versioned transcripts by \texttt{scripts/claude\_relay\_ledger.py}. Two qualifications. (i) $n = 3$ seeds, one alternate family on the \emph{incomplete} arm: this closes the ``GPT-only'' confound, it does not sweep models. A third family (Qwen) contributes a \emph{matched} contrast on this instrument: both arms, three seeds each, run against a single pinned backend --- our own vLLM $0.26.0$ serving the bf16 checkpoint at a recorded Hub revision, with the checkpoint's own \texttt{generation\_config} applied identically to both arms (\texttt{results/\allowbreak qwen\_\allowbreak vllm\_\allowbreak provenance.json}; server in \texttt{scripts/\allowbreak modal\_\allowbreak qwen3coder\_\allowbreak vllm.py}). Under that one configuration the translation control is clean ($3/3$ at gate $1.000$, zero refinement iterations) and the incomplete arm is \textbf{refused in $3$ of $3$ seeds} (accuracies $0.993$ to $0.997$, refinement maxed, the mode present in every training block, none accepted by the independent gate) --- the family's 1D stall-and-refuse class. An earlier first pass had mixed serving paths (HF-router full cells, vLLM incomplete cells) and is retained versioned as an exploratory unmatched spot-check; its incomplete cells reproduce the matched arm's gate accuracies to three decimals, but the contrast the text draws is the matched campaign's. At $n = 3$ per arm this is an exploratory diagnostic of a third family, not a cross-family replication. That control is clean: 3/3 gate 1.000 at zero refinement iterations, both discs written exactly on the landing position, per-patch blindness 0.0. Across all three families, then, \emph{stating} the region rule is easy and \emph{inducing} it from data is what fails; the induction evidence on this instrument is GPT-5.x's 156 seeds plus Claude's 3. (ii) The relay caveats of Section~\ref{sec:synthesis} apply, plus one specific to this arm: the relayed instances were instructed not to use tools or read files, so each artifact is a function of the pipeline message alone (the ground truth lives in this repository, so an instance reading it could have copied the true centers and radius instead of inducing them). The wrapper text is recorded verbatim in \texttt{results/\allowbreak claude\_relay\_transcripts/\allowbreak patch2d\_k3\_7\_RELAY\_FRAMING.txt}, together with the post-hoc leak check: no incomplete artifact wrote the true disc form at all, and none names the unseen far patch --- the signature repository access would have left.

\section{The eight ablations on the 2D mode, campaign by campaign}
\label{sup:ablations}

\paragraph{Ablation 4: trigger arity, and why it does not answer the question it was built for.} The slab campaign holds the plant, the action parametrization, the lodes, the evidence side, the prompt and the budget fixed and changes only the predicate's \emph{arity}, from a region in two landing coordinates to a band in one, with rarity matched by moving the mode as every other instrument in this paper matches it (the calibrated impermeable configuration: half-width $0.5$, centre $5.5$, giving $r_1 = 0.128$ against the disc's $0.132$ and $\mathrm{play\_cost} = 1.006$ against $1.006$). Translation is trivial, as it must be for the ablation to be admissible: the full arm is $\mathbf{40/40}$ at gate $1.000$ in \emph{zero} refinement iterations, both sizes, with every mode encoded. The incomplete arm then produces the paper's most instructive artifact class.

\textbf{By the criterion this paper has used for repair, \texttt{gpt-5.4} repaired 19 of its 20 mode-containing seeds. It recovered the rule in none of them.} All $19$ are behaviourally the \emph{same} artifact: the half-plane $\texttt{x2 >= 5.0}$ at the slab's near face. It passes the gate at $\varepsilon = 10^{-9}$ on its own sample, passes an \emph{independent} acceptance sample and an independent $100$-rollout evaluation sample at the same tolerance, and scores $0$ on the mode-blindness probe --- while freezing $4.6\times$ the true region's area and agreeing with the truth on only $22\%$ of the union of the two freeze sets. It is not a failure of induction: by Proposition~\ref{prop:entryclass} it is exactly consistent with every sample this instrument can produce, so it is the correct inference from the available evidence, and by part (iii) of that proposition its play cost is $0$, which is what is measured. The $\texttt{gpt-5.4-mini}$ arm reaches the class in none of its $20$ seeds; it writes refutable rules that the gate then refuses. So the size effect here is not that the larger model repairs and the smaller does not --- it is that the larger model finds the maximally-consistent rule while the smaller one does not.

\paragraph{Ablation 5: naming the variable the trigger reads.} The base campaign's audit found $36$ of $40$ guided artifacts conditioning the freeze on the position the step \emph{began} at rather than on the landing position --- the rule's actual argument --- so that confound is worth removing on its own. The \emph{landing} arm is the region-guided treatment with one sentence added, stating that a localized rule's trigger is evaluated on the state the integrator would produce, and naming nothing about the region's shape, centre, radius or count; everything else is matched to the region arm exactly ($120$ examples, $40$ failure lines, $15$ iterations, the same $20$ sample blocks). It does not restore repair: \textbf{0 of 40} mode-containing draws recover the disc, over $20$ distinct blocks, for an exact $95\%$ upper bound of $0.139$ on the per-block repair probability, with a best agreement of $0.147$ against the region arm's $0.286$. What the sentence changes is the failure: the dominant class becomes \emph{memorisation} of the observed freeze positions --- $20$ of $40$ artifacts write a region small enough to be a point on the probe grid, against $15$ of $40$ in the region arm, one of them the literal bounding box \texttt{2.01 <= x2 <= 2.17 and -0.23 <= y2 <= 0.03} of the crescent of contacts its sample happened to contain. So variable identification is not what was blocking region induction, and telling the synthesizer where to look pushes it further toward fitting the evidence exactly rather than generalising it.

\paragraph{Ablations 6 and 7: making the region's interior witnessable does \emph{not} restore repair.} Proposition~\ref{prop:entryclass} identifies the evidence's censoring as an obstruction, and the natural hypothesis is that it is \emph{the} obstruction. Two campaigns test it by changing where the mover ends up when the mode fires --- the firing predicate itself, and therefore the rarity, is untouched, so the knob needs no recalibration; each targets the post-state alone, though dwell time and the amount of mode evidence necessarily move with it (the caption's confound note). In \emph{landing} the mover stops where it entered, strictly inside the region; in \emph{clamp} it is projected onto the boundary, the 2D analogue of the cart's wall clamp. Both break the proposition's premise, verified on the gate sample rather than assumed: \texttt{freeze} yields no transition at all that separates the membership rule from an entry rule, \texttt{landing} yields $4614$ and \texttt{clamp} $233$ (\texttt{scripts/\allowbreak calibrate\_\allowbreak mode\_\allowbreak effect.py}). The trap survives in both ($\mathrm{play\_cost}$ $1.058$ and $1.059$ against \texttt{freeze}'s $1.059$, blind contact rate $1.00$), and translation stays trivial: the full arm is $\mathbf{20/20}$ at gate $1.000$ in zero refinement iterations in all four cells, so even the clamp's projection rule --- the only one whose post-state is a \emph{function} of the landing --- is as easy to write as the others when stated.

\textbf{Neither restores repair.} Across $40$ mode-containing draws each, over the same $20$ blocks: $0$ recover the region, and --- unlike every freeze campaign --- \textbf{no artifact passes the gate at all}, which is what Proposition~\ref{prop:entryclass} predicts once the equivalence class is empty. The best agreement with the truth is $0.10$ for \emph{landing} and $0.26$ for \emph{clamp}, against the disc's own $0.50$. What the artifacts write is the \emph{same library} at the \emph{same constants}: the dominant class is again the half-plane ($24/40$ and $25/40$) at $\texttt{x2 > 2.0}$, the disc's west edge, with reward-landmark zones, micro-discs and velocity superstitions making up the rest.

The two campaigns' confounds run in opposite directions, which is why we ran both. \emph{landing} leaves the mover inside the region, where most thrusts keep it, so its dwell time and hence the mode's share of the sample's transitions rise from $0.66\%$ to $7.25\%$ --- eleven times more mode evidence, a confound \emph{favourable} to repair, and it still fails. \emph{clamp} holds that share at $0.75\%$, matching \texttt{freeze}, and pays instead in the rule's complexity --- and it fails with the same classes. So the negative survives both the quantity of mode evidence and the difficulty of the rule.

\textbf{Witnessing the interior does not restore repair, so the censoring is not the operative obstruction.} The censoring Proposition~\ref{prop:entryclass} describes is real, and it is what makes the slab's nineteen accepted-but-wrong artifacts unfalsifiable and what bounds the reach of an independent gate; but it is \emph{not} the reason region induction fails. With the interior witnessed, by two mechanisms, with more evidence and with an easier rule, the synthesizer writes the same half-plane at the same place. Of the candidate causes these interventions can separate, what is left is the region-template prior, and what the interventions add to the description is stability: the prior does not yield to evidence dose, interior witnessing, budget, or the rule's difficulty. It remains unlocalized inside the model --- the alternatives named in Section~\ref{sec:arity} (an inability to execute an algebraic fit over textual transitions, a memoryless refinement loop, the absence of a system-identification objective) are compatible with every campaign here, and no design run distinguishes among them.

\paragraph{The exclusion matrix, versioned.} \texttt{results/\allowbreak h6\_\allowbreak exclusion\_\allowbreak matrix\_\allowbreak v1.json} (\texttt{scripts/\allowbreak h6\_\allowbreak exclusion\_\allowbreak matrix.py}) records, for each of the eight interventions: the target hypothesis, every changed variable with its unavoidable co-changes, whether a direction was recorded before the run, the experimental unit, the positive control, the observed result, and the strongest licensed inference beside the inference the design does \emph{not} license. The surviving classes after all eight are the three named in Section~\ref{sec:arity} plus a residual of model-internal causes these designs do not separate. Four factorial follow-ups that would separate the remaining confounded pairs are specified in the same file: prompt guidance versus example/failure-line budget versus iteration budget (intervention 1's three-way co-change); post-state semantics crossed with a calibrated equal contact dose (interventions 6/7); angular coverage crossed with start distribution (intervention 8); and trigger arity crossed with identifiability on a bounded one-coordinate region (intervention 4).

\section{The coverage certificate's scope, stated in full}
\label{sup:certscope}

\paragraph{Coverage certificates transfer, but only to the smooth case.} The companion paper's guarantees enumerate information sets, and that enumeration has no continuous analogue. What does transfer is the covering-number version: Proposition~\ref{prop:coverage} certifies $\sup_U \|f - \hat f\| \leq \varepsilon + 2L\rho$ for $L$-Lipschitz pairs once the gate $\rho$-covers $U$, with the sample size a covering number over a visitation density, and on the deployed cart gate it excludes a pair with $L = \max(\mathrm{Lip} f, \mathrm{Lip}\hat f) \leq 5.77$ carrying the wall's error magnitude of $4.2$, with probability at least $1-\delta$ over gate draws, against a rigorous bound of $0.933$ on the step-1 reachable set. Since the bound $\varepsilon + 2L\rho$ grows with $L$, that quantifier is the claim's scope, and at $4.5\times$ the plant's own $1.27$ it is a broad class. The step from the packing figure to that bound is a change of argument, not a new measurement: a packing bound must estimate three geometric factors (the covering-vs-packing direction, the ball's intersection with $U$, and the corner factor's failure to be shear-invariant), and a partition bound estimates none of them --- on this instrument the step-1 law is uniform on a sheared box, so equal sub-boxes have probability exactly $1/K$. The packing route's figure is $\rho = 1.165$ and $2.97$, licensing only $L \leq 1.80$. The residue is that this is exactly the case the paper is \emph{not} about: the reset mode has unbounded local Lipschitz constant, so no finite $L$ makes the certificate apply to it, and the danger the paper measures lives entirely in that gap. Two questions the certificate depends on are settled, and settling them made the certificate \emph{weaker} and the conclusion sharper. First, the within-rollout dependence: it needs no mixing argument and no factorization either --- the per-rollout miss probability is measured directly and the i.i.d.-ness of rollouts raises it to the $N$th power exactly. Composed with the partition, that is what gives the $0.933$ above, against $1.534$ from independent single samples: handling the dependence exactly is worth a factor of $1.6$, and the invalid all-steps-independent reading's $0.785$ sits a factor of $1.2$ below the valid figure. On a box $6.7\times$ larger the same treatment certifies net radius $1.0$, which is the region/resolution trade-off at fixed $N$ rather than a competing figure. Second, the density is derived at \emph{every} step $t \geq 2$ (the last two actions make the step-$t$ law absolutely continuous with an exact constant $1/0.036$, and a Minkowski erosion turns that into a pointwise infimum). That extends the certificate to larger regions, but at fixed $N$ extent is paid for with resolution, and no step-$t$ level set certifies a bound below the step-1 figure --- so what limits the certificate is not the step-1 restriction but the gate's size. The nonlinear analogue needs no new argument for the two-action factor: the Jacobian determinant is $\mathrm{gain}^2 dt^3$ identically for any $C^1$ force, so the same constant covers the cart and the pendulum --- the certificate's density hypothesis is established on the cart only and must be verified per instrument --- and it degenerates only on the mode clamp, where no density argument can work because the truth maps a neighborhood to a point. And the planner-versus-gate mismatch is measured rather than argued: the box the dependence-exact certificate covers carries $1.9\%$ of the exploited planner's queries, and a box containing every step-$t$ level set carries $7.8\%$ --- an upper bound on the certified share. What is left is therefore not a hole in the argument but its conclusion --- a certificate sound relative to the $L \leq 5.77$ Lipschitz class, with probability at least $1-\delta$ over the gate's draws, on a region the planner barely visits. Certifying the planner's own visitation would need a gate that samples it, which is precisely the change the companion paper's closing prescription asks for.

\section{The LLM protocol, in full}
\label{sup:protocol}
Everything in this appendix is read off the code that produced the campaigns, not
described from memory: the pipeline is
\texttt{scripts/\allowbreak continuous\_\allowbreak danger\_\allowbreak synthesis.\allowbreak py} driving
\texttt{src/\allowbreak cwm/\allowbreak continuous/\allowbreak contract.\allowbreak py}, and the per-campaign counts are
regenerated by \texttt{scripts/\allowbreak llm\_\allowbreak protocol\_\allowbreak facts.\allowbreak py} into
\texttt{results/\allowbreak llm\_\allowbreak protocol\_\allowbreak facts.\allowbreak json}. File and line references are as of
the revision that produced \texttt{results/}. The agent-relayed Claude arms send
the same message text through a different transport
(Section~\ref{app:proto-relay}).

\subsection{The two messages}
\label{app:proto-messages}

A cell of a campaign is one \emph{(arm, seed)} pair. It issues one
\emph{synthesis} call and then zero or more \emph{refine} calls. In the three
displays below, a trailing backslash marks a line wrapped for the page width:
the string actually sent contains neither the backslash nor a newline there.
Slots in braces are the template's substitution points.

\paragraph{System message (synthesis call only).} One string, sent verbatim
(\texttt{contract.\allowbreak py:\allowbreak 87--89}):

{\small
\begin{verbatim}
You are an expert Python programmer. You write deterministic, pure code \
that exactly implements a specified physics world model. Output ONLY a \
single Python code block, no prose.
\end{verbatim}
}

\paragraph{Synthesis user message.} The template (\texttt{contract.\allowbreak py:\allowbreak 90--98});
the messages list is exactly \texttt{[system, user]} (\texttt{contract.\allowbreak py:\allowbreak 97--98}):

{\small
\begin{verbatim}
{contract}

Here are observed transitions (ground truth) to match exactly:
{example_lines}

{guidance}

Write the Python module implementing the contract. Output only one \
```python code block.
\end{verbatim}
}

\noindent \texttt{\{contract\}} is \texttt{build\_\allowbreak contract(env,\allowbreak  include\_\allowbreak mode)\allowbreak }:
the instrument's API text (the pinned integrator equations, in order) plus its
rules text (physical constants, reward, and --- in the \emph{full} arm only ---
the mode clause); \texttt{contract.\allowbreak py:\allowbreak 32--34}, text in
\texttt{src/\allowbreak cwm/\allowbreak continuous/\allowbreak instruments.\allowbreak py}. \texttt{\{guidance\}} is the
prompt-variant knob of Section~\ref{app:proto-variants}; it is the empty string
in every default run, and because the template inserts it as
\texttt{extra =\allowbreak  f"\{guidance\}\textbackslash n\textbackslash n" if guidance else ""}
(\texttt{contract.py:90}), an empty guidance leaves the message byte-identical to
the pre-knob prompt --- asserted in
\texttt{tests/\allowbreak test\_\allowbreak continuous\_\allowbreak contract.\allowbreak py} and
\texttt{tests/\allowbreak test\_\allowbreak prompt\_\allowbreak variants.\allowbreak py}.

\paragraph{Example lines: how many, and which.} One line per shown transition,
using \texttt{repr()} throughout, so full float precision is shown
(\texttt{contract.\allowbreak py:\allowbreak 67--76}):

{\small
\begin{verbatim}
step({state!r}, {action!r}) -> {next_state!r}   reward(next) = {reward!r}
\end{verbatim}
}

\noindent Selection is a \emph{stride over the whole sample}, not a prefix (a
prefix would be a single rollout):
\texttt{idx = sorted(\{(i * n)\allowbreak  /\allowbreak /\allowbreak  max\_\allowbreak examples for i in range(min(max\_\allowbreak examples,\allowbreak  n)\allowbreak )\allowbreak \})}
(\texttt{contract.py:71}). At the campaign default of $N = 40$ rollouts and
$h_\mathrm{episode} = 80$ the sample is $3200$ transitions on all three
instruments, so the prompt shows $30$ lines ($0.937\%$ of the sample) in the
default variant and $120$ lines ($3.75\%$) in the guided/region/landing variants
(\texttt{example\_budget} in the JSON). The gate, by contrast, always scores the
candidate on all $3200$ transitions (\texttt{contract.\allowbreak py:\allowbreak 167--177}): the model is
shown about one percent of the evidence it is judged on.

\paragraph{Refine user message.} There is \emph{no} system message on a refine
call (\texttt{contract.\allowbreak py:\allowbreak 200--206}):

{\small
\begin{verbatim}
{contract}

The current implementation is below. It fails some transitions. Fix it \
so every transition matches to within {eps} in x, v and reward. Output \
only one ```python code block.

CURRENT CODE:
```python
{code}
```

FAILURES (expected vs got):
{first max_failures failure lines}

{guidance}
\end{verbatim}
}

\noindent Each failure line comes from \texttt{\_\allowbreak compare\_\allowbreak transitions}
(\texttt{contract.\allowbreak py:\allowbreak 135--164}) and is one of

{\small
\begin{verbatim}
step({state!r}, {action!r}) raised {exc!r}
step(...): wrong state arity: expected {k} components, got {m}
step(...): expected {next_state!r} r={reward!r}, got {got!r} \
r={got_r!r} (err {err:.3g})
\end{verbatim}
}

\paragraph{Truncation of the failure list.} Only the first
\texttt{max\_failures} lines are included,
\texttt{"\textbackslash n".\allowbreak join(failures[:max\_\allowbreak failures]\allowbreak )\allowbreak }
(\texttt{contract.py:205}), with \texttt{max\_\allowbreak failures =\allowbreak  20} by default
(\texttt{contract.py:191}) and $40$ in the guided/region/landing variants.
Failures are in sample order, so truncation keeps the \emph{earliest} failures,
not the worst ones.

\begin{remark}[A wart in the refine text, recorded because it is what ran]
The sentence says ``in x, v and reward'' on every instrument, including the
four-dimensional PatchField2D whose state is $[x, y, v_x, v_y]$
(\texttt{contract.py:202}). The tolerance actually enforced is the sup-norm over
\emph{all} state components and the reward (\texttt{contract.\allowbreak py:\allowbreak 154--155}), so
the text understates the check on the two-dimensional instrument. This is the
text used by every committed run.
\end{remark}

\subsection{The refine loop}
\label{app:proto-refine}

The loop is \texttt{contract.\allowbreak py:\allowbreak 188--214}.

\begin{itemize}
\item \textbf{Termination.} \texttt{while acc < 1.\allowbreak 0 and iterations < max\_\allowbreak iters}
(\texttt{contract.py:199}), where \texttt{acc = contract\_\allowbreak accuracy(code,\allowbreak 
transitions,\allowbreak  eps)\allowbreak } over the \emph{whole} sample. It stops on the first artifact
that matches every transition to $\varepsilon$, or when the budget runs out.
Nothing else stops it.
\item \textbf{Tolerance.} $\varepsilon = 10^{-9}$ in every committed campaign,
against a pinned integrator, so the gate is effectively exact match.
\item \textbf{Iteration budget per campaign.} \texttt{-{}-max-iters}, default
$5$ (\texttt{continuous\_\allowbreak danger\_\allowbreak synthesis.\allowbreak py:306}). The two region-guidance
campaigns used $15$ (the $3\times$-budget arm). No other value appears in
\texttt{results/}.
\item \textbf{The loop is memoryless.} Each iteration sends exactly one message,
\texttt{provider.\allowbreak complete([\{"role": "user",\allowbreak  "content": msg\}], model=model)}
(\texttt{contract.\allowbreak py:\allowbreak 207--208}), rebuilt from (contract, current code, current
failures). No conversation history, no system message and no assistant turn is
carried across iterations; the only state passed forward is the code string.
Asserted by
\texttt{test\_\allowbreak llm\_\allowbreak calls\_\allowbreak per\_\allowbreak seed\_\allowbreak is\_\allowbreak one\_\allowbreak plus\_\allowbreak refine\_\allowbreak iterations}, which
checks that the role sequence is \texttt{[system, user]} on call $0$ and
\texttt{[user]} on every refine call.
\item \textbf{What is recorded.} \texttt{refine\_\allowbreak iterations} per cell is the
number of refine calls issued (\texttt{contract.py:212}, \texttt{:304}); it
equals \texttt{max\_iters} exactly when the budget was exhausted without
reaching accuracy $1.0$.
\end{itemize}

\subsection{Code extraction, and an unparseable reply}
\label{app:proto-extract}

Extraction is \texttt{src/\allowbreak cwm/\allowbreak synthesizer.\allowbreak py:28--35}, used for both message types
(\texttt{contract.py:290}, \texttt{:210}): the first \texttt{python}-tagged
fenced block if there is one, else the first fenced block with any other tag,
else the whole reply stripped of surrounding whitespace.

There is therefore \textbf{no parse-failure path}: a reply containing no code
block is passed on as if it were code, and it fails the gate rather than raising.
\texttt{contract\_\allowbreak accuracy} runs the candidate in a sandbox subprocess
(\texttt{contract.\allowbreak py:\allowbreak 101--132}); if the sandbox fails, produces no output or
prints non-JSON, the accuracy is $0.0$ and the single failure line is the error
text (\texttt{contract.\allowbreak py:\allowbreak 173--175}), and the refine loop continues normally. In
the separate audit path such an artifact is classified \texttt{invalid} ---
\texttt{ast.parse} raising \texttt{SyntaxError}, a missing or uncallable
\texttt{step}/\texttt{reward}, or a raising trivial call
(\texttt{src/\allowbreak cwm/\allowbreak continuous/\allowbreak artifact\_\allowbreak class.\allowbreak py:90--113}, preflight at
\texttt{:59--88}) --- and \texttt{invalid} is deliberately kept distinct from
\texttt{gate\_failing}. Nothing retries a reply because it was unparseable, and
nothing discards a cell for it.

\subsection{Retry and error policy}
\label{app:proto-retry}

\begin{itemize}
\item \textbf{Transport-level retries only}, from the OpenAI SDK:
\texttt{max\_retries = 6}, \texttt{timeout = 120}\,s for Azure
(\texttt{src/\allowbreak cwm/\allowbreak llm/\allowbreak azure\_\allowbreak openai.\allowbreak py:10--16}) and
\texttt{max\_retries = 6}, \texttt{timeout = 300}\,s, \texttt{max\_tokens = 8192} for the OpenAI-compatible provider
(\texttt{src/\allowbreak cwm/\allowbreak llm/\allowbreak openai\_\allowbreak compat.\allowbreak py:11--16}). The SDK does exponential backoff
and honours \texttt{Retry-After}; the wrappers only raise the ceiling from the
SDK default of $2$.
\item \textbf{No application-level retry and no error swallowing.} There is no
\texttt{try}/\texttt{except} anywhere around \texttt{provider.\allowbreak complete} in
\texttt{contract.py} or in \texttt{continuous\_\allowbreak danger\_\allowbreak synthesis.\allowbreak py}. A refusal
(no choices) or missing usage metadata raises \texttt{ValueError}
(\texttt{azure\_\allowbreak openai.\allowbreak py:21--25}, \texttt{openai\_\allowbreak compat.\allowbreak py:21--22}) and the
exception propagates out of the sweep loop, ending the run.
\item \textbf{What that costs.} Nothing already computed: the sweep checkpoints
atomically after every cell (\texttt{continuous\_\allowbreak danger\_\allowbreak synthesis.\allowbreak py:179},
temporary file plus \texttt{os.replace}), skips already-present
\emph{(arm, seed index)} pairs on restart (\texttt{:155--160}), and refuses to
resume a file produced under different result-defining flags
(\texttt{:125--146}).
\item \textbf{Consequence for the paper.} No cell in \texttt{results/} was
retried, resampled or dropped because of a provider error: every cell is the
first and only draw for its \emph{(arm, seed)}.
\end{itemize}

\subsection{Sampling parameters, deployments, API version}
\label{app:proto-params}

The sampling parameters actually sent are those in the
\texttt{chat.\allowbreak completions.\allowbreak create(\dots)\allowbreak } call sites
(\texttt{provider\_\allowbreak sampling\_\allowbreak params} in the JSON): the Azure wrapper sends
\texttt{model} and \texttt{messages} only (\texttt{azure\_\allowbreak openai.\allowbreak py:\allowbreak 20}); the
OpenAI-compatible wrapper sends \texttt{model}, \texttt{messages} and
\texttt{max\_\allowbreak tokens=\allowbreak 8192} (\texttt{openai\_\allowbreak compat.\allowbreak py:19--20}). Neither sends
\texttt{temperature}, \texttt{top\_p} or a request-level \texttt{seed}: all three
are left at the provider default. The reason is recorded in the code
(\texttt{azure\_\allowbreak openai.\allowbreak py:\allowbreak 19}): \emph{no temperature/top\_p --- GPT-5.4 rejects
them}. So the runs are \textbf{not} temperature-$0$ runs; they are whatever the
deployment's default sampling is, which is why every campaign is reported as a
distribution over seeds rather than as a single artifact. Consequently the LLM
half of the pipeline is not bit-reproducible, while the sample half is (rollout
seed $10\,000 \cdot (i + 1 + \texttt{seed\_offset})$,
\texttt{continuous\_\allowbreak danger\_\allowbreak synthesis.\allowbreak py:164}); the synthesized code of every
cell is versioned in its result JSON, so every reported number is re-derivable
from the artifacts even though a re-run draws new ones.

Deployments are selected by the positional \texttt{size} argument
(\texttt{continuous\_\allowbreak danger\_\allowbreak synthesis.\allowbreak py:325--327}) and are confirmed by the
\texttt{model} field each result JSON records for itself: \texttt{large} is
\textbf{\texttt{gpt-5.4}}, \texttt{mini} is \textbf{\texttt{gpt-5.4-mini}}, and
\texttt{nano} (\texttt{gpt-5-nano}) is \textbf{never used in any paper-2
campaign} --- no result JSON carries it. The Azure API version is
\textbf{\texttt{2025-04-01-preview}}, read from \texttt{AZURE\_\allowbreak OPENAI\_\allowbreak API\_\allowbreak VERSION}
(\texttt{:331}). The cross-family arm is
\texttt{-{}-compat-model Qwen/Qwen3-Coder-30B-A3B-Instruct} through the Hugging
Face Inference Providers router at \url{https://router.huggingface.co/v1}
(\texttt{:307}) with \texttt{HF\_TOKEN} authentication (\texttt{:313}). The
Claude arm is agent-relayed and has no API wrapper
(Section~\ref{app:proto-relay}).

\subsection{LLM calls per seed}
\label{app:proto-calls}

Per cell, \texttt{llm\_calls} $=$ $1$ (synthesis, \texttt{contract.py:289}) $+$
\texttt{refine\_\allowbreak iterations} (\texttt{contract.py:207}, counter at \texttt{:212}).
The identity is verified against the actual number of provider invocations in
\texttt{tests/\allowbreak test\_\allowbreak prompt\_\allowbreak variants.\allowbreak py} (first-try success, one and three
refinements, and the \texttt{max\_iters} cap).

\begin{table}[htbp]
\centering
\footnotesize
\caption{LLM calls per campaign. ``Campaign'' names the artifact
\texttt{results/\allowbreak continuous\_\allowbreak synthesis\_\allowbreak $\langle$campaign$\rangle$.\allowbreak json}, with
\texttt{qwen} abbreviating \texttt{compat-qwen3-coder-30b-a3b-instruct}.
Every campaign uses the \texttt{default} prompt variant and
\texttt{max\_iters}~$=5$ except the two \texttt{pv-region\_it15} campaigns
(\texttt{region} variant, \texttt{max\_iters}~$=15$). ``Cap'' counts cells that
exhausted the iteration budget without passing the gate. Source:
\texttt{results/\allowbreak llm\_\allowbreak protocol\_\allowbreak facts.\allowbreak json} $\rightarrow$ \texttt{campaigns}.}
\label{tab:proto-calls}
\begin{tabular}{llrrcr}
\toprule
campaign & model & cells & calls & min / med / max & cap \\
\midrule
\texttt{qwen\_xwall8}                     & Qwen3-Coder-30B & 6  & 16  & 1 / 1 / 6    & 2  \\
\texttt{large\_xwall8}                    & gpt-5.4         & 40 & 48  & 1 / 1 / 2    & 0  \\
\texttt{large\_\allowbreak xwall8\_\allowbreak off20}             & gpt-5.4         & 40 & 50  & 1 / 1 / 2    & 0  \\
\texttt{mini\_xwall4}                     & gpt-5.4-mini    & 10 & 27  & 1 / 1 / 6    & 3  \\
\texttt{mini\_xwall8}                     & gpt-5.4-mini    & 40 & 62  & 1 / 1 / 6    & 1  \\
\texttt{patch2d\_\allowbreak qwen\_\allowbreak k3\_\allowbreak 7}             & Qwen3-Coder-30B & 3  & 3   & 1 / 1 / 1    & 0  \\
\texttt{patch2d\_\allowbreak large\_\allowbreak k3\_\allowbreak 7}            & gpt-5.4         & 40 & 140 & 1 / 3.5 / 6  & 20 \\
\texttt{patch2d\_\allowbreak large\_\allowbreak k3\_\allowbreak 7\_\allowbreak pv-region\_\allowbreak it15} & gpt-5.4   & 20 & 320 & 16 / 16 / 16 & 20 \\
\texttt{patch2d\_\allowbreak large\_\allowbreak k5\_\allowbreak 9}            & gpt-5.4         & 40 & 130 & 1 / 1 / 6    & 18 \\
\texttt{patch2d\_\allowbreak mini\_\allowbreak k3\_\allowbreak 7}             & gpt-5.4-mini    & 40 & 140 & 1 / 3.5 / 6  & 20 \\
\texttt{patch2d\_\allowbreak mini\_\allowbreak k3\_\allowbreak 7\_\allowbreak pv-region\_\allowbreak it15}  & gpt-5.4-mini & 20 & 320 & 16 / 16 / 16 & 20 \\
\texttt{patch2d\_\allowbreak mini\_\allowbreak k5\_\allowbreak 9}             & gpt-5.4-mini    & 40 & 130 & 1 / 1 / 6    & 18 \\
\texttt{patch2dsq\_\allowbreak large\_\allowbreak k3\_\allowbreak 7}          & gpt-5.4         & 40 & 140 & 1 / 3.5 / 6  & 20 \\
\texttt{patch2dsq\_\allowbreak mini\_\allowbreak k3\_\allowbreak 7}           & gpt-5.4-mini    & 40 & 140 & 1 / 3.5 / 6  & 20 \\
\texttt{pendulum\_\allowbreak qwen\_\allowbreak thstop1.\allowbreak 4}        & Qwen3-Coder-30B & 6  & 16  & 1 / 1 / 6    & 2  \\
\texttt{pendulum\_\allowbreak large\_\allowbreak thstop1.\allowbreak 4}       & gpt-5.4         & 40 & 50  & 1 / 1 / 2    & 0  \\
\texttt{pendulum\_\allowbreak large\_\allowbreak thstop1.\allowbreak 4\_\allowbreak off20}& gpt-5.4         & 40 & 54  & 1 / 1 / 2    & 0  \\
\texttt{pendulum\_\allowbreak large\_\allowbreak thstop1}         & gpt-5.4         & 40 & 59  & 1 / 1 / 2    & 0  \\
\texttt{pendulum\_\allowbreak mini\_\allowbreak thstop1.\allowbreak 4}        & gpt-5.4-mini    & 40 & 56  & 1 / 1 / 6    & 1  \\
\texttt{pendulum\_\allowbreak mini\_\allowbreak thstop1}          & gpt-5.4-mini    & 40 & 58  & 1 / 1 / 2    & 0  \\
\midrule
\textbf{total (20 campaigns)}             &                 & \textbf{625} & \textbf{1959} & 1 / --- / 16 & \\
\bottomrule
\end{tabular}
\end{table}

The totals are $20$ API campaigns, $625$ cells and $1959$ LLM calls, mean $3.13$
calls per cell, range $1$--$16$ (\texttt{totals} in the JSON); adding the two
agent-relayed Claude campaigns ($12$ cells, $39$ relayed calls) gives $1998$
calls in all. The shape of the distribution is itself a finding: on the
one-dimensional instruments the median cell costs \emph{one} call --- translation
succeeds on the first try --- whereas on PatchField2D at $k = (3,7)$ exactly half
the cells burn the whole budget ($20/40$ in each of the four \texttt{k3\_7}
campaigns), and $20/20$ in both fifteen-iteration campaigns, i.e.\ every
region-guidance cell exhausted $15$ refinements without passing the gate.

\subsection{What the gate is}
\label{app:proto-gate}

Per cell (\texttt{contract.\allowbreak py:\allowbreak 285--308}): (i)
\texttt{collect\_\allowbreak transitions(env,\allowbreak  n\_\allowbreak rollouts=40,\allowbreak  seed=10\_\allowbreak 000*(i+1+offset)\allowbreak )\allowbreak } ---
$N$ i.i.d.\ uniform-random rollouts on the truth, and \emph{the same transitions
are both the source of the shown examples and the gate}
(\texttt{contract.\allowbreak py:\allowbreak 37--57}); (ii) synthesize, then refine, giving
\texttt{gate\_accuracy} and \texttt{gate\_passed} $=$
(\texttt{accuracy} $= 1.0$); (iii) \texttt{sample\_\allowbreak contains\_\allowbreak wall} /
\texttt{sample\_\allowbreak contains\_\allowbreak mode\_\allowbreak per}, the identifiability event, logged per cell
(\texttt{contract.\allowbreak py:\allowbreak 60--64}); (iv) \texttt{mode\_blindness}
(\texttt{contract.\allowbreak py:\allowbreak 238--265}), computed \emph{only} when the gate passed: the
fraction of mode-region probe transitions the artifact gets wrong, per mode, at
$\varepsilon = 10^{-6}$, where the probes fire their mode in the truth by
construction and this is asserted at runtime (\texttt{contract.py:253}).
Single-mode instruments emit a scalar under \texttt{wall\_blindness};
PatchField2D emits the per-mode dictionary \texttt{mode\_blindness} and puts its
mean in \texttt{wall\_blindness}.

\subsection{Operational classification rules}
\label{app:proto-classes}

These are the executable definitions, not paraphrases of them.

\paragraph{repaired --- two levels, and they are not the same predicate.} On the
one-dimensional instruments (\texttt{scripts/\allowbreak audit\_\allowbreak paper2\_\allowbreak numbers.\allowbreak py:368--371})
it is \texttt{gate\_\allowbreak passed and (wall\_\allowbreak blindness or 0)\allowbreak  == 0.\allowbreak 0}: the gate reached
accuracy $1.0$ \emph{and} every mode probe is reproduced exactly, over cells with
\texttt{arm =\allowbreak =\allowbreak  "incomplete"} and \texttt{sample\_\allowbreak contains\_\allowbreak wall} true. On
PatchField2D (\texttt{audit\_\allowbreak paper2\_\allowbreak numbers.\allowbreak py:296}) it is
\texttt{[c for c in mode\_\allowbreak present if c["gate\_\allowbreak passed"]\allowbreak ]\allowbreak }: gate passage alone,
over cells whose sample contained some mode. That is the \emph{weaker} test, and
it is the honest one to quote for the $0/76$: not one artifact even reached
accuracy $1.0$, so no blindness test was needed to reject it.

\paragraph{blind, superstitious patch, disc-form --- behavioural.} All three are
decided by probing the artifact's \texttt{step} on a grid and comparing against
the pure integrator, behaviourally rather than by reading the source
(\texttt{scripts/\allowbreak patch2d\_\allowbreak artifact\_\allowbreak audit.\allowbreak py}). Instrument (\texttt{:51--55}):
an $81 \times 81$ grid over $x \in [-2, 14]$, $y \in [-8, 8]$, action $0.3$,
deviation threshold $10^{-6}$, and two velocity slices
$(v_x, v_y) \in \{(0,0), (1.5, 0.5)\}$ (\texttt{:191--192}); $m_0$ is the
zero-velocity deviation mask and $n_0$, $n_1$ the deviating-cell counts in the
two slices. Components are 4-neighbour connected (\texttt{:115--136}); the
classes in Table~\ref{tab:proto-classes} are mutually exclusive and exhaustive
and are evaluated in the order listed. The classifier is validated against
constructed artifacts of known class under \texttt{-{}-selftest}
(\texttt{:427--459}), which builds a blind artifact, a half-plane, a disc, a
square, a velocity-dependent freeze, a point trap and an exact-coordinate trap
and asserts the returned class.

\begin{table}[htbp]
\centering
\footnotesize
\caption{The behavioural classifier of
\texttt{scripts/\allowbreak patch2d\_\allowbreak artifact\_\allowbreak audit.\allowbreak py}, in evaluation order.}
\label{tab:proto-classes}
\begin{tabular}{p{0.16\textwidth}p{0.60\textwidth}r}
\toprule
class & operational rule & line \\
\midrule
\texttt{invalid} & \texttt{exec} raises, \texttt{step} missing, or
  \texttt{step} returns a state of the wrong arity & \texttt{:183--189} \\
\texttt{crash} & mask construction raises & \texttt{:190--194} \\
\texttt{blind} & $n_0 = 0$ and $n_1 = 0$, \emph{and} \texttt{step}'s body has no
  \texttt{if} & \texttt{:234--237} \\
\texttt{blind-\allowbreak textual-\allowbreak patch} (the paper's \textbf{superstitious patch}) &
  $n_0 = 0$ and $n_1 = 0$, \emph{and} \texttt{step}'s body does contain an
  \texttt{if} (\texttt{\_\allowbreak step\_\allowbreak has\_\allowbreak conditional}, \texttt{:170--175}) --- a
  written rule whose trigger set is behaviourally empty, e.g.\ an
  exact-coordinate trap, which is measure-zero and so never fires &
  \texttt{:234--236} \\
\texttt{vdep} & $n_0 = 0$ but $n_1 > 0$: deviates only at nonzero velocity &
  \texttt{:238--240} \\
\texttt{point} & largest connected component below $0.5$ world-units$^2$ &
  \texttt{:244--247} \\
\texttt{halfplane} & largest component reaches the far-east column or the top or
  bottom row of the probe window (unbounded proxy,
  \texttt{\_\allowbreak touches\_\allowbreak far\_\allowbreak edge}, \texttt{:161--167}) & \texttt{:252--254} \\
\texttt{square-form} & bounded, \texttt{bbox\_fill} $\geq 0.90$ &
  \texttt{:257--258} \\
\textbf{\texttt{disc-form}} & bounded, $0.62 \leq$ \texttt{bbox\_fill} $< 0.90$
  \emph{and} \texttt{radial\_ratio} $\leq 1.30$ ($\pi/4 \approx 0.785$ is a
  disc's bbox fill; \texttt{radial\_ratio} is max over min radial extent across
  $16$ angular bins about the centroid, \texttt{\_\allowbreak shape\_\allowbreak metrics},
  \texttt{:139--158}) & \texttt{:259--260} \\
\texttt{bounded-other} & bounded, neither of the above: ellipses, hulls, unions,
  arcs & \texttt{:261--262} \\
\bottomrule
\end{tabular}
\end{table}

Two derived per-artifact quantities the paper leans on:
\texttt{cover\_p1}/\texttt{cover\_p2}, the fraction of the \emph{true} patch's
grid cells that the artifact's $m_0$ mask covers (\texttt{:199--209}), where
``contains the seen patch'' means $\mathrm{cover} > 0.9$ and
``patch-selective'' means additionally $\mathrm{cover}(\text{unseen}) < 0.1$
(\texttt{:400--402}); and \texttt{integrator\_\allowbreak exact}, meaning no
\emph{numeric} (non-freeze-form) deviation anywhere west of $x = 1$, so that
deviations there which preserve position and zero velocity count as an
overreaching mode rule rather than as wrong arithmetic (\texttt{:216--233}).

\subsection{Prompt variants}
\label{app:proto-variants}

Each variant sets exactly three knobs --- \texttt{max\_examples},
\texttt{max\_failures} and \texttt{guidance} --- and nothing else
(\texttt{continuous\_\allowbreak danger\_\allowbreak synthesis.\allowbreak py:198--240}). The contract text is never
touched by a variant (\texttt{build\_contract} does not take \texttt{guidance}),
so no variant can leak the mode clause.

\begin{table}[htbp]
\centering
\small
\caption{The prompt variants. Only \texttt{default} and \texttt{region} appear
in \texttt{results/}; \texttt{guided} and \texttt{landing} are the intermediate
and the follow-up treatment.}
\label{tab:proto-variants}
\begin{tabular}{lrrp{0.42\textwidth}}
\toprule
variant & \texttt{max\_examples} & \texttt{max\_failures} & guidance \\
\midrule
\texttt{default} & 30  & 20 & empty string; byte-identical to every pre-knob run \\
\texttt{guided}  & 120 & 40 & describe-the-region-first process guidance
  (\texttt{:198--204}) \\
\texttt{region}  & 120 & 40 & \texttt{guided} plus ``the trigger region need not
  be a 1-D threshold'' de-bias (\texttt{:205--210}) \\
\texttt{landing} & 120 & 40 & \texttt{region} plus \emph{one sentence} naming the
  trigger's argument (\texttt{:221--225}) \\
\bottomrule
\end{tabular}
\end{table}

The \texttt{landing} sentence, verbatim (\texttt{:221--225}):

\begin{quote}
A localized rule's trigger is evaluated on the state the integrator WOULD produce
for this step --- the landing position (x2, y2) from the integrator equations
above --- not on the position the step began at (x, y).
\end{quote}

It exists because the audit of the $40$ region-guidance artifacts found $36/40$
conditioning the freeze rule on the \emph{current} position instead of the
landing position $(x_2, y_2)$, the causal variable of the true rule --- a
variable-identification failure, which is a different failure from getting the
region's geometry wrong. The sentence names only the argument: it uses the
variable names the contract's own integrator block already introduces and states
nothing about the region's shape, centre, radius or count.
\texttt{tests/\allowbreak test\_\allowbreak prompt\_\allowbreak variants.\allowbreak py} asserts that isolation four ways: the
delta over \texttt{region} is exactly one sentence; it contains no shape word and
none of the truth's numeric constants; no line that the full (mode-stating)
contract adds over the incomplete one appears in any variant's guidance; and the
guidance inserts into both the synthesis and the refine message without moving
anything else.

\subsection{The agent-relayed Claude arms}
\label{app:proto-relay}

Same contract, same example lines, same refine message and the same five-iteration
memoryless loop, but the message was relayed to a Claude instance rather than
POSTed, so there is no provider wrapper, no \texttt{max\_retries} and no usage
metadata. Every prompt and reply is versioned verbatim under
\texttt{results/\allowbreak claude\_\allowbreak relay\_\allowbreak transcripts/\allowbreak } (\texttt{*\_msg$k$.txt} /
\texttt{*\_reply$k$.txt}, one pair per iteration), and the framing wrapper ---
including the instruction not to use tools or read repository files --- is
\texttt{results/\allowbreak claude\_\allowbreak relay\_\allowbreak transcripts/\allowbreak patch2d\_\allowbreak k3\_\allowbreak 7\_\allowbreak RELAY\_\allowbreak FRAMING.\allowbreak txt}.
The two campaigns are the 1-D cart-and-pendulum ledger
(\texttt{results/\allowbreak continuous\_\allowbreak claude\_\allowbreak relay.\allowbreak json}: $8$ cells, $20$ relayed calls,
min/median/max $1/1.5/6$) and the 2-D one
(\texttt{results/\allowbreak continuous\_\allowbreak claude\_\allowbreak relay\_\allowbreak patch2d\_\allowbreak k3\_\allowbreak 7.\allowbreak json}: $4$ cells, $19$
relayed calls, $1/6/6$). Both counts are cross-checked against the number of
\texttt{*\_msg*.txt} files on disk and agree exactly. Two files in the 2-D
transcript directory are named \texttt{.\allowbreak .\allowbreak \_\allowbreak duplicate\_\allowbreak relay\_\allowbreak DISCARDED.\allowbreak txt}: a
message relayed twice by accident, whose second reply was discarded. They are not
protocol iterations and are excluded from the count
(\texttt{scripts/\allowbreak llm\_\allowbreak protocol\_\allowbreak facts.\allowbreak py}, \texttt{\_\allowbreak transcript\_\allowbreak call\_\allowbreak count});
with them the on-disk file count would be $21$ rather than $19$.

\section{Reproducibility}
\label{sup:repro}

Everything in this section is generated rather than transcribed: the facts below are read
from \texttt{results/\allowbreak repro\_\allowbreak manifest.json}, and the numeric audit
checks the paper against that file like any other result.

\paragraph{Interpreter, platform, and one honest gap.} CPython $3.12.8$ on macOS $15.7.7$
(build \texttt{24G720}, Darwin $24.6.0$), Intel Core i9-9880H at $2.30$\,GHz, $16$ logical
cores, $16$\,GiB. \texttt{pyproject.toml} declares \texttt{requires-python >= 3.11}. No
result depends on the hardware except wall-clock: every CPU result is pure-Python or
\texttt{numpy} float64 with explicitly seeded \texttt{random.Random} streams,
single-threaded, no GPU, and no reduction whose order a BLAS could change, so bit-identical
replay is expected on any IEEE-754 host at the same CPython minor version. That expectation
is an argument from the code and \emph{not} a measurement: every number in this paper was
produced on one machine, and no second platform has ever run the sweeps. The quantities most
exposed to a different \texttt{libm} are the ones printed to twelve digits
($J_{\mathrm{truth}} = 17.757356407381$ and $17.772246981024$) and the $10^{-15}$
linear-fit residuals of Table~\ref{tab:smooth}; those are the first to re-check elsewhere.

\paragraph{Dependencies, and what is still missing.} \texttt{pyproject.toml} declares lower
bounds only. A pinned environment is committed at
\texttt{docs/\allowbreak paper2/\allowbreak requirements-frozen.\allowbreak txt} ($\texttt{pip freeze}$ inside the repository's
virtual environment), which is \emph{not} a lockfile: it carries no hashes and no resolver
metadata. What reproducibility still needs and does not have is an archived release --- a
git tag with a DOI --- so that ``the code is on GitHub'' becomes a citable artifact rather
than a mutable branch. We state this as an outstanding gap rather than implying otherwise.

\paragraph{Licence.} Code: Apache-2.0 (\texttt{LICENSE}, declared in
\texttt{pyproject.toml}). Paper text, figures and the result artifacts under
\texttt{results/}, including every synthesized program the paper quotes: CC-BY-4.0
(\texttt{docs/\allowbreak paper2/\allowbreak LICENSE-artifacts}).

\paragraph{Credentials.} The scripts read exactly six environment variables, listed with
their call sites in \texttt{docs/\allowbreak paper2/\allowbreak env.\allowbreak example}: \texttt{AZURE\_\allowbreak OPENAI\_\allowbreak ENDPOINT},
\texttt{AZURE\_\allowbreak OPENAI\_\allowbreak API\_\allowbreak KEY}, \texttt{AZURE\_\allowbreak OPENAI\_\allowbreak API\_\allowbreak VERSION},
\texttt{AZURE\_\allowbreak DEPLOYMENT\_\allowbreak MINI}, \texttt{AZURE\_\allowbreak DEPLOYMENT\_\allowbreak LARGE} and \texttt{HF\_TOKEN}.
No secret appears in the repository. Every CPU result in this paper runs without any of them.

\paragraph{What backs each table and figure.} The manifest maps every table and figure to
the versioned JSON it is derived from and the script that produced it, and it distinguishes
three tiers of checking rather than collapsing them: \emph{cell-parsed}, where
\texttt{scripts/\allowbreak audit\_\allowbreak paper2\_\allowbreak numbers.py} parses the
\texttt{tabular} out of \texttt{main.tex} and compares every cell against the JSON at
printed precision; \emph{claims-only}, where the audit does not parse the object but asserts
values quoted from it; and \emph{unaudited}. Of the fifteen tables and figures, nine are
cell-parsed, five are claims-only, and one --- Table~\ref{tab:pendulum-synthesis}, whose
counts are read from the per-seed synthesis JSONs by hand --- is unaudited, which we record
as a gap rather than leave implicit. The manifest is itself cross-checked two ways: a static
parse of the audit's sections collects the file arguments it names, and a dynamic pass execs
the audit with file reads instrumented and records what it actually opens. The two agreeing
is a stored assertion, and it failed on the generator's first run (the static parse missed
the files reached through a wrapper and a glob), which is what the cross-check is for.

\begin{table}[ht]
\centering
\footnotesize
\begin{tabular}{lll}
\toprule
object & backing \texttt{results/}\,JSON & checking \\
\midrule
\texttt{tab:epsstar} & \texttt{eps\_\allowbreak invariance\_\allowbreak threshold} & claims \\
\texttt{tab:danger} & \texttt{continuous\_\allowbreak reach} & cells \\
\texttt{fig:threshold} & \texttt{continuous\_\allowbreak reach} & claims \\
\texttt{fig:reach} & \texttt{continuous\_\allowbreak reach} & claims \\
\texttt{tab:pendulum} & \texttt{continuous\_\allowbreak pendulum} & cells \\
\texttt{tab:patch2d} & \texttt{continuous\_\allowbreak patch2d} & cells \\
\texttt{tab:cem} & \texttt{continuous\_cem}, \texttt{cem\_\allowbreak crossing\_\allowbreak bound} & cells \\
\texttt{tab:axes} & \texttt{continuous\_\allowbreak axes} & cells \\
\texttt{fig:axes} & \texttt{continuous\_\allowbreak axes} & claims \\
\texttt{tab:eps-sweep} & \texttt{continuous\_\allowbreak eps\_\allowbreak sweep} & cells \\
\texttt{tab:mitigation} & \texttt{continuous\_\allowbreak mitigation} & cells \\
\texttt{tab:\allowbreak patch2d-mitigation} & \texttt{continuous\_\allowbreak mitigation\_\allowbreak patch2d} & cells \\
\texttt{tab:\allowbreak pendulum-synthesis} & \emph{none} (per-seed JSONs) & \emph{unaudited} \\
\texttt{tab:smooth} & \texttt{continuous\_\allowbreak smooth\_\allowbreak probe} & cells \\
\texttt{fig:smooth} & \texttt{continuous\_\allowbreak smooth\_\allowbreak probe} & claims \\
\bottomrule
\end{tabular}
\caption{The manifest: what backs each table and figure, and how it is checked. ``cells''
means every printed cell is compared against the JSON at printed precision on every run of
the numeric audit; ``claims'' means the audit asserts quoted values without parsing the
object. Each JSON is written by the script of the same stem under \texttt{scripts/} (\texttt{continuous\_reach.json} by \texttt{continuous\_reach.py}, and so on); the figures are rendered from their table's JSON by \texttt{make\_paper2\_figures.py}, which runs no experiment, so every figure inherits its table's provenance. Machine-readable form, including the audit call sites per object, in
\texttt{results/\allowbreak repro\_\allowbreak manifest.\allowbreak json} under \texttt{manifest[]}.}
\label{tab:manifest}
\end{table}

\paragraph{Runtime, as a lower bound.} Thirty-seven of the result files record their own
\texttt{elapsed\_s}: CPU sweeps total $4.78$\,h and the LLM synthesis arms $6.01$\,h, so the
named campaigns are $10.79$\,h. That is a \emph{lower} bound on what the paper cost, because
each figure is one run of the script that wrote the file and therefore excludes the 5-seed
cells later superseded at 20 seeds, the ablation resume passes, every sweep superseded by a
later sample-size raise, and every failed or retried LLM call. Seven scripts named in the
commands below carry a \texttt{\~{}N min} annotation in this paper but write no
\texttt{elapsed\_s}, so those seven runtimes are \emph{unverified estimates} and are marked
as such in the manifest.

\paragraph{LLM cost, as a range and why not a point.} The call count is exact:
\textbf{1959} API calls over \textbf{625} cells across twenty campaigns, derived per cell as
one synthesis call plus its refinement iterations, of which the two guided
$3\times$-budget cells account for $640$; plus $20$ relay rounds on the
subscription-transport Claude arms. The \emph{token} count is not recorded --- the providers
capture usage and the refinement loop keeps it, but the emitted cell never copies it, so no
paper-2 artifact carries a token count. Combining the measured call count with the design
spec's pre-run estimate of $1$--$2$k prompt tokens per call brackets the API arms at roughly
$10^6$ to $10^7$ tokens, i.e.\ single to low-double-digit US dollars at 2026 prices for this
model class. We state that range rather than a figure, and note what would make it exact: an
Azure billing export for the window, or a re-derivation of first-call prompt tokens offline,
which is possible at zero API spend because \texttt{collect\_\allowbreak transitions},
\texttt{build\_contract} and \texttt{build\_\allowbreak synthesis\_\allowbreak messages} are pure functions of the
instrument, knob, seed and $N$ (completion tokens for superseded refinement iterations are
gone). The only \emph{measured} dollar figures in the repository belong to the companion
paper's campaigns and total \$3.73, useful as an order-of-magnitude anchor and nothing more.

\paragraph{Reproducing the paper.} Two commands regenerate everything derived:
\texttt{make\_\allowbreak paper2\_\allowbreak figures.\allowbreak py} rebuilds the figures from the JSONs, and
\texttt{audit\_\allowbreak paper2\_\allowbreak numbers.\allowbreak py} re-derives every table cell and counted prose claim from
\texttt{results/} and exits non-zero on any disagreement. The experiment commands themselves
are listed below; the CPU ones need no credentials.

\section{Pre-specification, and what was added after the result}
\label{sup:prespec}

A study that reports negative results, ablations built to explain them, and theory derived
after the fact owes the reader a way to tell those apart. This section is that ledger. Every
date below is read from the repository's history rather than recalled, and the commands that
re-derive each one are listed in \texttt{docs/\allowbreak paper2/\allowbreak PRESPEC-LEDGER.\allowbreak md}. Five labels are
used: \textbf{pre-specified} (a dated design document predicted it before any run),
\textbf{confirmatory} (ran as specified), \textbf{diagnostic} (added to explain or repair an
observed result), \textbf{exploratory ablation} (added after seeing a result, to test a
hypothesis the result suggested), and \textbf{post-hoc} (an analysis or a theorem chosen
after the data existed).

\paragraph{What was pre-specified.} The design document is dated 2026-07-06 and its
predictions were committed before the first run. Fourteen items are pre-specified and
confirmatory: the threshold law's expected shape; the reach mechanism as an explicit
\emph{go/no-go} gate on the whole study; the wall position as the rarity knob; the
gate-miss exactness check against $(1-r)^N$; the pervasive-error control arm that makes the
axis separation a separation; the smooth-bump contrast, with ``danger stays low'' predicted;
the smooth-learner probe; the full-spec synthesis control; the mode-absent conditional,
including the ${\approx}60\%$ rate; the $\varepsilon$-sensitivity sweep, which the spec
listed under \emph{Risks} as required; the pendulum as the second instrument; the Qwen
cross-family check; and the pendulum synthesis arm. The tightest of these is worth naming,
because it is the strongest pre-specification claim the paper can make: the spec's runbook
predicted the mode-present branch as a \emph{contingency} --- ``either outcome is a finding;
if it repairs, that becomes its own section'' --- so the repair result is a pre-registered
branch and not a post-hoc pivot. The first confirmation landed $2$\,h $40$\,min after the
prediction was committed.

\paragraph{What was added after the result it responds to.} Fifteen items, and the five the
review asks about specifically are all in this group. The PatchField2D instrument was named
in the spec as one of two options for a second instrument, with no predictions attached, and
its campaign was built eleven days later in response to a review of the 1D work. The
\textbf{square ablation} was added two to three days after the 2D repair failure it explains,
and it is a genuine falsification test of a hypothesis that was itself post-hoc: the
curvature reading was born of the result the ablation then excluded. The
\textbf{guided-prompt ablation} was added two days after the same failure, and its two cells
are the most expensive in the paper ($320$ calls each, $640$ of the campaign's $1959$). The
\textbf{disjoint seed blocks} and the \textbf{$50{,}000$-rollout dependence re-measurement}
are both repairs of errors a review found --- pooled intervals over shared samples, and an
unstated independence assumption --- eighteen and eight days after the results they correct.
The \textbf{sharp-plateau reward variant} was added nineteen days after the ``below random at
every knob'' claim it rescues. The distrust-region mitigation appears in no version of the
design document at all; the paper presents it as a contribution rather than as a prediction,
which is the correct framing. Two further items changed the paper's own reading of a result
rather than confirming it: the per-episode fence census turned a ``boundary-mapping
transient'' into ``a growing fraction of outright failures'', and the behavioural artifact
audit shipped once with a vacuous check (a mis-keyed dictionary lookup meant its
partial-repair condition never fired), whose correction moved a headline count.

\paragraph{The theory is post-hoc, and part of it was explicitly out of scope.} Every
proposition-backing artifact in this paper was produced after all the data existed, most of
them because a measurement looked suspiciously exact and the question was whether something
could be proved. That is a legitimate way to find a theorem and an illegitimate way to claim
a prediction, so none of the theory is presented as predicted. One item is more than
post-hoc: the design document said of exactly this material, ``covering-number analogues are
open; mention as a limitation, do not attempt for paper 2''. The paper now contains four
such results. They were added late, and the first version of the fencing bound was wrong in
the direction that mattered --- a covering number where a packing number was needed, with a
hand-computed constant that was also wrong --- caught by a reader's question rather than by
the numeric audit, and corrected the same day.

\paragraph{The manuscript's own claims of pre-registration and falsification.} Ten phrases
in this paper assert something about our process. Seven are supported, one is supported as a
statistical rather than a temporal claim (the coverage certificate's cell family is fixed
before the data by \emph{construction}, which is what that argument needs and does not
require a dated record), and \textbf{two were not supported by any dated record} and have
been corrected in this revision:
\begin{itemize}
\item The sharp-plateau variant's competence risk was described as ``the pre-registered
risk''. No dated text records it before the result; the pre-run note for that variant states
a cost, a payoff and a different caveat. The text now says ``the risk recorded with the
variant's script''.
\item The 2D section said ``we predicted partial repair''. The instrumentation that detects
partial repair --- the per-mode sample flags and the four-way branch summary --- was written
in the same commit as the run, so the expectation is real and designed for; but the sentence
itself first appears one day \emph{after} the measurement. The text now says the instrument
was built to detect partial repair and that none occurred.
\end{itemize}
A third phrase, ``an axis-aligned square with flat edges excludes the curvature reading'',
is sound as logic but was silent about chronology; it now states that the ablation was added
after the result it responds to.

\paragraph{Summary.} Fourteen pre-specified and confirmatory items, all on the
one-dimensional story: the mechanism, the axis separation, the smooth contrast, both 1D
instruments, and the synthesis trichotomy including its contingency branch. Fifteen items
added after the results they respond to, of which five are diagnostic repairs driven by
review, five are exploratory ablations, three are post-hoc analyses and two are mixed. Ten
post-hoc theory results. The reader should treat the 1D synthesis result as confirmatory and
everything about the two-dimensional mode --- including the two ablations that exclude
curvature and prompting --- as exploratory work whose hypotheses were formed after the
finding.

\section*{Declarations}

\paragraph{Funding.} This work received no external funding. The LLM API costs, bounded in
Section~\ref{sup:repro}, were borne by the author.

\paragraph{Competing interests.} The author declares no competing interests.

\paragraph{Data and code availability.} All code, every result artifact, and every
synthesized program the paper quotes are in the repository named in Section~\ref{sup:repro}
(code under Apache-2.0, paper and artifacts under CC-BY-4.0). The two commands that
regenerate every figure and re-derive every table cell from those artifacts are given there.
Section~\ref{sup:repro} also states what availability still lacks: an archived, DOI-bearing
release rather than a branch.

\paragraph{Use of AI systems.} Large language models are the object of study here: they
synthesize the world models the paper evaluates, under the protocol specified in full in
Section~\ref{sup:protocol}, with every prompt, deployment identifier and call count
recorded. Language models were additionally used as writing and analysis assistants during
the preparation of the manuscript and the experiment code. All theorem statements, proofs,
experimental designs and conclusions are the author's, and every number in the paper is
re-derived from a versioned artifact by \texttt{scripts/audit\_paper2\_numbers.py} rather
than taken on trust.

\paragraph{Anonymization.} This version is not anonymized: it names the author, the
affiliation, a personal URL, and a repository whose name identifies the author. A
double-blind submission requires removing those four, and replacing the nine citations to
the companion paper with a third-person reference to it.

\bibliography{references}

\end{document}